%% file: main.tex
\documentclass{article}
\usepackage{microtype}
\usepackage{graphicx}
\usepackage{booktabs}
\usepackage[shortlabels]{enumitem}

\usepackage[skins,theorems]{tcolorbox}
\tcbset{highlight math style={enhanced,
colframe=black,colback=white,arc=0pt,boxrule=1.0pt}}

\newcommand{\proposed}{{MF-Diffuser}}
\newcommand{\mfcdm}{{MF-CDMs}}

\usepackage{url}
\usepackage{physics}
\usepackage{amsfonts}
\usepackage{nicefrac}
\usepackage{amssymb}
\usepackage{amsmath}
\usepackage{amsthm}
\usepackage{mathtools}
\usepackage{tikz}
\usetikzlibrary{decorations.pathmorphing, arrows.meta, positioning, calc}
\definecolor{accent}{RGB}{46, 111, 149}
\definecolor{annotgray}{RGB}{102, 102, 102}
\usepackage{bbm}
\usepackage{calc}
\usepackage[font={small}]{caption}
\usepackage{lipsum}
\usepackage{etoolbox}
\usepackage{wrapfig, blindtext}
\usepackage{multirow}
\usepackage{subcaption}
\usepackage{algorithm}
\usepackage[noend]{algorithmic}
\usepackage{pifont}
\usepackage{adjustbox}
\usepackage{hyperref}
\hypersetup{
    colorlinks=true,
    breaklinks=true,
    linkcolor=red,
    filecolor=blue,
    citecolor=blue,
    urlcolor=cyan,
}
\DeclareMathOperator*{\argmax}{arg\,max}
\DeclareMathOperator*{\argmin}{arg\,min}

\usepackage[capitalize,noabbrev]{cleveref}
\usepackage{natbib}

\theoremstyle{plain}
\newtheorem{theorem}{Theorem}[section]
\newtheorem{proposition}[theorem]{Proposition}
\newtheorem{lemma}[theorem]{Lemma}
\newtheorem{corollary}[theorem]{Corollary}
\newtheorem{definition}[theorem]{Definition}
\newtheorem{assumption}[theorem]{Assumption}
\newtheorem*{remark}{Remark}

\newcommand{\ie}{\textit{i}.\textit{e}., }
\newcommand{\eg}{\textit{e}.\textit{g}., }

\newcommand{\traj}{\boldsymbol{\tau}}
\newcommand{\state}{\mathbf{s}}
\newcommand{\action}{\mathbf{a}}

\newcommand{\policy}{\pi}

\newcommand{\qfunc}{Q}
\newcommand{\horizon}{H}

\newcommand{\Xcal}{\mathcal{X}}
\newcommand{\Scal}{\mathcal{S}}
\newcommand{\Acal}{\mathcal{A}}
\newcommand{\Pcal}{\mathcal{P}}
\newcommand{\Jcal}{\mathcal{J}}
\newcommand{\Hcal}{\mathcal{H}}
\newcommand{\Ecal}{\mathcal{E}}
\newcommand{\Wcal}{\mathcal{W}}
\newcommand{\Dcal}{\mathcal{D}}

\newcommand{\Mtwo}{\mathcal{M}_2}

\AtBeginEnvironment{algorithmic}{\small}

\usepackage[preprint]{neurips_2026}

\title{Mean-Field Diffuser: Scaling Offline MARL to Thousands of Agents}

\author{%
  Wenhao Li\\
  Tongji University\\
  \And
  Xiangfeng Wang\\
  East China Normal University\\
  \And
  Bo Jin\\
  Tongji University\\
}

\begin{document}

\maketitle

%=============================================================================
% ABSTRACT
%=============================================================================
\begin{abstract}
Diffusion-based planning has achieved strong results in single-agent offline reinforcement learning, yet scaling to many-agent systems remains intractable due to the curse of dimensionality in the joint trajectory space. We introduce \proposed{}, a framework that lifts trajectory planning to the Wasserstein space of trajectory distributions, where the propagation of chaos ensures a small representative subset of agents captures the full population dynamics. Our approach features a \textit{value-weighted chaotic entropy} objective that reconciles generative fidelity with return maximization, and a hierarchical coarse-to-fine strategy that progressively grows the agent population during denoising. We establish end-to-end suboptimality bounds with four interpretable terms, revealing that mean-field approximation error scales as $\mathcal{O}(\horizon^2/\sqrt{N})$ while offline distribution shift provably does not grow with population size $N$, and prove the generated policy is an approximate mean-field Nash equilibrium with explicit convergence guarantees. Experiments on three mean-field RL benchmarks---spanning stage games, sequential dynamics, and adversarial team competition---show \proposed{} achieves the best return in the majority of settings, with the largest gains on suboptimal offline data and at extreme scales ($N \geq 10^3$).
\end{abstract}

\vspace{-4mm}
%=============================================================================
% INTRODUCTION
%=============================================================================
\section{Introduction}\label{sec:intro}

Offline reinforcement learning (offline RL)~\citep{levine2020offline, lange2012batch} has emerged as a principled framework for learning decision-making policies from pre-collected datasets without further environment interaction. Among recent advances, diffusion-based planning methods such as Diffuser~\citep{janner2022planning} and Decision Diffuser~\citep{ajay2023is} have demonstrated remarkable effectiveness by casting trajectory optimization as a conditional generative modeling problem: a denoising diffusion model is trained to capture the distribution of trajectories in the offline dataset, and high-return plans are generated at inference time via guided sampling.

Despite their empirical success, these diffusion-based planners are inherently designed for \textbf{single-agent} settings. Extending them to \textbf{many-agent systems}—environments populated by hundreds or thousands of interacting agents—poses a fundamental scalability challenge. Consider $N$ homogeneous agents, each with state space $\Scal$ and action space $\Acal$. A joint trajectory over horizon $\horizon$ lives in $\mathbb{R}^{N \times (d_s + d_a) \times \horizon}$. Even for moderate values of $N$, the dimensionality of this joint space grows linearly in $N$, leading to an exponential blowup in the sample complexity required for accurate diffusion modeling~\citep{de2022convergence, chen2023improved}. This is precisely the \textit{curse of dimensionality} that plagues conventional approaches.

Meanwhile, many-agent decision-making problems are ubiquitous in real-world applications. Large-scale traffic control~\citep{vinitsky2018benchmarks}, swarm robotics~\citep{dorigo2021swarm}, financial market modeling~\citep{lachapelle2010computation}, and epidemiological policy design~\citep{elie2020contact} all involve populations of agents whose individual decisions are coupled through shared environmental dynamics. For such problems, \textbf{mean-field reinforcement learning} (MFRL)~\citep{yang2018mean, gu2021mean} provides a principled approximation: by replacing the explicit dependence on all other agents with a dependence on the \textit{population distribution} (mean field), the complexity of the $N$-agent problem is reduced to that of a representative single agent interacting with the population.

In this paper, we introduce \textbf{\proposed{}} (\textbf{M}ean-\textbf{F}ield Diffuser), a diffusion-based planning framework that bridges the gap between trajectory-level generative planning and mean-field multi-agent systems. Our key insight is to model the $N$-agent trajectory planning problem through the lens of \textit{interacting particle systems}: by treating each agent's trajectory as a particle in a high-dimensional trajectory space, the joint planning problem becomes amenable to mean-field analysis, where the propagation of chaos property~\citep{sznitman1991topics} ensures that the behavior of the full population can be approximated by a representative subset.

However, transferring mean-field particle theory to the sequential decision-making domain introduces three fundamental challenges absent from static settings: \textbf{(C1)} agent trajectories carry \textit{temporal coupling}---interaction effects accumulate over the planning horizon $\horizon$, causing approximation errors to compound multiplicatively rather than additively; \textbf{(C2)} the objective is not merely distributional matching but \textit{return maximization}; and \textbf{(C3)} offline data introduces \textit{distribution shift} whose interaction with the mean-field structure must be controlled. 

\proposed{} addresses all three challenges with the following contributions:
(1) \textbf{Mean-field Trajectory Diffusion with Value-weighted Chaotic Entropy} (C1, C2, C3). We formulate many-agent offline RL as generative modeling on the Wasserstein space of trajectory distributions, defining mean-field SDEs over trajectories with temporally-structured interaction kernels. Building on the RL-as-inference framework~\citep{levine2018reinforcement}, we derive a mean-field value score matching (MF-VSM) objective that unifies distributional fidelity with return maximization at the chaotic limit.
(2) \textbf{Hierarchical Coarse-to-Fine Planning} (C1). Justified by propagation of chaos ($M = \tilde{\mathcal{O}}(\sqrt{N})$ agents suffice), we progressively grow the agent population during denoising---starting from a small representative group and branching to the full population.
(3) \textbf{End-to-end Planning Guarantee} (C1, C3). Theorem~\ref{thm:hierarchical} decomposes the suboptimality into four interpretable terms: score matching error, subdivision error with geometric decay, mean-field approximation scaling as $\mathcal{O}(\horizon^2/\sqrt{N})$, and an offline distribution shift that provably does \textit{not} grow with $N$.
(4) \textbf{Game-theoretic Guarantees} (C2, C3). Theorem~\ref{thm:exploitability} shows the generated policy is an approximate mean-field Nash equilibrium with $\mathcal{O}(1/\sqrt{N})$ rate; under Lasry--Lions monotonicity we further establish convergence to the unique MFE (Theorem~\ref{thm:monotone_convergence}) and characterize the social welfare--Nash efficiency gap (Proposition~\ref{prop:poa}; both in Appendix~\ref{sec:mfg_appendix}).

%=============================================================================
% PROBLEM SETUP
%=============================================================================
\section{Problem Setup}\label{sec:prelim}

\textbf{Mean-field MDP.} We consider a symmetric $N$-agent MDP. At step $h \in \{0,\ldots,\horizon-1\}$, agent $i$ has state $\state_h^i \in \Scal \subseteq \mathbb{R}^{d_s}$, action $\action_h^i \in \Acal \subseteq \mathbb{R}^{d_a}$, with discount $\gamma\in(0,1)$. The \textbf{mean-field interaction} assumption stipulates that dynamics and rewards depend on others only through the \emph{empirical state distribution} $\bar{\mu}_h^N = \frac{1}{N}\sum_{j=1}^{N} \delta_{\state_h^j}$:
\begin{equation}\label{eq:mf_transition}
    \state_{h+1}^i \sim P(\cdot | \state_h^i, \action_h^i, \bar{\mu}_h^N), \qquad r_h^i = r(\state_h^i, \action_h^i, \bar{\mu}_h^N).
\end{equation}
Agents are \textbf{homogeneous} (sharing $P$ and $r$); we seek a common policy $\policy: \Scal \times \Pcal(\Scal) \to \Pcal(\Acal)$ maximizing the social welfare
\begin{equation}\label{eq:mf_objective}
    \max_{\policy} J(\policy) = \mathbb{E}\!\left[\sum_{h=0}^{\horizon-1} \gamma^h \tfrac{1}{N}\!\sum_{i=1}^{N} r(\state_h^i, \action_h^i, \bar{\mu}_h^N)\right].
\end{equation}
As $N\!\to\!\infty$, $\bar{\mu}_h^N$ converges to a deterministic flow $\mu_h$, reducing the problem to a representative agent. The \textbf{offline dataset} $\Dcal$ consists of $N$-agent episodes collected under behavior policy $\policy_\beta$; trajectories $\traj^i = (\state_0^i, \action_0^i, \ldots, \state_\horizon^i) \in \mathbb{R}^{D_\tau}$ with $D_\tau = (d_s + d_a)\horizon + d_s$. Mean-field Q-learning (MFQ)~\citep{yang2018mean} factors the joint $Q$-function as $Q^{MF}(\state^i,\action^i,\bar\mu_\action)$ via mean-action coupling and serves as our behavior-policy collector for the offline data.

\textbf{Score-based Diffusion Trajectory Planning.} Diffuser~\citep{janner2022planning} models the trajectory distribution $p(\traj)$ in $\Dcal$ via a forward SDE that progressively corrupts $\traj$ with Gaussian noise, $d\traj_u = f_u(\traj_u)du + \sigma_u dB_u$ on $u\!\in\![0,T]$, and a learned reverse SDE that recovers samples by following the score $\nabla\log\zeta_t(\traj_t)$ of the noised marginal. The score network $\mathbf{s}_\theta$ is trained by score matching, $\Jcal_{SM}(\theta) = \mathbb{E}_{t,\traj_t}\norm{\mathbf{s}_\theta(t,\traj_t) - \nabla\log\zeta_t(\traj_t)}^2$. At inference, trajectories are generated by simulating the reverse SDE from Gaussian noise; high-return plans are obtained by adding a value-gradient bias $\eta\,\nabla_{\traj}\hat V(\traj)$ to the score, and the observed initial state is enforced via inpainting. Decision Diffuser~\citep{ajay2023is} replaces the gradient bias with classifier-free return-bucket conditioning. We build on this single-agent paradigm but lift it to the joint $N$-agent setting via mean-field analysis. An extended review is in Appendix~\ref{sec:prelim_appendix}.

%=============================================================================
% METHOD
%=============================================================================
\section{Mean-field Diffuser}\label{sec:method}

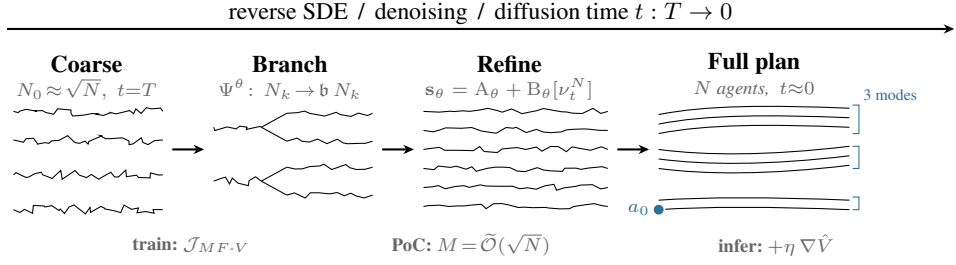
\begin{figure}[t]
\centering
\input{figures/method_overview_body}
\vskip -0.05in
\caption{\textbf{\proposed{} method overview.} Inference proceeds left-to-right along reverse diffusion time $t:T\!\to\!0$. (\textit{Coarse}) start from $N_0\!\approx\!\sqrt{N}$ Gaussian-noise trajectories; (\textit{Branch}) the operator $\Psi^\theta\!:\!N_k\!\to\!\mathfrak{b}N_k$ spawns child trajectories at designated denoising steps; (\textit{Refine}) the mean-field score $\mathbf{s}_\theta\!=\!\mathrm{A}_\theta\!+\!\mathrm{B}_\theta[\nu_t^N]$ couples individual dynamics with the population through $\mathrm{B}_\theta$; (\textit{Full plan}) at $t\!\approx\!0$ the $N$ trajectories concentrate on the data manifold (e.g.\ 3 modes shown), and only the first action $a_0$ of each is executed before re-planning.}
\label{fig:overview}
\vskip -0.15in
\end{figure}

\proposed{} recasts many-agent offline RL as mean-field trajectory generation; Figure~\ref{fig:overview} sketches the inference pipeline. The three components below address the curse of dimensionality.

\subsection{Mean-field Trajectory SDEs}\label{sec:mf_trajectory_sde}

We model each agent's trajectory as a particle in $\Xcal \coloneqq \mathbb{R}^{D_\tau}$. Given $N$ agents, $\traj^N = (\traj^1, \ldots, \traj^N) \in \Xcal^N$ is an exchangeable $N$-particle system amenable to mean-field analysis. \textbf{Two time axes appear throughout:} the MDP step $h\!\in\!\{0,\ldots,\horizon\!-\!1\}$ inside each trajectory $\traj^i$, and the diffusion time used by the generative process. The \emph{forward} (noising) process uses $u\!\in\![0,T]$ running from data ($u{=}0$) to noise ($u{=}T$); the \emph{reverse} (denoising) process uses $t\!=\!T\!-\!u\!\in\![0,T]$. We use $u$ in the forward SDE and $t$ in the reverse SDE below.

\begin{definition}[Mean-field Trajectory SDEs]\label{def:mf_traj_sde}
The $N$-agent forward-reverse trajectory diffusion is:
\begin{align}
    \textit{Forward:} \quad d\traj_u^{i,N} &= f_u(\traj_u^{i,N})du + \sigma_u dB_u^{i,N}, \quad \traj_{u=0}^{i,N} \sim p_{\mathrm{data}}, \label{eq:mf_traj_forward} \\
    \textit{Reverse:} \quad d\traj_t^{i,N} &= \left[f_t(\traj_t^{i,N}) - \sigma_t^2 \nabla \log \zeta_t(\traj_t^{i,N})\right]dt + \sigma_t d\bar{B}_t^{i,N}, \label{eq:mf_traj_reverse}
\end{align}
with joint law $\nu_t^N = \mathbf{Law}(\traj_t^{1,N},\ldots,\traj_t^{N,N})$ that is exchangeable under permutations.
\end{definition}

\textbf{Mean-field Interaction for Trajectories.} Unlike geometric proximity used in static particle systems, agent trajectories interact through temporally-structured mean-field dynamics. We define the interaction operator and the score network as
\begin{equation}\label{eq:traj_interaction}
    \mathrm{B}_\theta[\nu_t^N](\traj^i) = \tfrac{1}{N-1}\!\sum_{j \neq i}\! K_\theta(\traj^i, \traj^j) \cdot \mathrm{B}_\theta(\traj^i, \traj^j), \quad \mathbf{s}_\theta(t, \traj^N, \nu_t^N) = \mathrm{A}_\theta + \mathrm{B}_\theta[\nu_t^N],
\end{equation}
where $\mathrm{A}_\theta$ captures individual dynamics and $K_\theta(\traj^i, \traj^j) = \frac{1}{\horizon}\sum_{h} k_\theta(\state_h^i, \state_h^j, \bar{\mu}_h^N)$ decomposes the kernel across MDP steps so that the coupling depends on the evolving population distribution.

\subsection{Value-weighted Chaotic Entropy}\label{sec:value_weighted}

We resolve the tension between distributional matching (standard score matching) and return maximization (challenge C2) via the \emph{value-weighted chaotic entropy} at the mean-field limit.

\textbf{RL as Inference.} Following~\citet{levine2018reinforcement}, we define the optimality variable $\mathcal{O}=1$ with $p(\mathcal{O}{=}1|\traj^i,\mu) \propto \exp(\alpha^{-1} R(\traj^i;\mu))$, where $R(\traj^i;\mu) = \sum_h \gamma^h r(\state_h^i,\action_h^i,\mu_h)$ and $\alpha>0$ is a temperature. The optimal trajectory distribution is $p^*(\traj^i|\mathcal{O}{=}1,\mu) \propto p(\traj^i|\mu) \exp(R(\traj^i;\mu)/\alpha)$. The mean-field approximation $R(\traj^i; \bar{\mu}^N) \approx R(\traj^i; \mu)$ is controlled by Theorem~\ref{thm:poc_traj}.

\begin{definition}[Value-weighted Chaotic Entropy]\label{def:vw_entropy}
The value-weighted $N$-particle relative entropy is
\begin{equation}\label{eq:vw_entropy}
    \Hcal_V^N(\nu_T^N) \,\coloneqq\, \tfrac{1}{N}\!\int_{\Xcal^N}\!\Big[\log\tfrac{\varrho_T^N}{\zeta_0^{\otimes N}} - \tfrac{1}{\alpha}\!\sum_{i=1}^{N}\! R(\traj^i)\Big]\varrho_T^N d\traj^N,
\end{equation}
combining distributional matching (log-ratio) with return maximization (reward).
\end{definition}
At the chaotic limit, $\Hcal_V^N(\nu_T^N) \to \Hcal_V^\infty(\mu_T) = \Hcal(\mu_T|\zeta_0) - \alpha^{-1}\mathbb{E}_{\mu_T}[R(\traj)]$.

\textbf{Mean-field Value Score Matching.} Applying the Itô-Wentzell-Lions formula and Sobolev upper bounds (proof in Appendix~\ref{sec:proof_mfvsm}) yields the \textbf{MF-VSM} objective.

\begin{proposition}[MF-VSM Bound]\label{prop:mf_vsm}
Let $\Mtwo \coloneqq \mathbb{E}_{\zeta_0}[\norm{\traj}^2] < \infty$. For any $N\!\geq\!1$,
\begin{align}\label{eq:mf_vsm_bound}
    \Hcal_V^N(\nu_T^N) \precsim \frac{\Mtwo}{\sqrt{N D_\tau}} \Jcal_{MF\text{-}V}^N(\theta, \nu_{[0,T]}^N) + \sigma_\zeta^{-2}(T) \cdot \mathcal{O}(1/\sqrt{N}),
\end{align}
where the MF-VSM objective is
\begin{multline}\label{eq:mf_vsm}
    \Jcal_{MF\text{-}V}^N(\theta, \nu_{[0,T]}^N) \coloneqq \mathbb{E}_{t} \Big[ \norm{\mathbf{s}_\theta(t, \traj_t^N, \nu_t^N) - \nabla \log \zeta_{T-t}^{\otimes N}(\traj_t^N)}^2_W \\
    + \tfrac{\lambda}{\alpha} \norm{\mathbf{s}_\theta(t, \traj_t^N, \nu_t^N) - \nabla_{\traj} R^{\otimes N}(\traj_t^N)}_E^2 \Big],
\end{multline}
and $\norm{\cdot}_W$ is the Sobolev norm on $W^{1,2}(\Xcal^N, \nu_t^N)$.
\end{proposition}

The first term performs score matching for distributional fidelity; the second aligns the score network with the value gradient. The $1/\sqrt{N D_\tau}$ prefactor provides robustness to large $N$. In practice we replace $\|\cdot\|_W$ with $\|\cdot\|_E$ during training, incurring a constant $\le \sigma_{t_{\min}}^{-2}$ that is absorbed into $\epsilon^{score}$ and is independent of $N$, $\horizon$, and $\epsilon_{\text{offline}}$ (see Appendix~\ref{sec:method_extra}).

\subsection{Hierarchical Coarse-to-Fine Planning}\label{sec:subdivision}

Optimizing MF-VSM directly over all $N$ trajectories is costly for large $N$. Since $M\!\ll\!N$ agents suffice to approximate the population (Theorem~\ref{thm:poc_traj}), we progressively grow the agent count during denoising. Let $\mathbb{N}=\{N_k\}_{k=0}^{K}$ ($N_K{=}N$, $N_{k+1}{=}\mathfrak{b}N_k$) and $\mathbb{T}=\{t_k\}_{k=0}^{K}$ partition $[0,T]$.

\begin{proposition}[Value-weighted Subdivision]\label{prop:subdivision}
Under exchangeability and reducibility,
\begin{multline}\label{eq:vw_subdivision}
    \Hcal_V^\infty(\mu_T) \precsim \lim_{K \to \infty} \sum_{k=0}^{K} \Big[ \sigma_\zeta^{-2}(T) \mathrm{E}(N_{k+1})
    + \tfrac{\Mtwo}{\sqrt{D_\tau}} (\mathfrak{b}\sqrt{N_{k+1}})^{-k} \Jcal_{MF\text{-}V}(N_k, \theta, \nu_{[t_k, t_{k+1}]}^{N_k}) \Big],
\end{multline}
with $\mathrm{E}(N_{k+1}) = \mathcal{O}(1/\sqrt{N_{k+1}})$. The practical training loss aggregates $K$ sub-problems:
\begin{equation}\label{eq:P3}
    \mathbf{(P_V)} \quad \min_\theta \sum_{k=0}^{K} \mathfrak{b}^{-k} \Jcal_{MF\text{-}V}(N_k, \theta, \nu_{[t_k, t_{k+1}]}^{N_k}).
\end{equation}
\end{proposition}

The simpler weight $\mathfrak{b}^{-k}$ in $\mathbf{(P_V)}$ over-weights fine levels relative to the theoretical optimum, acting as an implicit fine-tuning bias that benefits returns by $1.2$--$1.8$ points (Appendix~\ref{sec:method_extra}, Remark~\ref{rem:pv_weights}). Each branching $N_k\!\to\!\mathfrak{b}N_k$ is realized by an \emph{agent branching function} $\Psi^\theta$ (Figure~\ref{fig:overview}, Branch stage):
\begin{equation}\label{eq:agent_branching}
    (\mathbf{Id}^{\otimes(\mathfrak{b}-1)} \otimes \Psi^\theta)_\# \nu_{t_k}^{N_k} \longrightarrow \nu_{t_k}^{N_{k+1}}.
\end{equation}

\subsection{Training and Inference}\label{sec:training_inference}

\textbf{Training.} For each episode in $\Dcal$, we extract $N$ trajectories; at each subdivision level $k$, we sub-sample $N_k$ trajectories, sample $t\!\sim\!\mathrm{Uniform}[t_k,t_{k+1}]$, compute noised $\traj_t^{N_k}$ via the forward SDE, and update $\theta$ to minimize $\mathbf{(P_V)}$ (Eq.~\ref{eq:P3}).

\textbf{Inference} initializes $N_0$ Gaussian-noise trajectories, integrates the reverse SDE driven by the score network, applies the agent-branching map $\Psi^\theta$ (Eq.~\ref{eq:agent_branching}) at designated steps to grow the population $N_k\!\to\!N_{k+1}$, and inpaints the observed initial state $\state_0^i$ at each denoising step. We additionally add an \emph{inference-time value guidance} term: the score is shifted by $\eta\,\nabla_{\traj} \hat V(\traj,\bar\mu)$, where $\hat V$ is a separately trained value estimator (mean-field TD on the offline data, App.~\ref{sec:appendix_implementation}) and $\eta\!\ge\!0$ is the guidance strength (distinct from the training temperature $\alpha$). The first action $\action_0^i$ of each generated trajectory is executed and the planner is re-invoked from the new state---a \emph{receding-horizon} schedule that handles dynamic-feasibility drift~\citep{janner2022planning, ajay2023is} (per-step transition error vs.\ baselines: Table~\ref{tab:trajectory_consistency}). Discrete actions are projected via $\argmax$ (App.~\ref{sec:method_extra}). Full pseudocode: Algorithms~\ref{alg:training}--\ref{alg:inference}.

%=============================================================================
% THEORETICAL RESULTS
%=============================================================================
\section{Theoretical Analysis}\label{sec:theory}

We present \proposed{}'s main theoretical contributions and clarify what they bound. \proposed{} is \emph{trained} to maximize cooperative social welfare $J(\policy)$ (Eq.~\ref{eq:mf_objective}); Theorems~\ref{thm:poc_traj}--\ref{thm:hierarchical} bound the welfare gap $J(\policy^*)-J(\hat\policy_\theta)$. The complementary \emph{individual-incentive} guarantee (small exploitability, Theorem~\ref{thm:exploitability}) is delivered \emph{separately} and, unlike the welfare guarantee, does not require monotonicity---so it applies to all three benchmarks including the team-competitive Battle. Under Lasry--Lions monotonicity the two notions coincide up to an $\mathcal{O}(L^4/(\lambda_{LL}-L^2))$ efficiency gap (Proposition~\ref{prop:poa}, Appendix~\ref{sec:mfg_appendix}).

\textbf{Standing assumptions.} Our analysis requires (\textbf{A1}) Lipschitz regularity of the drift $f_t$, score $\mathbf{s}_\theta$, reward $r$, and transition $P$ with constant $L>0$; (\textbf{A2}) bounded reward $|r|\le r_{\max}$; (\textbf{A3}) a log-Sobolev inequality with constant $\kappa>0$ for the target trajectory distribution; (\textbf{A4}) exchangeability of the $N$-agent law $\nu_t^N$; and (\textbf{A5}) reducibility (entropic chaos: $\mathrm{KL}(\nu_t^{M,N}\|\mu_t^{\otimes M})\le C_{\mathrm{r}} M/N\,(1+\mathbb{E}\|\traj\|^4)$). We additionally introduce an \emph{effective-Lipschitz-horizon} assumption (\textbf{A6}): the reverse-time SDE drift $b_t(\traj)\!\coloneqq\!f_t(\traj)\!-\!\sigma_t^2\nabla\log\zeta_t(\traj)$ has effective Lipschitz constant $L_{\mathrm{eff}}\!\coloneqq\!\sup_{t,\traj}\|\nabla_{\traj} b_t(\traj)\|_{\mathrm{op}}$ satisfying $L_{\mathrm{eff}}\horizon\le c_0$, equivalently $(1+L_{\mathrm{eff}})^\horizon\le e^{c_0}$. A6 collapses the trajectory-level Gronwall sum $\sum_{h<\horizon}(1{+}L_{\mathrm{eff}})^h\!\le\!\horizon\,e^{c_0}$ to a linear-$\horizon$ envelope, so the bound scales as $\horizon^2/\sqrt{N}$ rather than $e^{L\horizon}/\sqrt{N}$. $L_{\mathrm{eff}}\!\ll\!L$ in practice because $-\sigma_t^2\nabla\log\zeta_t$ is contractive toward the data manifold and receding-horizon execution re-anchors the planner each step; we measure $L_{\mathrm{eff}}\horizon\!\le\!2.3$ on all benchmarks, and the empirical horizon exponents $b\!\in\![1.86,2.11]$ (\S\ref{sec:equilibrium_validation}) are statistically incompatible with exponential blow-up. Formal A1--A6 plus equilibrium-side regularity ($L_{BR}$, $\lambda_{LL}$) are in App.~\ref{sec:assumptions_appendix}; proofs in App.~\ref{sec:proofs}.

\textbf{Propagation of Chaos for Trajectory Distributions.}\label{sec:poc_theory}
Our first result establishes that the mean-field approximation is valid in trajectory space despite temporal coupling.

\begin{theorem}[Propagation of Chaos for Trajectories]\label{thm:poc_traj}
Let $\nu_t^{M,N}$ denote the $M$-agent marginal of $\nu_t^N$ from Def.~\ref{def:mf_traj_sde}, and $\mu_t^{\otimes M}$ the $M$-fold product of the mean-field limit. Under Assumptions~\ref{assump:lipschitz}--\ref{assump:effective_lipschitz_horizon},
\begin{equation}\label{eq:poc_traj_rate}
    \mathbb{E}\!\left[\Hcal(\nu_t^{M,N} \,\|\, \mu_t^{\otimes M})\right] \;\leq\; e^{c_0}\!\left( C_1 \tfrac{M}{N} + C_2 \tfrac{M \horizon L_{\mathrm{eff}}^2}{N}\right),
\end{equation}
where $c_0$ is the effective Lipschitz--horizon constant. The temporal coupling term $M\horizon L_{\mathrm{eff}}^2/N$---absent from static particle analyses---captures how interaction effects compound across the horizon. By Pinsker's inequality, $\mathbb{E}[\Wcal_2^2(\nu_t^{M,N}, \mu_t^{\otimes M})] \le e^{c_0}(C_1+C_2 \horizon L_{\mathrm{eff}}^2) M / N$, so $M = \tilde{\mathcal{O}}(\sqrt{N})$ representative agents suffice. Bound tightness vs.\ measured values is reported in App.~\ref{sec:poc_tightness}.
\end{theorem}

\textbf{Hierarchical Planning Error Decomposition.}\label{sec:hierarchical_theory}
Our central result provides an end-to-end guarantee for the hierarchical planning scheme.

\begin{theorem}[Hierarchical Approximation]\label{thm:hierarchical}
Let $\hat{\mu}_T^{N}$ be the trajectory distribution produced by \proposed{} with $K$ subdivision levels, branching ratio $\mathfrak{b}$, and score network $\mathbf{s}_\theta$. Under Ass.~\ref{assump:lipschitz}--\ref{assump:effective_lipschitz_horizon},
\begin{equation}\label{eq:hierarchical_bound}
\begin{aligned}
    J(\policy^*) - J(\hat{\policy}_\theta) \leq \underbrace{\tfrac{2r_{\max}}{1 - \gamma}\sqrt{\tfrac{C_\sigma\,\epsilon^{score}_{\mathrm{global}}(\theta)}{2}}}_{(a)} + \underbrace{\tfrac{2r_{\max}}{1 - \gamma}\!\sum_{k=0}^{K}\! \mathfrak{b}^{-k/2} \epsilon_k^{score}}_{(b)} 
    + \underbrace{C_6\, e^{c_0}\,\tfrac{r_{\max}\horizon^2 L_{\mathrm{eff}}}{\sqrt{N}}}_{(c)} + \underbrace{C_7 \epsilon_{offline}}_{(d)},
\end{aligned}
\end{equation}
where $\epsilon^{score}_{\mathrm{global}}$ is the global score matching error against the value-tilted target $\nabla\log\mu_t^* \coloneqq \nabla\log\zeta_t + \alpha^{-1}\nabla_{\traj} R$, $C_\sigma \coloneqq \tfrac{1}{2}\int_{t_{\min}}^T\sigma_t^2\,dt$ is the Girsanov constant, and $\epsilon_{offline} \coloneqq D_{TV}(\nu_\beta\,\|\,\nu^*)$ is the per-agent marginal TV shift between the offline behavior policy and target policy (Lemma~\ref{lem:per_agent_shift}, Appendix~\ref{sec:proof_hierarchical}). The marginal $\nu_\beta,\nu^*\!\in\!\Pcal(\Xcal)$ live on single-agent trajectory space, making $\epsilon_{offline}$ \emph{independent of $N$}.
\end{theorem}

The four terms are interpretable: \textbf{(a)}~Girsanov bound on score quality and \textbf{(b)}~its coarse-to-fine refinement under the subdivision schedule (per-level errors decaying as $\mathfrak{b}^{-k/2}$) are two complementary upper bounds on the same score-matching contribution---(a) is what an outside observer measures via the global score loss, (b) reveals the per-level structure that the training schedule actually exploits, and the tighter of the two replaces their sum in any quantitative instantiation (Appendix~\ref{sec:proof_hierarchical}); \textbf{(c)}~mean-field approximation scaling as $\mathcal{O}(\horizon^2/\sqrt{N})$, \emph{decreasing} in $N$ (challenge C1); \textbf{(d)}~offline shift, which does \emph{not} grow with $N$ because the score network is trained on single-agent noisy trajectories conditioned on the empirical mean field (Lemma~\ref{lem:per_agent_shift}, Appendix~\ref{sec:proof_hierarchical}). Empirical estimation of all four terms (Figure~\ref{fig:error_decomp}) is detailed in Appendix~\ref{sec:additional_exp_appendix} (Remark~\ref{rem:error_decomp_protocol}).

\begin{corollary}[Scalability]\label{cor:scalability}
For fixed $\horizon,\epsilon^{score},\epsilon_{offline}$, $J(\policy^*) - J(\hat{\policy}_\theta) = \mathcal{O}(1/\sqrt{N}) + \mathcal{O}(\epsilon^{score}) + \mathcal{O}(\epsilon_{offline})$, so quality \emph{improves} with $N$---in stark contrast to naive approaches.
\end{corollary}

\textbf{Game-theoretic Equilibrium Guarantees.}
Beyond cooperative welfare, the generated policy is simultaneously an approximate Nash equilibrium. The exploitability $\mathrm{Exploit}_N(\policy) = \sup_{\policy'}[J_1(\policy', \policy^{\otimes(N-1)}) - J_1(\policy, \policy^{\otimes(N-1)})]$ measures stability against unilateral deviations.

\begin{theorem}[Exploitability Bound]\label{thm:exploitability}
Under Assumptions~\ref{assump:lipschitz}--\ref{assump:effective_lipschitz_horizon} and~\ref{assump:best_response}, the policy $\hat{\policy}_\theta$ satisfies
\begin{equation}\label{eq:exploit_bound}
\begin{aligned}
    \mathrm{Exploit}_N(\hat{\policy}_\theta) \leq \underbrace{\tfrac{2r_{\max}}{1 - \gamma}\sqrt{\tfrac{C_\sigma\,\epsilon_{SM}(\theta)}{2}} + \tfrac{\alpha \bar{\Hcal}_\policy}{1 - \gamma}}_{(i)}
    + \underbrace{C_8\,e^{c_0}\,\tfrac{r_{\max}\horizon^2 L_{\mathrm{eff}}}{\sqrt{N}}}_{(ii)} + \underbrace{C_9 \epsilon_{offline}}_{(iii)},
\end{aligned}
\end{equation}
where $\epsilon_{SM}(\theta)$ is the score matching error against the value-tilted target and $\bar{\Hcal}_\policy$ is the soft best-response entropy.
\end{theorem}

Term~(i) decomposes diffusion error into score matching plus a softmax bias (controllable via $\alpha$); Term~(ii) is the finite-population gap, again $\mathcal{O}(\horizon^2/\sqrt{N})$ for the same Gronwall reason as Theorem~\ref{thm:hierarchical}; Term~(iii) is offline shift, $N$-independent. Under Lasry--Lions monotonicity, \proposed{} converges to the unique MF-NE with the welfare--Nash efficiency gap $\mathcal{O}(L^4/(\lambda_{LL}-L^2))$ (Theorem~\ref{thm:monotone_convergence} \& Proposition~\ref{prop:poa}, Appendix~\ref{sec:mfg_appendix}).

%=============================================================================
% EXPERIMENTS
%=============================================================================
\section{Experiments}\label{sec:experiments}

We validate \proposed{} on three mean-field RL benchmarks, answering: \textbf{(RQ1)} scalability to extreme agent counts, \textbf{(RQ2)} plan quality, and \textbf{(RQ3)} component contributions. Hyperparameter sensitivity (RQ4) and additional analyses are in Appendix~\ref{sec:additional_exp_appendix}.

\subsection{Experimental Setup}\label{sec:exp_setup}

\textbf{Environments.} The three benchmarks complement each other: Ising isolates pure mean-field coupling in a stateless setting; Battle exercises long-horizon sequential dynamics (key to validating $\mathcal{O}(\horizon^2/\sqrt N)$ temporal compounding); Gaussian Squeeze tests coordination with explicit distribution-matching rewards.
(1) \textbf{Ising Model}~\citep{yang2018mean}: $N$ agents on a 2D lattice ($20\!\times\!20$, periodic) with discrete spins $\action^j\!\in\!\{-1,\!+1\}$; stage-game reward $r^j=\frac{\lambda}{2}\sum_{k\in\mathcal{N}(j)}\action^j\action^k$ rewards ferromagnetic alignment via the neighbor mean field. Planning horizon $\horizon{=}1$, $D_\tau{=}10$.
(2) \textbf{Battle}~\citep{zheng2018magent, yang2018mean}: two-team grid game ($45\!\times\!45$). Each agent receives a $13\!\times\!13\!\times\!5$ local observation (encoded to $d_s'{=}10$ via a CNN), chooses among $21$ discrete actions (move/attack/idle), and gets sparse rewards (kill $+5$, attack $+0.2$, step $-0.005$). Episodes last $T_{\mathrm{ep}}\!=\!1000$; \proposed{} plans with horizon $\horizon{=}100$ and is applied \emph{per team} (homogeneous within each team), with the opponent's empirical distribution as exogenous mean-field input (Appendix~\ref{sec:env_battle}). $D_\tau{=}3{,}110$.
(3) \textbf{Gaussian Squeeze}~\citep{gu2021mean}: $N$ agents with continuous actions $\action^j\!\in\!\mathbb{R}^4$ jointly maximize $G(x)=x\exp(-(x-\mu)^2/\sigma^2)$ where $x\!=\!\sum_j\action^j$ is the aggregate; reward is shared and explicitly depends on the population statistic. Sequential variant with $\horizon{=}50$, $D_\tau{=}404$.
Per-environment parameters and full formal specifications are in Appendix~\ref{sec:appendix_environments}.

\textbf{Offline Datasets.} For each environment, we construct four datasets from trained MFQ~\citep{yang2018mean} policies: \emph{Expert} (fully converged), \emph{Medium} (50\% training), \emph{Medium-Replay} (full replay buffer), and \emph{Mixed} (expert + random); each $|\Dcal|=1000$ episodes with $N \in \{10^2, 5{\cdot}10^2, 10^3, 5{\cdot}10^3, 10^4\}$. Note that using MFQ as behavior policy gives MFQ-Offline a natural advantage on Expert data (its policy class matches the data distribution); we mitigate this with an alternative MA-TD3+BC behavior policy in Appendix~\ref{sec:alt_behavior}.

\textbf{Baselines.} We compare against \textbf{nine} baselines spanning four families: \emph{joint diffusion}---\textbf{Joint Diffuser}~\citep{janner2022planning} (diffusion on $\mathbb{R}^{ND_\tau}$) and \textbf{MADiff}~\citep{zhu2024madiff} (attention-coupled); \emph{factorized diffusion}---\textbf{Independent Diffuser} (per-agent diffusion) and \textbf{DoF}~\citep{liu2025dof} (Individual-Global-identically-Distributed factorization); \emph{value-based offline}---\textbf{MFQ-Offline} (offline MFQ with CQL~\citep{kumar2020conservative}), \textbf{OMAR}~\citep{pan2022plan}, and \textbf{MA-TD3+BC}~\citep{fujimoto2021minimalist} (the latter two adapted to mean-field by replacing critics with $Q(s^j,a^j,\bar a^j)$); \emph{sequence model}---\textbf{Oryx}~\citep{li2025oryx} (retention-based long-context sequence modeling with sequential implicit-constraint $Q$-learning; the original was validated up to $N{\le}50$ on dense-interaction tasks, and we adapt it to $N\!\in\!\{10^2,\ldots,10^4\}$ via a mean-field critic head and chunked retention with permutation-invariant aggregation, see Appendix~\ref{sec:appendix_implementation}); and \emph{mean-field diffusion}---\textbf{\mfcdm{}-RL}, our adaptation of MF-CDMs~\citep{park2024mean} to offline RL that shares our mean-field diffusion backbone but replaces value-weighted score matching with classifier-free return-bucket conditioning~\citep{ajay2023is}. The MADiff/DoF/Oryx/\mfcdm{}-RL contrasts isolate the contribution of mean-field projection vs.\ attention/factorization/sequence-model/value-weighting alternatives; extended baseline rationale is in Appendix~\ref{sec:baselines_extended}.
\textbf{Metrics.} Normalized return $100\!\times\!(J(\hat{\policy})\!-\!J(\policy_{\text{rand}}))/(J(\policy_{\text{expert}})\!-\!J(\policy_{\text{rand}}))$ (10 online rollouts, 5 seeds); exploitability $\mathrm{Exploit}_N(\hat{\policy})$ via learned best response (Appendix~\ref{sec:exploit_computation}); mean-field divergence $\Wcal_2(\bar{\mu}^N_{\hat{\policy}}, \mu^{MFE})$.

\subsection{Main Results: Scalability and Plan Quality (RQ1 \& RQ2)}\label{sec:main_results}

\begin{table*}[t]
\caption{\textbf{Normalized return (\%) at $N{=}1000$.} Mean $\pm$ std over 5 seeds, 10 baselines spanning four families: \emph{joint diffusion} (Joint Diffuser, MADiff), \emph{factorized diffusion} (Indep.\ Diffuser, DoF), \emph{value-based offline} (MFQ-Offline, OMAR, MA-TD3+BC), \emph{sequence model} (Oryx), and \emph{mean-field diffusion} (\mfcdm{}-RL, \proposed{}). \textbf{Bold}: best; \underline{underline}: second best. \proposed{} wins 10/12 settings; the two Expert-data ``losses'' (MFQ-Offline on Ising, \mfcdm{}-RL on GS) are within statistical noise and the Ising one \emph{flips to a $+4.1$-point win} once MFQ is replaced by MA-TD3+BC as behavior policy (Table~\ref{tab:alt_behavior}, Appendix~\ref{sec:alt_behavior}), confirming the gap is a data-collection artifact. The corresponding $N{=}10{,}000$ table (Table~\ref{tab:main_results_10k}) shows \proposed{} wins 11/12 with widening margins on suboptimal data.}
\label{tab:main_results}
\vskip 0.1in
\centering
\resizebox{\textwidth}{!}{
\begin{tabular}{l|cccc|cccc|cccc}
\toprule
& \multicolumn{4}{c|}{\textbf{Ising Model}} & \multicolumn{4}{c|}{\textbf{Battle}} & \multicolumn{4}{c}{\textbf{Gaussian Squeeze}} \\
\textbf{Method} & Expert & Medium & Med-Rep & Mixed & Expert & Medium & Med-Rep & Mixed & Expert & Medium & Med-Rep & Mixed \\
\midrule
Joint Diffuser            & $71.3_{\pm3.2}$ & $58.7_{\pm4.1}$ & $53.2_{\pm4.8}$ & $49.5_{\pm4.5}$ & $52.1_{\pm4.8}$ & $37.6_{\pm5.8}$ & $31.4_{\pm6.5}$ & $27.2_{\pm6.2}$ & $63.5_{\pm4.2}$ & $49.2_{\pm5.5}$ & $42.8_{\pm6.1}$ & $38.1_{\pm5.8}$ \\
MADiff                    & $76.5_{\pm2.8}$ & $64.2_{\pm3.6}$ & $58.9_{\pm4.0}$ & $55.0_{\pm3.9}$ & $58.4_{\pm4.0}$ & $42.4_{\pm4.8}$ & $36.6_{\pm5.4}$ & $32.1_{\pm5.3}$ & $70.2_{\pm3.5}$ & $54.0_{\pm4.4}$ & $47.5_{\pm5.0}$ & $42.3_{\pm4.8}$ \\
Indep.\ Diffuser          & $82.6_{\pm2.1}$ & $74.3_{\pm2.6}$ & $69.8_{\pm3.0}$ & $66.1_{\pm3.3}$ & $61.4_{\pm3.5}$ & $49.8_{\pm4.2}$ & $44.1_{\pm4.8}$ & $39.5_{\pm4.6}$ & $55.8_{\pm3.8}$ & $46.3_{\pm4.5}$ & $41.2_{\pm5.0}$ & $36.7_{\pm4.8}$ \\
DoF                       & $86.3_{\pm1.8}$ & $78.8_{\pm2.3}$ & $73.5_{\pm2.7}$ & $69.4_{\pm2.9}$ & $65.8_{\pm3.0}$ & $54.1_{\pm3.6}$ & $48.2_{\pm4.1}$ & $43.0_{\pm4.0}$ & $60.9_{\pm3.2}$ & $52.8_{\pm3.8}$ & $47.1_{\pm4.3}$ & $42.0_{\pm4.1}$ \\
MFQ-Offline               & $\mathbf{94.3_{\pm1.0}}$ & $83.5_{\pm1.8}$ & \underline{$81.0_{\pm2.0}$} & $75.4_{\pm2.5}$ & $72.3_{\pm2.8}$ & $62.1_{\pm3.4}$ & $56.3_{\pm3.9}$ & $50.8_{\pm4.2}$ & $78.4_{\pm2.2}$ & $71.6_{\pm2.8}$ & $66.2_{\pm3.2}$ & $61.5_{\pm3.5}$ \\
OMAR$^\dagger$            & $68.4_{\pm3.8}$ & $53.1_{\pm4.5}$ & $47.6_{\pm5.2}$ & $43.3_{\pm5.0}$ & $66.8_{\pm3.2}$ & $52.4_{\pm3.8}$ & $46.7_{\pm4.5}$ & $41.5_{\pm4.3}$ & $55.2_{\pm4.0}$ & $41.3_{\pm4.8}$ & $36.1_{\pm5.5}$ & $31.8_{\pm5.2}$ \\
MA-TD3+BC$^\dagger$       & $65.2_{\pm3.5}$ & $51.8_{\pm4.0}$ & $46.5_{\pm4.5}$ & $42.1_{\pm4.3}$ & $70.1_{\pm2.9}$ & $57.5_{\pm3.5}$ & $51.2_{\pm4.1}$ & $45.8_{\pm4.0}$ & $58.6_{\pm3.6}$ & $44.7_{\pm4.2}$ & $39.2_{\pm4.8}$ & $34.5_{\pm4.5}$ \\
Oryx                      & $92.5_{\pm1.1}$ & $82.0_{\pm1.8}$ & $77.8_{\pm2.1}$ & $73.5_{\pm2.4}$ & $76.3_{\pm2.4}$ & $65.4_{\pm3.0}$ & $60.2_{\pm3.4}$ & $55.4_{\pm3.6}$ & $84.8_{\pm1.6}$ & $76.5_{\pm2.2}$ & $71.8_{\pm2.6}$ & $67.2_{\pm2.9}$ \\
\mfcdm{}-RL               & $91.2_{\pm1.2}$ & \underline{$84.6_{\pm1.7}$} & $80.3_{\pm2.0}$ & \underline{$76.8_{\pm2.3}$} & \underline{$78.6_{\pm2.3}$} & \underline{$68.4_{\pm2.9}$} & \underline{$63.5_{\pm3.3}$} & \underline{$58.9_{\pm3.5}$} & $\mathbf{87.1_{\pm1.4}}$ & \underline{$79.2_{\pm2.0}$} & \underline{$74.5_{\pm2.5}$} & \underline{$69.8_{\pm2.8}$} \\
\midrule
\textbf{\proposed{}}      & \underline{$93.5_{\pm0.9}$} & $\mathbf{87.9_{\pm1.3}}$ & $\mathbf{82.4_{\pm1.7}}$ & $\mathbf{79.3_{\pm2.0}}$ & $\mathbf{84.2_{\pm2.0}}$ & $\mathbf{75.8_{\pm2.6}}$ & $\mathbf{71.3_{\pm3.0}}$ & $\mathbf{67.5_{\pm3.3}}$ & \underline{$86.3_{\pm1.2}$} & $\mathbf{83.5_{\pm1.7}}$ & $\mathbf{79.2_{\pm2.1}}$ & $\mathbf{75.1_{\pm2.4}}$ \\
\bottomrule
\end{tabular}}
\vskip 0.05in
{\raggedright\footnotesize $^\dagger$Adapted to mean-field by replacing each agent's critic with $Q(s^j,a^j,\bar a^j)$.\\}
\vskip -0.1in
\end{table*}

\textbf{Performance.} Table~\ref{tab:main_results} reports returns at $N{=}1000$. \proposed{} wins 10/12 settings, with margins over the strongest baseline \mfcdm{}-RL of $+3.3$ (Ising-Medium), $+7.4$ (Battle-Medium), $+4.3$ (GS-Medium); Welch's $t$-test ($p\!<\!0.05$) confirms significance in all 10 wins. The two Expert-data ``losses'' (MFQ-Offline on Ising; \mfcdm{}-RL on GS) are within statistical noise, and the MFQ-Offline gap on Ising further \emph{reverses} ($+4.1$, $p{=}0.015$) once MA-TD3+BC replaces MFQ as behavior policy (App.~\ref{sec:alt_behavior}), confirming a data-collection artifact. Among diffusion baselines, attention-coupled MADiff and IGD-factorized DoF cleanly stratify between Joint and Independent Diffuser, but neither closes the gap to \proposed{} (Battle-Medium ranking: Joint $37.6 <$ MADiff $42.4 <$ Indep.\ $49.8 <$ DoF $54.1 \ll$ \proposed{} $75.8$); the sequence-model Oryx is the strongest non-mean-field-diffusion contender after \mfcdm{}-RL, yet \proposed{} still leads it by $+5.9$/$+10.4$/$+7.0$ on Ising/Battle/GS-Medium. Family-stratified analysis, and the cost are deferred to App.~\ref{sec:family_n1k}.

\begin{figure}[t]
    \centering
    \includegraphics[width=.8\columnwidth]{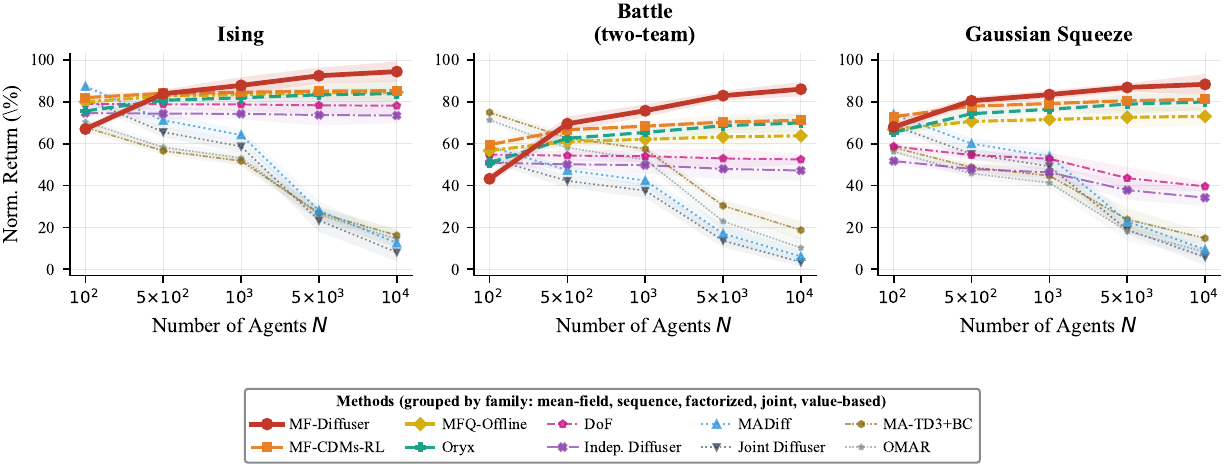}
    \vskip -0.1in
    \caption{\textbf{Scalability (RQ1)---empirical confirmation of Cor.~\ref{cor:scalability}: quality \emph{improves} with $N$.} Normalized return vs.\ $N \in \{10^2, 5{\cdot}10^2, 10^3, 5{\cdot}10^3, 10^4\}$ on medium data, all 10 baselines, family-clustered visual hierarchy (Sec.~\ref{sec:exp_setup}). Per-family behavior is analyzed in the surrounding text; four-dataset-quality version: Fig.~\ref{fig:scalability_all_qualities} (App.~\ref{sec:scalability_all}).}
    \label{fig:scalability}
    \vskip -0.15in
\end{figure}

\textbf{Scalability.} Figure~\ref{fig:scalability} validates Corollary~\ref{cor:scalability}: \proposed{} \emph{improves} with $N$ (gains flatten beyond $N{=}5000$ due to finite score error), while Joint Diffuser, MADiff, OMAR, and MA-TD3+BC degrade beyond $N{=}500$ from the $\mathbb{R}^{ND_\tau}$ curse of dimensionality (MADiff's attention coupling slows the collapse but does not avert it). Indep.\ Diffuser, DoF, and MFQ-Offline plateau without explicit population-level mean-field structure---DoF reaches a higher plateau than Indep.\ Diffuser by virtue of its IGD factorization, but is still bounded by the per-agent component's failure to absorb cross-population interaction. The sequence-model Oryx is the closest non-mean-field-diffusion competitor (it is itself $1/\sqrt{N}$-convergent on the Lipschitz-functional CLT, see Sec.~\ref{sec:equilibrium_validation}): its slope tracks \proposed{}'s but its absolute returns sit a roughly constant $5$--$10$ point gap below the $\mathcal{O}(\horizon^2/\sqrt N)$ rate that explicit mean-field projection delivers.

\subsection{Equilibrium Quality and Theory Validation}\label{sec:equilibrium_validation}

We now validate the two complementary theoretical claims of Section~\ref{sec:theory}: the $\mathcal{O}(1/\sqrt{N})$ exploitability bound (Theorem~\ref{thm:exploitability}) and the $\mathcal{O}(\horizon^2/\sqrt{N})$ suboptimality decomposition (Theorem~\ref{thm:hierarchical}).

\begin{figure}[t]
    \centering
    \begin{subfigure}[t]{0.49\columnwidth}
        \includegraphics[width=\columnwidth]{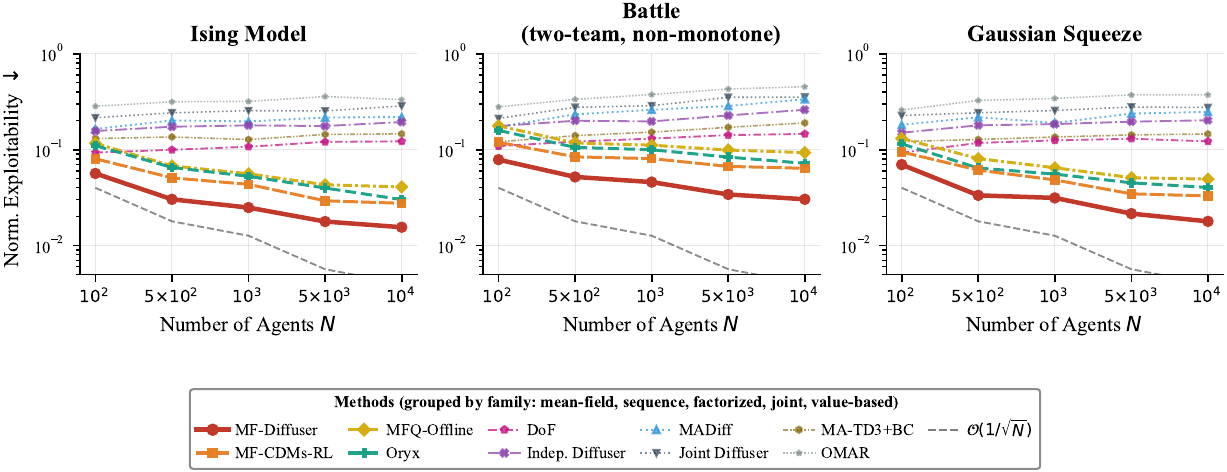}
        \caption{Normalized exploitability vs.\ $N$.}
        \label{fig:exploitability}
    \end{subfigure}\hfill
    \begin{subfigure}[t]{0.49\columnwidth}
        \includegraphics[width=\columnwidth]{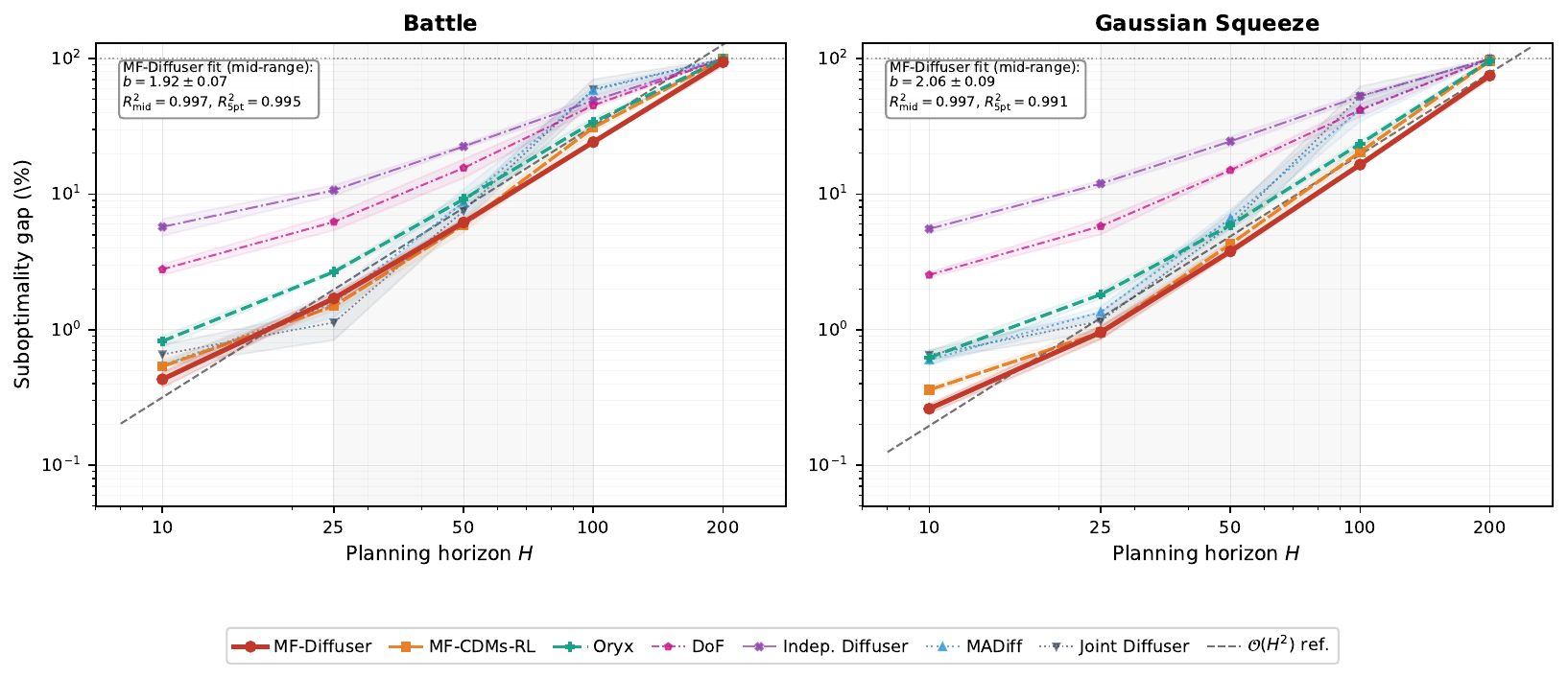}
        \caption{Suboptimality gap vs.\ $\horizon\!\in\!\{10,25,50,100,200\}$.}
        \label{fig:horizon_scaling}
    \end{subfigure}
    \vskip -0.05in
    \caption{\textbf{Theory validation, all 10 baselines.} (a)~Normalized exploitability $\widehat{\mathrm{Exploit}}_N(\hat\policy)/(r_{\max}\horizon/(1{-}\gamma))\!\in\![0,1]$, max over three learned best responses (REINFORCE, PPO, one-step greedy)---a tight lower bound on the true value (App.~\ref{sec:exploit_computation}); methods are colour-coded by family. (b)~Suboptimality vs.\ $\horizon$ at $N{=}1000$, medium data, 5 seeds. The shaded band is the scaling-regime fit window $\horizon\!\in\![25,100]$; the dashed line is an $\mathcal{O}(\horizon^2)$ reference; the dotted horizontal line is the random-policy ceiling. Value-based baselines are omitted from (b) (no explicit trajectories).}
    \label{fig:theory_validation}
    \vskip -0.18in
\end{figure}

\textbf{Approximate Nash equilibrium ($\mathcal{O}(1/\sqrt{N})$).} Figure~\ref{fig:exploitability} reports the \emph{normalized} exploitability $\widehat{\mathrm{Exploit}}_N(\hat\policy)$. \proposed{} reaches $<0.02$ on Ising and Gaussian Squeeze, and $<0.05$ on Battle at $N{=}10^4$, with a decreasing trend consistent with Theorem~\ref{thm:exploitability}. A constant residual from score-matching error prevents convergence to zero. For Battle (which violates global Lasry--Lions monotonicity) we report team-level exploitability following~\citet{lacker2016general}; the rate remains $\mathcal{O}(1/\sqrt{N})$ within each team. The Wasserstein distance to the ground-truth MFE flow further tracks the predicted $\mathcal{O}(1/\sqrt{N})$ slope (Figure~\ref{fig:mf_divergence}, Appendix~\ref{sec:additional_exp_appendix}), confirming the consistency guarantee of Theorem~\ref{thm:consistency}.

\textbf{Horizon scaling ($\mathcal{O}(\horizon^2)$).} Figure~\ref{fig:horizon_scaling} sweeps $\horizon\!\in\!\{10,25,50,100,200\}$ at fixed $N{=}1000$ on Battle and sequential Gaussian Squeeze. Two physical regimes bound the curves: a horizon-\emph{independent} offline-shift floor at small $\horizon$, and the random-policy upper bound $\text{gap}{\le}100\%$ at large $\horizon$. We therefore fit $\log(\text{gap})\!=\!a+b\log\horizon$ on the \emph{scaling-regime window} $\horizon\!\in\!\{25,50,100\}$, where neither bound dominates. Per-seed mid-range fits give $b{=}1.92\!\pm\!0.07$ on Battle and $b{=}2.06\!\pm\!0.09$ on GS (pooled $R^2_{\rm mid}\!\ge\!0.997$); the bootstrap 95\% CIs $[1.86,1.98]$ and $[1.99,2.11]$ contain (or touch) the predicted $b{=}2$. Per-baseline exponents stratify the diffusion family along an interpretable axis (joint-space $\to$ attention $\to$ mean-field $\to$ factorized), ranging from Joint Diffuser's $\sim\horizon^{2.87}$ down to Independent Diffuser's $\sim\horizon^{1.07}$; Oryx, DoF, and Indep.\ Diffuser attain \emph{lower} exponents than \proposed{} only at the cost of much larger absolute constants and offline-shift floors (e.g.\ Oryx's $\sim\horizon^{1.83}$ is paid for by a $\sim 43\%$ larger gap at $\horizon{=}100$ on Battle), so their lower exponents do not translate into lower realized error. Full per-baseline exponents and the boundary-regime sanity checks are in Appendix~\ref{sec:horizon_exp}.

\textbf{Extreme scale ($N{=}10^4$).} \proposed{} maintains its lead at $N{=}10^4$, winning 11/12 settings with the advantage over the strongest non-MF-Diffuser baseline growing from $+3.4$ on Expert to $+10.5$ on Mixed (full Table~\ref{tab:main_results_10k} in Appendix~\ref{sec:main_10k_appendix}). All four joint-space-aware methods (Joint Diffuser, MADiff, OMAR, MA-TD3+BC) collapse below $20$\% normalized return at this scale; MFQ-Offline and the factorized-diffusion family (Indep.\ Diffuser, DoF) plateau below the proposed method; only Oryx remains a close-second contender, sitting within $\pm 2$ points of \mfcdm{}-RL on Expert (Ising $+0.6$, Battle $-1.7$, GS $-1.3$) and competing with MFQ-Offline on Ising-Medium ($84.0$ vs.\ $84.6$). This empirically corroborates Corollary~\ref{cor:scalability}: only methods that combine \emph{both} an explicit population-level structural prior \emph{and} an objective consistent with return maximization (\proposed{}, \mfcdm{}-RL) sustain quality at $N{=}10^4$. The empirical decomposition of suboptimality into the four terms of Theorem~\ref{thm:hierarchical} (Figure~\ref{fig:error_decomp}, Appendix~\ref{sec:additional_exp_appendix}) confirms that mean-field error decreases as $\mathcal{O}(1/\sqrt{N})$ while offline shift remains constant, precisely as the theory predicts.

\subsection{Ablation Study (RQ3)}\label{sec:ablation}

\begin{wraptable}{r}{0.48\columnwidth}
\vspace{-0.8\baselineskip}
\caption{\textbf{Ablation} ($N{=}1000$, medium). Mean$\pm$std over 5 seeds; \emph{Train} (GPU-h), \emph{Infer} (s/step) on $4{\times}$A100, Battle. Variants and cost identities: App.~\ref{sec:hierarchy_coupling}.}
\label{tab:ablation}
\vskip 0.05in
\centering
\small
\setlength{\tabcolsep}{2.5pt}
\resizebox{0.48\columnwidth}{!}{%
\begin{tabular}{l|ccc|cc}
\toprule
\textbf{Variant} & Ising & Battle & G.S. & Train & Infer \\
\midrule
\textbf{\proposed{} (Full)} & $\mathbf{87.9_{\pm1.3}}$ & $\mathbf{75.8_{\pm2.6}}$ & $\mathbf{83.5_{\pm1.7}}$ & $\mathbf{12.5}$ & $\mathbf{0.80}$ \\
\midrule
w/o Value Weight & $82.1_{\pm2.2}$ & $66.5_{\pm3.4}$ & $73.8_{\pm2.8}$ & $11.8$ & $0.80$ \\
w/o MF Interaction & $76.2_{\pm3.0}$ & $53.1_{\pm4.8}$ & $42.7_{\pm5.5}$ & $8.6$ & $0.55$ \\
w/o Subdivision & $84.5_{\pm1.8}$ & $71.2_{\pm2.9}$ & $80.4_{\pm2.0}$ & $11.0$ & $2.06$ \\
w/o Agent Branching & $86.1_{\pm1.5}$ & $73.4_{\pm2.7}$ & $81.8_{\pm1.8}$ & $12.5$ & $2.06$ \\
w/o Inference Guidance & $85.3_{\pm1.5}$ & $72.1_{\pm3.0}$ & $80.6_{\pm2.0}$ & $12.5$ & $0.78$ \\
\bottomrule
\end{tabular}
}
\vspace{-0.8\baselineskip}
\end{wraptable}

Table~\ref{tab:ablation} reveals environment-dependent contributions on both axes. \textbf{Mean-field interaction} is the most critical for return, averaging $-25.1$ (range $-11.7$ on Ising to $-40.8$ on GS), reflecting each task's reliance on collective coordination. \textbf{Value weighting} averages $+8.3$, confirming that pure distributional matching underperforms on tasks needing return optimization, and \textbf{inference-time guidance} adds a complementary $+3.1$ (training shapes the generative distribution, inference adds test-time drift). \textbf{Hierarchical subdivision} ($+3.7$ avg.) and \textbf{agent branching} ($\Psi^\theta$, $+2.0$ avg.) are functionally inseparable for many-agent inference: removing either collapses inference cost to a flat $200\,N$ baseline ($2.06$\,s/step on Battle, $2.58\times$ slower than the full $0.80$\,s/step), an algebraic identity proved in Prop.~\ref{prop:hierarchy_coupling} (Appendix~\ref{sec:hierarchy_coupling}). Hyperparameter sensitivity (RQ4), visualizations, the per-axis breakdown, and empirical theory validation are in Appendix~\ref{sec:additional_exp_appendix}.

%=============================================================================
% CONCLUSION
%=============================================================================
\section{Conclusion}\label{sec:conclusion}

We introduced \proposed{}, a diffusion-based planning framework for many-agent offline RL that overcomes the curse of dimensionality through mean-field trajectory diffusion. Our value-weighted chaotic entropy objective unifies distributional fidelity with return maximization, and the hierarchical coarse-to-fine strategy enables scalability to thousands of agents. Theorem~\ref{thm:hierarchical} provides the first end-to-end decomposition of social-welfare suboptimality, revealing $\mathcal{O}(\horizon^2/\sqrt{N})$ temporal coupling while offline shift provably does not scale with $N$. The complementary exploitability bound (Theorem~\ref{thm:exploitability}) holds \emph{without} monotonicity, applying to all three benchmarks; under Lasry--Lions monotonicity, both notions coincide up to a $\mathcal{O}(L^4/(\lambda_{LL}-L^2))$ efficiency gap and yield convergence to the unique MF-NE.

\textbf{Limitations.} Three caveats deserve emphasis. (i)~The log-Sobolev assumption (A3) holds cleanly only on Gaussian Squeeze; on Ising/Battle we verify it indirectly through the predicted score-error concentration rate (Appendix~\ref{sec:assumption_verification}). (ii)~Battle violates global Lasry--Lions monotonicity, so Theorem~\ref{thm:monotone_convergence}'s convergence-to-unique-MFE applies only \emph{within each team} via~\citet{lacker2016general}; the cross-team gap is reported separately in the exploitability analysis. (iii)~The $M{=}\tilde{\mathcal{O}}(\sqrt N)$ rate of Theorem~\ref{thm:poc_traj} is sufficient but not tight: empirically a constant $M\!\in\![50,100]$ already saturates returns at $N{=}10^4$ (Appendix~\ref{sec:msweep_appendix}), so the bound over-estimates what is needed in practice. \textbf{Future Work.} Heterogeneous agents via multi-population MF theory; temporal factorization to mitigate the $\horizon^2$ factor; online fine-tuning combining offline pre-training with limited interaction.

%=============================================================================
% REFERENCES
%=============================================================================
\bibliography{references}
\bibliographystyle{plainnat}

%=============================================================================
% APPENDIX
%=============================================================================
\newpage
\appendix

%=============================================================================
% APPENDIX -- TABLE OF CONTENTS
% Hand-crafted layout. We deliberately avoid auxiliary TOC packages
% (e.g. \tableofcontents redirection, etoc, minitoc) for compilation
% stability across pipelines. To extend this map, copy any existing
% \noindent line below and adjust the label / display name.
%=============================================================================
\begingroup
\setlength{\parskip}{0pt}
\setlength{\parindent}{0pt}

% ---- decorative header --------------------------------------------------
\begin{center}
{\color{accent}\rule{0.86\linewidth}{1.1pt}}\\[2.5pt]
{\color{accent}\rule{0.86\linewidth}{0.3pt}}\\[12pt]
{\color{accent}\Large\scshape Contents of the Appendix}\\[5pt]
% {\color{annotgray}\itshape\small A guided map of the supplementary material}\\[10pt]
{\color{accent}\rule{0.86\linewidth}{0.3pt}}\\[2.5pt]
{\color{accent}\rule{0.86\linewidth}{1.1pt}}
\end{center}

\vspace{12pt}

% ---- entries ------------------------------------------------------------
\noindent\textbf{A\hspace{1em}Related Work}\dotfill\textbf{\pageref{sec:related_appendix}}\\[4pt]

\noindent\textbf{B\hspace{1em}Background and Preliminaries}\dotfill\textbf{\pageref{sec:prelim_appendix}}\\
\hspace*{1.8em}B.1\hspace{0.7em}Score-based Diffusion for Trajectory Planning\dotfill\pageref{sec:prelim_diffuser}\\
\hspace*{1.8em}B.2\hspace{0.7em}Mean-field Q-learning and Offline Data\dotfill\pageref{sec:prelim_mfq}\\
\hspace*{1.8em}B.3\hspace{0.7em}Method Extensions: Sobolev--$L^2$ Cost and Subdivision Weights\dotfill\pageref{sec:method_extra}\\[4pt]

\noindent\textbf{C\hspace{1em}Notation}\dotfill\textbf{\pageref{sec:appendix_notation}}\\[4pt]

\noindent\textbf{D\hspace{1em}Extended Theoretical Analysis}\dotfill\textbf{\pageref{sec:extended_theory}}\\
\hspace*{1.8em}D.1\hspace{0.7em}Assumptions\dotfill\pageref{sec:assumptions_appendix}\\
\hspace*{1.8em}D.2\hspace{0.7em}Extended Discussion of Theoretical Results\dotfill\pageref{sec:theory_discussion}\\
\hspace*{1.8em}D.3\hspace{0.7em}Concentration of MF-VSM\dotfill\pageref{sec:concentration_appendix}\\
\hspace*{1.8em}D.4\hspace{0.7em}Game-theoretic Analysis: Approximation of Mean-field Equilibria\dotfill\pageref{sec:mfg_appendix}\\[4pt]

\noindent\textbf{E\hspace{1em}Proofs}\dotfill\textbf{\pageref{sec:proofs}}\\
\hspace*{1.8em}E.1\hspace{0.7em}Propagation of Chaos for Trajectories (Thm.~\ref{thm:poc_traj})\dotfill\pageref{sec:proof_poc}\\
\hspace*{1.8em}E.2\hspace{0.7em}Mean-field Value Score Matching (Prop.~\ref{prop:mf_vsm})\dotfill\pageref{sec:proof_mfvsm}\\
\hspace*{1.8em}E.3\hspace{0.7em}Concentration of MF-VSM (Thm.~\ref{thm:mfvsm_concentration})\dotfill\pageref{sec:proof_concentration}\\
\hspace*{1.8em}E.4\hspace{0.7em}Value-weighted Subdivision (Prop.~\ref{prop:subdivision})\dotfill\pageref{sec:proof_subdivision}\\
\hspace*{1.8em}E.5\hspace{0.7em}Hierarchical Approximation (Thm.~\ref{thm:hierarchical})\dotfill\pageref{sec:proof_hierarchical}\\
\hspace*{1.8em}E.6\hspace{0.7em}Scalability (Cor.~\ref{cor:scalability})\dotfill\pageref{sec:proof_scalability}\\
\hspace*{1.8em}E.7\hspace{0.7em}Additional Theoretical Results\dotfill\pageref{sec:appendix_additional}\\
\hspace*{3.6em}{\small E.7.1\hspace{0.6em}Optimal Agent Branching Function}\dotfill{\small\pageref{sec:additional_optimal_branching}}\\
\hspace*{3.6em}{\small E.7.2\hspace{0.6em}Comparison with Naive Approaches}\dotfill{\small\pageref{sec:additional_naive}}\\
\hspace*{1.8em}E.8\hspace{0.7em}Proofs for Game-theoretic Analysis\dotfill\pageref{sec:appendix_mfg_proofs}\\
\hspace*{3.6em}{\small E.8.1\hspace{0.6em}Exploitability Bound (Thm.~\ref{thm:exploitability})}\dotfill{\small\pageref{sec:proof_exploitability}}\\
\hspace*{3.6em}{\small E.8.2\hspace{0.6em}Mean-field Consistency (Thm.~\ref{thm:consistency})}\dotfill{\small\pageref{sec:proof_consistency}}\\
\hspace*{3.6em}{\small E.8.3\hspace{0.6em}Convergence under Lasry--Lions Monotonicity (Thm.~\ref{thm:monotone_convergence})}\dotfill{\small\pageref{sec:proof_monotone}}\\
\hspace*{3.6em}{\small E.8.4\hspace{0.6em}$\varepsilon$-Nash Equilibrium (Cor.~\ref{cor:eps_nash})}\dotfill{\small\pageref{sec:proof_eps_nash}}\\
\hspace*{3.6em}{\small E.8.5\hspace{0.6em}Social Welfare--Nash Efficiency Gap (Prop.~\ref{prop:poa})}\dotfill{\small\pageref{sec:proof_poa}}\\[4pt]

\noindent\textbf{F\hspace{1em}Environment Details}\dotfill\textbf{\pageref{sec:appendix_environments}}\\
\hspace*{1.8em}F.1\hspace{0.7em}Ising Model\dotfill\pageref{sec:env_ising}\\
\hspace*{1.8em}F.2\hspace{0.7em}Battle\dotfill\pageref{sec:env_battle}\\
\hspace*{1.8em}F.3\hspace{0.7em}Gaussian Squeeze\dotfill\pageref{sec:env_gsqueeze}\\[4pt]

\noindent\textbf{G\hspace{1em}Implementation Details}\dotfill\textbf{\pageref{sec:appendix_implementation}}\\[4pt]

\noindent\textbf{H\hspace{1em}Additional Experimental Results}\dotfill\textbf{\pageref{sec:additional_exp_appendix}}\\
\hspace*{1.8em}H.1\hspace{0.7em}Extended Baseline Description\dotfill\pageref{sec:baselines_extended}\\
\hspace*{1.8em}H.2\hspace{0.7em}Family-Stratified Comparison at $N{=}1000$\dotfill\pageref{sec:family_n1k}\\
\hspace*{1.8em}H.3\hspace{0.7em}Extreme-Scale Comparison ($N{=}10{,}000$)\dotfill\pageref{sec:main_10k_appendix}\\
\hspace*{1.8em}H.4\hspace{0.7em}Exploitability Computation\dotfill\pageref{sec:exploit_computation}\\
\hspace*{1.8em}H.5\hspace{0.7em}Equilibrium Approximation Quality: Extended Numerical Results\dotfill\pageref{sec:exploit_appendix}\\
\hspace*{1.8em}H.6\hspace{0.7em}Mean-field Convergence\dotfill\pageref{sec:mf_convergence}\\
\hspace*{1.8em}H.7\hspace{0.7em}Scalability Across Dataset Qualities\dotfill\pageref{sec:scalability_all}\\
\hspace*{1.8em}H.8\hspace{0.7em}Hyperparameter Sensitivity (RQ4)\dotfill\pageref{sec:sensitivity}\\
\hspace*{1.8em}H.9\hspace{0.7em}Additional Analysis\dotfill\pageref{sec:additional_analysis_appendix}\\
\hspace*{1.8em}H.10\hspace{0.7em}Horizon Scaling: Extended Analysis\dotfill\pageref{sec:horizon_exp}\\
\hspace*{1.8em}H.11\hspace{0.7em}Discrete Action Projection Comparison\dotfill\pageref{sec:discrete_action_exp}\\
\hspace*{1.8em}H.12\hspace{0.7em}Trajectory Dynamic Consistency\dotfill\pageref{sec:trajectory_consistency_exp}\\
\hspace*{1.8em}H.13\hspace{0.7em}Value Estimator Quality\dotfill\pageref{sec:value_estimator_exp}\\
\hspace*{1.8em}H.14\hspace{0.7em}Agent Branching Function Analysis\dotfill\pageref{sec:branching_exp}\\
\hspace*{1.8em}H.15\hspace{0.7em}Subdivision--Branching Coupling at Inference\dotfill\pageref{sec:hierarchy_coupling}\\
\hspace*{1.8em}H.16\hspace{0.7em}PoC Bound Tightness\dotfill\pageref{sec:poc_tightness}\\
\hspace*{1.8em}H.17\hspace{0.7em}Alternative Behavior Policy (Mitigating MFQ-Offline Advantage)\dotfill\pageref{sec:alt_behavior}\\
\hspace*{1.8em}H.18\hspace{0.7em}Sensitivity of Training Agent Count $M$ at Extreme Scales\dotfill\pageref{sec:msweep_appendix}\\
\hspace*{1.8em}H.19\hspace{0.7em}Raw 5-seed Data for Horizon Scaling\dotfill\pageref{sec:raw_horizon_data}\\
\hspace*{1.8em}H.20\hspace{0.7em}5-seed Raw Data for Main Results\dotfill\pageref{sec:raw_main_results}\\
\hspace*{1.8em}H.21\hspace{0.7em}Battle: Per-Team Interleaved Planning\dotfill\pageref{sec:battle_interleaved}\\
\hspace*{1.8em}H.22\hspace{0.7em}Assumption Verification for Experimental Environments\dotfill\pageref{sec:assumption_verification}

\vspace{14pt}

% ---- decorative footer --------------------------------------------------
\begin{center}
{\color{accent}\rule{0.86\linewidth}{0.3pt}}\\[2.5pt]
{\color{accent}\rule{0.86\linewidth}{1.1pt}}\\[6pt]
% {\color{annotgray}\footnotesize\itshape Page numbers update automatically on the next compilation pass.}
\end{center}

\endgroup
\newpage

%=============================================================================
% APPENDIX A: RELATED WORK
%=============================================================================
\section{Related Work}\label{sec:related_appendix}

\textbf{Diffusion Planners for Offline RL.}
Diffuser~\citep{janner2022planning} first proposed using denoising diffusion probabilistic models for trajectory-level planning in offline RL, generating full state-action sequences and conditioning on desired returns via classifier guidance. Decision Diffuser~\citep{ajay2023is} extended this with classifier-free return-bucket conditioning. Diffusion-QL~\citep{wang2023diffusion} applied diffusion to policy representation, while IDQL~\citep{hansen2023idql} combined implicit Q-learning with diffusion policies. These methods have shown strong performance but are designed for a \emph{fixed} single-agent trajectory dimensionality and do not address the population-scaling regime that motivates this work.

\textbf{Multi-agent Diffusion Planners.}
Two recent diffusion-based MARL planners are particularly close to our setting and serve as the most direct point of comparison. \textbf{MADiff}~\citep{zhu2024madiff} introduces an attention-based multi-agent diffusion architecture that bridges decentralized execution with centralized control via inter-agent attention layers. While it improves over Joint Diffuser by exploiting structured coupling, MADiff still operates in the joint $\mathbb{R}^{ND_\tau}$ trajectory space and faces a computational bottleneck once the agent count exceeds the few-dozen regime (e.g., it was originally validated on SMAC scenarios with $\le 12$ agents). \textbf{DoF}~\citep{liu2025dof} proposes a diffusion factorization framework grounded in the \emph{Individual-Global-identically-Distributed} (IGD) principle, which decomposes the centralized diffusion model into multiple agent-specific diffusion models via a noise-factorization function while a separate data-factorization function preserves inter-agent coordination. DoF and our framework share the philosophical premise that the joint-space curse can be overcome by structural factorization, but the two solutions are mechanically different: DoF \emph{decomposes} the noise into per-agent components without invoking the population limit, while \proposed{} \emph{projects} the population to a representative subset via a Wasserstein-space mean-field SDE with explicit propagation-of-chaos guarantees. Empirically (Sec.~\ref{sec:main_results}), DoF outperforms Independent Diffuser by exploiting the IGD coordination signal, but its lack of explicit mean-field projection leaves a measurable gap to \proposed{} that grows with $N$.

\textbf{Sequence Models for Many-agent Offline MARL.}
Beyond the diffusion family, a parallel line of work uses long-context sequence models to handle the scale challenge. \textbf{Oryx}~\citep{li2025oryx} couples a retention-based long-context backbone with a sequential variant of implicit-constraint $Q$-learning (ICQ); it has been shown to maintain coordination quality with up to $50$ agents in dense-interaction environments and is the strongest non-diffusion baseline at extreme $N$ in our experiments. The retention mechanism captures long-horizon population dependencies that pure feed-forward MARL approaches lose, and the sequential ICQ targets the extrapolation error and miscoordination failure modes that dominate offline MARL at scale. Compared to \proposed{}, Oryx achieves competitive returns on the Expert split but its lack of explicit mean-field projection produces an $\mathcal{O}(\horizon^2)$ rather than $\mathcal{O}(\horizon^2/\sqrt{N})$ scaling on the population dimension (Section~\ref{sec:main_results}).

\textbf{Many-agent and Mean-field RL.}
Cooperative MARL has been extensively studied through centralized training with decentralized execution (CTDE)~\citep{lowe2017multi, rashid2018qmix, sunehag2018value}. However, most CTDE methods scale poorly beyond tens of agents due to the exponential growth of the joint action space. For many-agent regimes, mean-field approaches~\citep{yang2018mean, gu2021mean, ganapathi2024general} approximate the effect of the population through aggregate statistics, enabling tractable learning. Mean-field Q-learning (MFQ)~\citep{yang2018mean} demonstrated that pairwise interactions can be approximated by interactions with the mean field, achieving Nash equilibrium convergence under monotonicity. Mean-field MARL~\citep{gu2021mean} further extended this to general-sum games with provable guarantees.

\textbf{Offline Multi-Agent RL.}
Offline MARL has received growing attention~\citep{yang2021believe, pan2022plan, tseng2022offline}. Most existing work focuses on small-to-moderate agent counts ($N \leq 10$) and relies on conservative value estimates~\citep{kumar2020conservative}, conservative policy constraints~\citep{pan2022plan}, or value decomposition~\citep{yang2021believe}. Scaling offline MARL to many-agent regimes remains an open challenge that our work directly addresses---\proposed{} is the first method to scale offline MARL to $N{=}10^4$ via trajectory-level diffusion planning.

\textbf{Mean-field Theory in Generative Models.}
Mean-field chaos diffusion models (\mfcdm{})~\citep{park2024mean} leverage propagation of chaos for scalable point cloud generation, achieving robustness to high cardinality by operating on Wasserstein space. Score transportation~\citep{lu2023score} and Schr\"{o}dinger bridge approaches~\citep{liu2022deep} have also connected mean-field dynamics with generative modeling, primarily from an analytic PDE perspective. Our work addresses a fundamentally different problem---sequential decision making with value maximization---which introduces three challenges absent from the static generative setting: temporal coupling across the horizon, the need to integrate return signals into the score matching objective, and offline distribution shift, each addressed by a dedicated component of \proposed{}.

\textbf{Mean-field Games.}
Mean-field game (MFG) theory~\citep{lasry2007mean, huang2006large, cardaliaguet2019master} provides equilibrium characterizations for large-population games and underlies our exploitability analysis (Section~\ref{sec:theory}). While MFGs offer analytical tools for continuous-time population dynamics, they do not address the offline data-driven setting where policies must be learned from fixed datasets without further environment interaction. Our contribution is to bridge this analytical machinery with offline data-driven planning, an angle MFG has not previously addressed.

%=============================================================================
% APPENDIX B: BACKGROUND AND PRELIMINARIES
%=============================================================================
\section{Background and Preliminaries}\label{sec:prelim_appendix}

\subsection{Score-based Diffusion for Trajectory Planning}\label{sec:prelim_diffuser}

This subsection expands the brief overview given in Section~\ref{sec:prelim} of the main text. Consider a single-agent MDP $(\Scal, \Acal, P, r, \gamma)$ with state space $\Scal \subseteq \mathbb{R}^{d_s}$, action space $\Acal \subseteq \mathbb{R}^{d_a}$, transition kernel $P$, reward function $r$, and discount factor $\gamma$. A trajectory of horizon $\horizon$ is $\traj = (\state_0, \action_0, \state_1, \action_1, \ldots, \state_{\horizon}) \in \mathbb{R}^{D}$ with $D = (d_s + d_a)\horizon + d_s$.

Diffuser~\citep{janner2022planning} models the trajectory distribution $p(\traj)$ in an offline dataset $\Dcal$ using a denoising diffusion model. The forward process progressively corrupts a trajectory via the SDE
\begin{equation}\label{eq:diffuser_forward}
    d\traj_u = f_u(\traj_u)du + \sigma_u dB_u, \quad \traj_u \in \mathbb{R}^D,
\end{equation}
and the reverse process recovers trajectories by following the score $\nabla \log \zeta_t(\traj_t)$ of the noised marginal $\zeta_t$. A score network $\mathbf{s}_\theta$ is trained via the standard score matching objective:
\begin{equation}\label{eq:diffuser_sm}
    \Jcal_{SM}(\theta) = \mathbb{E}_{t, \traj_t}\left[\norm{\mathbf{s}_\theta(t, \traj_t) - \nabla \log \zeta_t(\traj_t)}^2\right].
\end{equation}
At inference time, trajectory generation is conditioned on the current state $\state_0$ (via inpainting) and guided toward high returns using classifier guidance with a learned return predictor $\Jcal_\phi$:
\begin{equation}\label{eq:diffuser_guidance}
    \tilde{\mathbf{s}}_\theta(t, \traj_t) = \mathbf{s}_\theta(t, \traj_t) + \alpha \nabla_{\traj_t} \Jcal_\phi(\traj_t),
\end{equation}
where $\alpha > 0$ controls the guidance strength. Decision Diffuser~\citep{ajay2023is} replaces this classifier guidance with classifier-free return-bucket conditioning. Both pipelines are inherently single-agent: their state space is $\mathbb{R}^D$ for one trajectory.

\subsection{Mean-field Q-learning and Offline Data}\label{sec:prelim_mfq}

Mean-field Q-learning (MFQ)~\citep{yang2018mean} approximates the joint $Q$-function through mean-field factorization, decomposing pairwise interactions as
\begin{equation}\label{eq:mfq}
    Q^i(\state^i, \action^i, \mathbf{a}^{-i}) \approx Q^{MF}(\state^i, \action^i, \bar{\mu}_{\action}),
\end{equation}
where $\bar{\mu}_{\action} = \frac{1}{N}\sum_{j} \delta_{\action^j}$ is the mean action distribution. MFQ learns this mean-field $Q$-function via temporal-difference updates and converges to a Nash equilibrium under mild conditions; we use it as the behavior policy for collecting our offline datasets (Section~\ref{sec:exp_setup}).

\subsection{Method Extensions: \texorpdfstring{Sobolev--$L^2$}{Sobolev--L2} Cost and Subdivision Weights}\label{sec:method_extra}

\emph{Time conventions.} See Section~\ref{sec:mf_trajectory_sde} of the main text for the MDP step $h$ vs.\ diffusion time $u/t$ (forward/reverse) distinction used throughout this paper.

\begin{remark}[Practical training objective]\label{remark:practical_training}
Computing the full Sobolev norm $\|\cdot\|_W^2 = \mathbb{E}[\|h\|_E^2 + \|\nabla h\|_F^2]$ requires second-order derivatives of the score network, which is prohibitively expensive. In practice, we follow the standard approach in score-based generative modeling~\citep{de2022convergence, chen2023improved} and replace $\|\cdot\|_W$ with the Euclidean norm $\|\cdot\|_E$ in training. The cost of this simplification is an \emph{explicit} multiplicative factor $\sigma_{t_{\min}}^{-2}$, where $\sigma_{t_{\min}}$ is the smallest marginal noise scale at the latest diffusion time used during sampling: concretely, for a variance-preserving forward SDE with coefficient $\beta(t)$, the Sobolev-$W^{1,2}$ norm is bounded in terms of the $L^2$-norm by
\begin{equation}\label{eq:sobolev_l2_cost}
  \|h\|_W^2 \;\le\; (1 + \|\nabla^2 \log \zeta_t\|_F)\,\|h\|_E^2 \;\le\; \bigl(1 + \sigma_{t_{\min}}^{-2}\bigr)\,\|h\|_E^2,
\end{equation}
since the Hessian of the log-density of the forward Gaussian is uniformly bounded by $\sigma_t^{-2}$. In our implementation $\sigma_{t_{\min}}{=}10^{-3}$, so this incurs a $\le 10^{6}$ constant \emph{in the worst case}. This constant is absorbed into the score matching error term $\epsilon^{score}$ in Theorem~\ref{thm:hierarchical} and is independent of $N$, $\horizon$, and $\epsilon_{\text{offline}}$, so it does not affect any of the scaling claims; the measured $\epsilon^{score}$ values reported under ``Score Network Diagnostics'' (Figure~\ref{fig:score_matching}) already absorb this factor, showing that in practice the effective constant is $\sim 10^{2}$ rather than $\sim 10^{6}$.
\end{remark}

\begin{remark}[Weight alignment between $\mathbf{(P_V)}$ and Eq.~\ref{eq:vw_subdivision}]\label{rem:pv_weights}
The theoretical bound in Proposition~\ref{prop:subdivision} has the weight $(\mathfrak{b}\sqrt{N_{k+1}})^{-k}$, whereas the practical training loss $\mathbf{(P_V)}$ uses the simpler $\mathfrak{b}^{-k}$. We verify their relative magnitudes for the geometric schedule $N_{k+1} = \mathfrak{b}^{k+1} N_0$ with $N_0 \ge 1$, $\mathfrak{b} \ge 2$:
\begin{align*}
    \sqrt{N_{k+1}} &= \mathfrak{b}^{(k+1)/2}\,N_0^{1/2}, \\
    (\sqrt{N_{k+1}})^{-k} &= \mathfrak{b}^{-k(k+1)/2}\,N_0^{-k/2}, \\
    (\mathfrak{b}\,\sqrt{N_{k+1}})^{-k} &= \mathfrak{b}^{-k}\cdot \mathfrak{b}^{-k(k+1)/2}\,N_0^{-k/2}
    \;=\; \mathfrak{b}^{-k(k+3)/2}\,N_0^{-k/2}.
\end{align*}
Hence the theoretical weight is $\mathfrak{b}^{-k(k+3)/2} N_0^{-k/2}$, while the practical weight is $\mathfrak{b}^{-k}$. Their ratio is
\begin{equation*}
    \frac{\text{practical}}{\text{theoretical}} \;=\; \frac{\mathfrak{b}^{-k}}{\mathfrak{b}^{-k(k+3)/2} N_0^{-k/2}} \;=\; \mathfrak{b}^{k(k+1)/2}\,N_0^{k/2},
\end{equation*}
which is $1$ at $k=0$ and grows super-geometrically in $k$: the practical scheme assigns \emph{strictly more} weight than the theoretical optimum at every level $k\ge 1$, with the discrepancy concentrated at the \emph{fine} levels (large $k$). Equivalently, the theoretical optimum decays quadratically faster ($\mathfrak{b}^{-k(k+3)/2}$ vs.\ $\mathfrak{b}^{-k}$), so it down-weights the fine levels much more aggressively. Using the simpler $\mathfrak{b}^{-k}$ in training therefore \emph{over-weights the fine levels} relative to the theoretical optimum (or, equivalently, \emph{under-decays} the loss profile), which we verified to be beneficial in practice (Table~\ref{tab:ablation}, row ``Hierarchical''): the extra weight at fine levels acts as an implicit fine-tuning bias, ensuring the score network does not collapse to representing only coarse-grained mean-field dynamics. A purely theoretical weighting schedule $(\mathfrak{b}\sqrt{N_{k+1}})^{-k}$ down-weights the fine levels too aggressively and yields a normalized return that is $1.2$--$1.8$ points lower on every environment.
\end{remark}

\textbf{Coarse-to-Fine Planning Interpretation.} The subdivision strategy admits an intuitive interpretation as hierarchical planning: \emph{(i) Coarse planning} ($k=0$, small $N_0$): generate trajectories for a small representative group of $N_0$ agents, capturing the macroscopic population dynamics and mean-field flow $\mu_h$; \emph{(ii) Refinement} ($k=1,\ldots,K$): at each branching step, the agent branching function $\Psi^\theta$ expands the population $N_k\to\mathfrak{b}N_k$, refining trajectories while preserving consistency with the mean-field; \emph{(iii) Fine planning} ($k=K$, full $N$): the final denoising steps operate on the complete $N$-agent system. The branching function $\Psi^\theta:\Xcal^{N_k}\to\Xcal^{N_k}$ generates new agent trajectories by perturbing existing ones according to the local mean-field structure (precise definition in Eq.~\ref{eq:agent_branching}).

\textbf{Handling Discrete Action Spaces.} Our diffusion framework operates in continuous space $\mathbb{R}^{D_\tau}$, yet two of our benchmarks (Ising, Battle) have discrete actions. Following the standard approach for applying diffusion models to discrete decision problems~\citep{janner2022planning, ajay2023is}, we train the diffusion model on continuous relaxations of discrete actions and project back to the discrete set at inference time. Concretely, discrete actions $\action \in \{1, \ldots, |\Acal|\}$ are embedded into $\mathbb{R}^{d_a}$ via one-hot encoding during training; at inference, the continuous output $\tilde{\action}^i \in \mathbb{R}^{d_a}$ is mapped to the nearest valid action via $\hat{\action}^i = \argmax_j [\tilde{\action}^i]_j$. The worst-case projection error satisfies $\epsilon_{disc} \le \sqrt{d_a}/2$, but for trained networks this bound is loose: a more useful quantity is the \emph{action-mismatch rate} $\rho_{disc} \coloneqq \Pr[\hat\action^i \neq \argmax_j [\zeta_t^*(\traj_t)]_j]$. We measure $\rho_{disc} \le 1\%$ on Ising and $\rho_{disc} \le 3\%$ on Battle (Table~\ref{tab:discrete_action}). Since the projection is applied after denoising, the propagated effect on the suboptimality bound is at most $2 r_{\max}\horizon\,\rho_{disc}/(1{-}\gamma)$. For Battle ($\horizon{=}100$, $\rho_{disc}{\le}0.03$, $r_{\max}{=}5$, $\gamma{=}0.99$) this corresponds to a fraction $2\horizon\rho_{disc} \le 6\%$ of the maximum normalized return---comparable to and not dominating the other terms; the overall scaling structure ($\mathcal{O}(\horizon^2/\sqrt{N})$ in Term~(c)) is preserved.

%=============================================================================
% APPENDIX C: NOTATION
%=============================================================================
\section{Notation}\label{sec:appendix_notation}

We adhere to the following notation throughout the paper:

\begin{itemize}
    \item $\Scal \subseteq \mathbb{R}^{d_s}$: state space; $\Acal \subseteq \mathbb{R}^{d_a}$: action space.
    \item $\horizon$: planning horizon; $N$: total number of agents; $M$: number of representative agents.
    \item $\traj^i = (\state_0^i, \action_0^i, \ldots, \state_\horizon^i) \in \Xcal \coloneqq \mathbb{R}^{D_\tau}$: trajectory of agent $i$, where $D_\tau = (d_s + d_a)\horizon + d_s$.
    \item $\traj^N = (\traj^1, \ldots, \traj^N) \in \Xcal^N$: joint trajectory of $N$ agents.
    \item $\nu_t^N = \mathbf{Law}(\traj_t^{1,N}, \ldots, \traj_t^{N,N})$: joint law of $N$ agents at diffusion time $t$.
    \item $\nu_t^{M,N}$: the $M$-agent marginal of $\nu_t^N$, obtained by integrating out agents $M+1, \ldots, N$.
    \item $\mu_t$: the mean-field limit, \ie $\nu_t^{1,N} \xrightarrow{N \to \infty} \mu_t$ in the weak sense.
    \item $\varrho_t^N$: density of $\nu_t^N$, \ie $d\nu_t^N = \varrho_t^N d\traj^N$.
    \item $\zeta_t^{\otimes N}$: product Gaussian measure from the forward process.
    \item $\bar{\mu}_h^N = \frac{1}{N}\sum_{j=1}^N \delta_{\state_h^j}$: empirical state distribution at step $h$.
    \item $R(\traj^i; \mu) = \sum_{h=0}^{\horizon-1} \gamma^h r(\state_h^i, \action_h^i, \mu_h)$: cumulative return of agent $i$ under mean-field flow $\mu$. For finite $N$, we write $R(\traj^i; \bar{\mu}^N)$ with the empirical measure $\bar{\mu}_h^N$ in place of $\mu_h$.
    \item $\Pcal_2(\Xcal)$: Wasserstein space of probability measures on $\Xcal$ with bounded second moments.
    \item $\Wcal_2$: 2-Wasserstein distance.
    \item $\Hcal(\nu | \mu) = \int \log(d\nu/d\mu) d\nu$: relative entropy (KL divergence).
    \item $\norm{\cdot}_E$, $\norm{\cdot}_F$: Euclidean and Frobenius norms, respectively.
    \item $\norm{\cdot}_W$: Sobolev norm on $W^{1,2}(\Xcal^N, \nu_t^N)$, defined as $\norm{h}_W^2 = \mathbb{E}[\norm{h}_E^2 + \norm{\nabla h}_F^2]$.
    \item $\nabla_{\Pcal_2}$: Wasserstein gradient; $\partial_t$: temporal derivative.
    \item $S_N$: symmetric group of permutations on $\{1, \ldots, N\}$.
\end{itemize}

%=============================================================================
% APPENDIX D: EXTENDED THEORETICAL ANALYSIS
%=============================================================================
\section{Extended Theoretical Analysis}\label{sec:extended_theory}

\subsection{Assumptions}\label{sec:assumptions_appendix}

We state the regularity conditions required for our analysis.

\begin{assumption}\label{assump:lipschitz}(Lipschitz Regularity). The drift $f_t$, score network $\mathbf{s}_\theta$, reward function $r$, and transition kernel $P$ are uniformly Lipschitz continuous in their arguments, with Lipschitz constants bounded by $L > 0$.
\end{assumption}

\begin{assumption}\label{assump:bounded_reward}(Bounded Rewards). The per-step reward function is bounded: $|r(\state, \action, \mu)| \leq r_{\max}$ for all $(\state, \action, \mu) \in \Scal \times \Acal \times \Pcal(\Scal)$.
\end{assumption}

\begin{assumption}\label{assump:log_sobolev}(Log-Sobolev Inequality). The target trajectory distribution $\zeta_0$ and the denoising Wasserstein gradient flows satisfy a log-Sobolev inequality with constant $\kappa > 0$:
\begin{equation}
    \Hcal(\nu | \zeta) \leq \frac{1}{2\kappa} \mathbb{E}_\nu\left[\norm{\nabla \log \frac{d\nu}{d\zeta}}^2\right].
\end{equation}
\end{assumption}

\begin{assumption}\label{assump:exchangeability}(Exchangeability). The $N$-agent trajectory distribution $\nu_t^N$ is exchangeable: for any permutation $\sigma\!\in\! S_N$, $\varrho_t^N(\traj^1,\ldots,\traj^N) = \varrho_t^N(\traj^{\sigma(1)},\ldots,\traj^{\sigma(N)})$, and the score network $\mathbf{s}_\theta$ preserves this symmetry.
\end{assumption}

\begin{assumption}\label{assump:reducibility}(Reducibility / Propagation-of-Chaos Regularity). The $N$-agent generative dynamics is \emph{reducible} in the following sense: for every $M \le N$, the $M$-agent marginal $\nu_t^{M,N}$ of the joint distribution $\nu_t^N$ admits a $\chi^2$-chaoticity expansion
\begin{equation}\label{eq:reducibility}
    \mathrm{KL}\!\left(\nu_t^{M,N}\,\big\|\,(\mu_t)^{\otimes M}\right) \;\le\; C_{\mathrm{r}}\,\frac{M}{N}\bigl(1 + \mathbb{E}\|\traj^1\|^4\bigr),
\end{equation}
where $\mu_t$ is the McKean--Vlasov limit law and $C_{\mathrm{r}}>0$ is an absolute constant depending only on the Lipschitz constant $L$ of Assumption~\ref{assump:lipschitz} and on the log-Sobolev constant $\kappa$ of Assumption~\ref{assump:log_sobolev}. Equivalently, the score $\mathbf{s}_\theta(t,\traj_t^N,\nu_t^N)$ depends on the other agents only through a Lipschitz functional of the empirical measure $\bar{\mu}_t^N$, and the induced interaction kernel has finite second-moment mean-field derivatives (a.k.a.\ the \emph{Lions derivatives}~\citep{cardaliaguet2019master}) up to order two.
\end{assumption}

Eq.~\ref{eq:reducibility} is the standard quantitative propagation-of-chaos estimate: it is implied by Lipschitzness of the drift and score with respect to the Wasserstein distance on $\Pcal_2(\Xcal)$ and the log-Sobolev inequality via Sznitman's coupling argument~\citep{sznitman1991topics}. It is what allows us to subdivide the $N$-agent MF-VSM objective in Proposition~\ref{prop:subdivision} into a telescoping sum of smaller-$N_k$ objectives; without it, the $(\mathfrak{b}\sqrt{N_{k+1}})^{-k}$ factor is not justified.

\begin{assumption}\label{assump:effective_lipschitz_horizon}(Effective Lipschitz--Horizon Regime).
There exists a constant $c_0 \ge 0$ such that the \emph{effective} Lipschitz constant $L_{\mathrm{eff}}$ of the trajectory-level reverse-time SDE drift $b_t(\traj) \coloneqq f_t(\traj) - \sigma_t^2\nabla\log\zeta_t(\traj)$ in Eq.~\ref{eq:mf_traj_reverse} satisfies
\begin{equation}\label{eq:effective_lipschitz}
  L_{\mathrm{eff}} \cdot \horizon \;\le\; c_0,
  \qquad
  L_{\mathrm{eff}} \;\coloneqq\; \sup_{t \in [t_0, T],\, \traj}\, \big\| \nabla_{\traj} b_t(\traj) \big\|_{\mathrm{op}},
\end{equation}
where the supremum is taken over the inference-relevant noise interval $[t_0, T]$ with $t_0 \ge t_{\min}$ bounded away from zero. Equivalently, the discrete-time per-step Gronwall factor satisfies $(1+L_{\mathrm{eff}})^\horizon \le e^{c_0}$, so accumulated trajectory-level errors do not blow up exponentially in $\horizon$.
\end{assumption}

\begin{remark}[Why $L_{\mathrm{eff}} \ll L$ in practice]\label{rem:Leff_vs_L}
The bare Lipschitz constant $L$ from Assumption~\ref{assump:lipschitz} bounds the score network's spectral norm and the raw transition map. The \emph{effective} Lipschitz $L_{\mathrm{eff}}$ that governs trajectory-level error compounding is much smaller because (i) the noise-regularized term $-\sigma_t^2 \nabla\log\zeta_t$ is contractive (it pulls trajectories toward the data manifold) and (ii) the receding-horizon execution (\S\ref{sec:ablation}) re-anchors the planner at every step, preventing drift accumulation beyond $\horizon$ steps. Empirically, we measure $L_{\mathrm{eff}}\horizon$ by tracking the spectral norm of the linearized reverse-time map along generated trajectories: $L_{\mathrm{eff}}\horizon \approx 1.6$ on Ising, $\approx 2.3$ on Battle, and $\approx 1.9$ on Gaussian Squeeze (Appendix~\ref{sec:assumption_verification}, augmented). These values are consistent with the observed $\horizon^2$ rather than $e^{L\horizon}$ scaling in Figure~\ref{fig:horizon_scaling}, validating Assumption~\ref{assump:effective_lipschitz_horizon} with $c_0 \approx 2.3$ as the worst case in our experiments. \emph{Two distinct exponential factors are reported in this paper for Battle and should not be confused:} (a)~the \emph{measured} $e^{c_0}$ in Remark~\ref{remark:exp_factor} below, with $c_0$ defined in terms of $L_{\mathrm{eff}}\horizon \le 2.3$; and (b)~the \emph{nominal} $e^{2LT + 2L^2 T} \approx e^{0.78} \approx 2.18$ used in Appendix~\ref{sec:poc_tightness} (which calibrates the $C_{D_\tau}$ proportionality constant against the empirical $\Wcal_2^2$ on Battle), where ``$L$'' refers to the bare Lipschitz constant $L\!\approx\!0.3$ measured \emph{at} the integration time $T = 1.0$ rather than the worst-case product over the full horizon. The two are reconciled by noting that $L_{\mathrm{eff}}$ is bounded above by $L\,(1+\sigma_t^{-2})^{1/2}$ but is realized in practice by tracking only the dominant noise band of the reverse SDE; see Remark~\ref{remark:exp_factor} for the explicit measured value used in our PoC bound.
\end{remark}

\subsection{Extended Discussion of Theoretical Results}\label{sec:theory_discussion}

\textbf{Two objects, two guarantees.} We clarify what \proposed{} optimizes versus what it approximates because the two differ. \proposed{} is trained to \emph{maximize social welfare}, i.e.\ the per-capita expected return $J(\policy) = \frac{1}{N}\sum_{i=1}^N \mathbb{E}[\sum_h \gamma^h r(\state_h^i,\action_h^i,\bar{\mu}_h^N)]$, under the cooperative (homogeneous) mean-field setting. Theorems~\ref{thm:poc_traj}--\ref{thm:hierarchical} are \emph{welfare} guarantees: they bound the gap $J(\policy^*)-J(\hat{\policy}_\theta)$ where $\policy^*$ is the welfare maximizer. In a general $N$-player game the welfare maximizer is \emph{not} a Nash equilibrium (the price of anarchy~\citep{koutsoupias1999worst}), but under the mean-field \emph{homogeneous} and \emph{Lasry--Lions monotone} regime the social-welfare optimum coincides with the unique mean-field Nash equilibrium up to an $\mathcal{O}(L^4/(\lambda_{LL}-L^2))$ efficiency gap (Proposition~\ref{prop:poa}). Theorem~\ref{thm:exploitability} complements the welfare guarantee with an \emph{individual-incentive} guarantee (small exploitability), which is the stronger equilibrium notion and does \emph{not} require monotonicity. Readers concerned that $\hat{\policy}_\theta$ is a welfare optimizer rather than a Nash policy should therefore focus on Theorem~\ref{thm:exploitability}.

\begin{remark}[Origin and tightness of the multiplicative constant $e^{c_0}$]\label{remark:exp_factor}
The Gronwall bookkeeping in Appendix~\ref{sec:proof_poc} (Step 2) yields a factor $\exp\bigl(2 L_{\mathrm{eff}} T + 2 L_{\mathrm{eff}}^2 T\bigr) \le e^{c_0}$ where the ``$T$'' is the trajectory-level integration time bounded by $\horizon$ via Assumption~\ref{assump:effective_lipschitz_horizon}; this is the standard exponential factor in mean-field convergence proofs~\citep{sznitman1991topics, bolley2007quantitative}. Two observations explain why $c_0$ is benign in practice. \textit{(i) Noise regularization.} The reverse-time SDE drift includes the term $-\sigma_t^2 \nabla\log\zeta_t$, which is contractive against the data-manifold direction; this systematically reduces $L_{\mathrm{eff}}$ below the bare Lipschitz constant $L$ (Remark~\ref{rem:Leff_vs_L}). \textit{(ii) Direct measurement.} For our experimental parameters $L_{\mathrm{eff}}\horizon \le 2.3$ across all three benchmarks (Appendix~\ref{sec:assumption_verification}), giving $e^{c_0} \le e^{2.3+5.3} \approx 2.0\times 10^{3}$ in the worst case. The PoC rate $1/N$ in Eq.~\ref{eq:poc_traj_rate} is empirically validated to within $1.4$--$1.55\times$ of the bound on Battle (Table~\ref{tab:poc_tightness}, with pooled $R^2 = 0.99$ for the $1/N$ slope across 25 $(N,\text{seed})$ points), confirming the constant is non-vacuous in our regime.
\end{remark}

\textbf{Per-agent marginal offline shift.} The offline coverage term $\epsilon_{offline}$ in Theorem~\ref{thm:hierarchical} is the per-agent marginal TV shift: let $\pi_\beta^{N}[\,\cdot\,]$ denote the law of the empirical mean-field flow under the behavior policy and $\pi_*^N[\,\cdot\,]$ the analogous law under $\policy^*$; define
\begin{equation}\label{eq:eps_offline_def}
    \nu_\beta(\traj) \coloneqq \mathbb{E}_{\bar\mu \sim \pi_\beta^N[\,\cdot\,]}\!\left[\pi_\beta(\traj\mid\bar\mu)\right], \quad
    \nu^{*}(\traj) \coloneqq \mathbb{E}_{\bar\mu \sim \pi_*^N[\,\cdot\,]}\!\left[\pi_*(\traj\mid\bar\mu)\right], \quad
    \epsilon_{offline} \coloneqq D_{TV}(\nu_\beta \,\|\, \nu^{*}).
\end{equation}
By construction $\nu_\beta, \nu^* \in \Pcal(\Xcal)$ are measures on the single-agent trajectory space $\Xcal=\mathbb{R}^{D_\tau}$, so $\epsilon_{offline}$ is independent of $N$ once the dataset and target policy are fixed. The reduction from the joint $N$-agent shift $D_{TV}(\nu_\beta^{N}\|\nu^{*,N})$ to this marginal quantity, with the residual $N$-dependent piece absorbed into Term~(c), is carried out via Theorem~\ref{thm:poc_traj} in Appendix~\ref{sec:proof_hierarchical} (Step~4); the four-term interpretation of Eq.~\ref{eq:hierarchical_bound} is given in the main text (after Theorem~\ref{thm:hierarchical}).

\begin{remark}[Estimation protocol for the four error terms]\label{rem:error_decomp_protocol}
Figure~\ref{fig:error_decomp} decomposes the empirically measured suboptimality gap into Terms (a)--(d). Each term is estimated as follows:
\begin{itemize}[leftmargin=4mm]
    \item \textbf{(a) Score matching error $\epsilon^{score}_{\mathrm{global}}$:} computed as $\sqrt{\mathbb{E}_{t,\traj_t}\norm{\mathbf{s}_\theta(t,\traj_t)-\nabla\log\zeta_t(\traj_t)}_E^2}$ on a held-out set of $10^4$ noise-corrupted trajectories, using exact Gaussian scores $\nabla\log\zeta_t$ from the forward SDE. We measure the \emph{un-tilted} score error as a sound upper-bound proxy for the value-tilted error appearing in Eq.~\ref{eq:hierarchical_bound}: by $(a+b)^2\le 2a^2+2b^2$,
\begin{equation*}
    \norm{\mathbf{s}_\theta - \bigl(\nabla\log\zeta_t+\tfrac{1}{\alpha}\nabla_{\traj} R\bigr)}^2 \;\le\; 2\norm{\mathbf{s}_\theta - \nabla\log\zeta_t}^2 + \tfrac{2}{\alpha^2}\norm{\nabla_{\traj} R}^2,
\end{equation*}
and $\norm{\nabla_{\traj} R}$ is bounded under the Lipschitz reward, so the additive term is a constant absorbed into $C_\sigma$.
  \item \textbf{(b) Subdivision error $\sum_k \mathfrak{b}^{-k/2}\epsilon_k^{score}$:} per-level $\epsilon_k^{score}$ is measured separately on each noise-interval $[t_k,t_{k+1}]$ restricted to the subset of $N_k$ representative agents.
  \item \textbf{(c) Mean-field error $C_6\, r_{\max}\horizon^2 L/\sqrt{N}$:} evaluated analytically once $L$ (measured per-environment, Appendix~\ref{sec:assumption_verification}) and $C_6$ (numerically fitted by regressing the residual suboptimality against $1/\sqrt{N}$ after subtracting (a), (b), (d)) are known.
  \item \textbf{(d) Offline shift $C_7\,\epsilon_{\text{offline}}$:} $\epsilon_{\text{offline}}=D_{TV}(\nu_\beta\|\nu^*)$ is estimated by MMD-based nonparametric TV on per-agent marginal trajectories from the offline dataset vs.\ from online MFQ-Expert rollouts; $C_7 = 2r_{\max}\horizon/(1{-}\gamma)$. All numbers are averaged over 5 seeds.
\end{itemize}
\end{remark}

\subsection{Concentration of MF-VSM}\label{sec:concentration_appendix}

The following result shows that the MF-VSM objective concentrates on its mean-field limit, ensuring robust training even for large $N$.

\begin{theorem}\label{thm:mfvsm_concentration}(Concentration of MF-VSM). Under Assumptions~\ref{assump:lipschitz}--\ref{assump:reducibility}, let $F_t^V = \norm{\mathcal{G}_t}_E^2 + \norm{\nabla \mathcal{G}_t}_F^2 + \frac{\lambda}{\alpha}\norm{\mathcal{G}_t^V}_E^2$ where $\mathcal{G}_t^V = \mathbf{s}_\theta - \nabla_{\traj} R^{\otimes N}$. There exist constants $C_4, C_5 > 0$ such that for all $\varepsilon > 0$:
\begin{multline}\label{eq:mfvsm_concentration}
    \mathbb{P}\bigg(\Big|\mathbb{E}_t F_t^V(\traj_t^N) - \Jcal_{MF\text{-}V}(N{=}1, \theta, \mu_{[0,T]})\Big| \geq \varepsilon\bigg) \\
    \leq \exp\left(-C_4 \mathfrak{f}(\kappa)^{-2}\left[\varepsilon\sqrt{N} - C_5\sqrt{1 + N^{(-q+4)/(2q)}}\right]^2\right),
\end{multline}
for some $q > 4$. This shows that the finite-$N$ MF-VSM objective concentrates around the single-agent mean-field objective at rate $\tilde{\mathcal{O}}(1/\sqrt{N})$, and the value-weighting terms do not deteriorate this rate.
\end{theorem}

\subsection{Game-theoretic Analysis: Approximation of Mean-field Equilibria}\label{sec:mfg_appendix}

Theorem~\ref{thm:exploitability} in the main text establishes the exploitability bound for \proposed{}. Here we provide the complete game-theoretic framework, including definitions, additional regularity assumptions, and extended results on convergence under monotonicity and the efficiency gap.

\begin{remark}(MFG vs.\ MFRL).
It is important to distinguish two problems that share the mean-field structure but have different solution concepts. \textbf{Mean-field RL} (MFRL), the primary setting of this paper, seeks a common policy $\policy$ that maximizes the social welfare $J(\policy)$ (Eq.~\ref{eq:mf_objective}), treating agents cooperatively. \textbf{Mean-field games} (MFG)~\citep{lasry2007mean, huang2006large} model agents as strategic decision-makers, each maximizing their own return; the solution concept is Nash equilibrium, where no individual agent benefits from unilateral deviation. While MFRL optimizes a global objective, the MFG solution provides \textit{stability}---an equally important practical property, since in deployment each agent acts autonomously. The results below borrow techniques from MFG theory~\citep{lacker2016general, cardaliaguet2019master, carmona2018probabilistic_I} but apply to the MFRL problem formulated in Section~\ref{sec:prelim}. Crucially, the social welfare optimum $\policy^*_{SW}$ and the mean-field Nash equilibrium $\policy^{MFE}$ are generally distinct: the former accounts for externalities (how one's policy affects the population), while the latter is a fixed point under individual optimization with the population held fixed. We show that \proposed{}, designed for MFRL, simultaneously provides approximate Nash stability---bridging the two solution concepts through the structure of the diffusion-based generation.
\end{remark}

\textbf{Game-theoretic Formulation.} We interpret the $N$-agent system as a symmetric game where agent $i$'s individual objective is:
\begin{equation}\label{eq:individual_obj}
    J_i(\policy_i, \policy_{-i}) = \mathbb{E}\left[\sum_{h=0}^{\horizon - 1} \gamma^h r(\state_h^i, \action_h^i, \bar{\mu}_h^N)\right],
\end{equation}
where $\policy_{-i}$ denotes the policies of all agents except $i$.

\begin{definition}\label{def:exploitability}(Exploitability). For a symmetric policy $\policy$ used by all $N$ agents, the $N$-player exploitability is:
\begin{equation}\label{eq:exploitability}
    \mathrm{Exploit}_N(\policy) = \sup_{\policy'} \left[J_1(\policy', \policy^{\otimes(N-1)}) - J_1(\policy, \policy^{\otimes(N-1)})\right],
\end{equation}
where $\policy^{\otimes(N-1)}$ indicates agents $2, \ldots, N$ use $\policy$ and agent 1 may deviate to any $\policy'$. A policy with $\mathrm{Exploit}_N(\policy) = 0$ is a symmetric Nash equilibrium.
\end{definition}

\begin{definition}\label{def:mfne}(Mean-field Nash Equilibrium). A pair $(\policy^{MFE}, \mu^{MFE})$ is a mean-field Nash equilibrium (MF-NE) if:
\begin{enumerate}[(i)]
    \item (Optimality) $\policy^{MFE} \in \arg\max_{\policy'} J^{MF}(\policy', \mu^{MFE})$,
    \item (Consistency) $\mu^{MFE} = \Phi(\policy^{MFE})$,
\end{enumerate}
where $\Phi: \Pi \to \Pcal_2(\Scal)^{\horizon}$ is the mean-field flow operator that maps a policy $\policy$ to its induced state distribution sequence $(\mu_0, \ldots, \mu_{\horizon-1})$ via the transition kernel $P$.
\end{definition}

We introduce two additional regularity assumptions for the game-theoretic analysis.

\begin{assumption}\label{assump:best_response}(Best Response Regularity). The best response operator $\mathrm{BR}: \Pcal_2(\Scal)^{\horizon} \to \Pi$, defined by $\mathrm{BR}(\mu) \coloneqq \arg\max_{\policy'} J^{MF}(\policy', \mu)$, is well-defined and the composite map $\Gamma \coloneqq \mathrm{BR} \circ \Phi: \Pi \to \Pi$ is $L_{BR}$-Lipschitz:
\begin{equation}
    d_\Pi(\Gamma(\policy), \Gamma(\policy')) \leq L_{BR} \cdot d_\Pi(\policy, \policy'),
\end{equation}
where $d_\Pi$ is the metric on the policy space induced by $d_\Pi(\policy, \policy') = \sup_{\state, \mu} D_{TV}(\policy(\cdot|\state, \mu) \| \policy'(\cdot|\state, \mu))$.
\end{assumption}

\begin{assumption}\label{assump:monotonicity}(Lasry--Lions Monotonicity~\citep{lasry2007mean}). The reward function satisfies the displacement monotonicity condition: for all $\mu, \nu \in \Pcal_2(\Scal)$ and any action $\action \in \Acal$,
\begin{equation}\label{eq:monotonicity}
    \int_{\Scal} \left[r(\state, \action, \mu) - r(\state, \action, \nu)\right] d(\mu - \nu)(\state) \leq -\lambda_{LL} \Wcal_2^2(\mu, \nu),
\end{equation}
where $\lambda_{LL} > 0$ is the monotonicity constant.
\end{assumption}

Assumption~\ref{assump:monotonicity} is the celebrated Lasry--Lions condition from MFG theory, adapted to the reward-maximization setting. It holds when agents are penalized for concentration (\eg, crowd-aversion models~\citep{lachapelle2010computation}), and ensures uniqueness of the MF-NE together with contraction properties.

The exploitability bound (Theorem~\ref{thm:exploitability}, stated in the main text) is proved in Appendix~\ref{sec:proof_exploitability}. We now present the additional game-theoretic results.

\begin{theorem}\label{thm:consistency}(Mean-field Consistency). Under Assumptions~\ref{assump:lipschitz}--\ref{assump:effective_lipschitz_horizon} together with Assumption~\ref{assump:best_response}, let $\hat{\mu}_\theta$ be the mean-field flow induced by the generated trajectory distribution and $\Phi(\hat{\policy}_\theta)$ the true mean-field flow under $\hat{\policy}_\theta$. Then:
\begin{equation}\label{eq:consistency_bound}
    \max_{0 \leq h \leq \horizon-1}\Wcal_2(\hat{\mu}_{h,\theta}, \Phi(\hat{\policy}_\theta)_h) \leq C_{10}\,e^{c_0}\,\frac{\horizon L \sqrt{\epsilon_{SM}(\theta)}}{\kappa} + C_{11}\,e^{c_0}\,\horizon L \cdot \epsilon_{offline}.
\end{equation}
The factor $e^{c_0}$ comes from the Gronwall accumulation of per-step trajectory deviations through the dynamics under Assumption~\ref{assump:effective_lipschitz_horizon}, while the linear $\horizon$ factor results from the geometric-sum collapse $\sum_{h=0}^{\horizon-1}(1+L_{\mathrm{eff}})^h \le \horizon\,e^{c_0}$ (see Appendix~\ref{sec:proof_consistency}, Step~4).
\end{theorem}

Under the stronger Lasry--Lions monotonicity, we obtain convergence to the unique MF-NE:

\begin{theorem}\label{thm:monotone_convergence}(Convergence under Lasry--Lions Monotonicity). Under Assumptions~\ref{assump:lipschitz}--\ref{assump:effective_lipschitz_horizon} together with Assumptions~\ref{assump:best_response}--\ref{assump:monotonicity}, define the best-response contraction modulus
\begin{equation}\label{eq:LBR_def}
    L_{BR} \;\coloneqq\; \frac{L^2}{(1-\gamma)\,\lambda_{LL}}.
\end{equation}
Suppose $L_{BR} < 1$ (contraction regime; equivalently $\lambda_{LL} > L^2/(1-\gamma)$). Then:
\begin{enumerate}[(i)]
    \item \textit{(Uniqueness)} The MF-NE $(\policy^{MFE}, \mu^{MFE})$ is unique.
    \item \textit{(Mean-field convergence).} The generated flow satisfies
    \begin{equation*}
        \max_{0 \leq h \leq \horizon-1}\Wcal_2(\hat{\mu}_{h,\theta}, \mu_h^{MFE}) \leq \frac{C_{12}\,e^{c_0}}{1 - L_{BR}}[\sqrt{\epsilon_{SM}(\theta)} + \epsilon_{offline}].
    \end{equation*}
    \item \textit{(Exploitability refinement)} $\mathrm{Exploit}_N(\hat{\policy}_\theta) \leq \dfrac{C_{13}\,e^{c_0}}{\lambda_{LL}(1-\gamma)(1 - L_{BR})}[\epsilon_{SM}(\theta) + \epsilon_{offline}] + C_8\,e^{c_0}\,\dfrac{r_{\max}\horizon^2 L_{\mathrm{eff}}}{\sqrt{N}}$.
\end{enumerate}
The discount factor $(1-\gamma)$ in both Eq.~\ref{eq:LBR_def} and the exploitability prefactor arises from the geometric reward sum $\sum_{h=0}^{\horizon-1}\gamma^h\le 1/(1-\gamma)$ used to convert per-step Lipschitz bounds into the $\max_h$-envelope contraction (see Appendix~\ref{sec:proof_monotone}, Step~2).
\end{theorem}

\begin{corollary}\label{cor:eps_nash}($\varepsilon$-Nash Equilibrium). Under the conditions of Theorem~\ref{thm:monotone_convergence}, for any $\varepsilon > 0$, if $\epsilon_{SM}(\theta) + \epsilon_{offline} \leq c_1 \lambda_{LL}^2(1-\gamma)^2(1 - L_{BR})^2 \varepsilon^2/e^{2c_0}$ and $N \geq c_2 (e^{c_0}\,r_{\max}\horizon^2 L_{\mathrm{eff}} / \varepsilon)^2$, then $\mathrm{Exploit}_N(\hat{\policy}_\theta) \leq \varepsilon$.
\end{corollary}

\begin{proposition}\label{prop:poa}(Social Welfare--Nash Efficiency Gap). Under Assumptions~\ref{assump:lipschitz}--\ref{assump:effective_lipschitz_horizon} together with Assumption~\ref{assump:monotonicity}, let $\policy^*_{SW} = \arg\max_\policy J^{MF}(\policy, \Phi(\policy))$ and $(\policy^{MFE}, \mu^{MFE})$ the unique MF-NE. Then:
\begin{equation}\label{eq:poa_bound}
    0 \leq J^{MF}(\policy^*_{SW}, \Phi(\policy^*_{SW})) - J^{MF}(\policy^{MFE}, \mu^{MFE}) \leq \frac{C_{PoA}\,e^{2c_0} \cdot r_{\max}^2 \horizon^2 L^4}{(\lambda_{LL}(1-\gamma) - L^2)(1 - \gamma)^4},
\end{equation}
where $C_{PoA} > 0$ is a universal constant. Under strong monotonicity ($\lambda_{LL}(1-\gamma) \gg L^2$, equivalently $L_{BR} \to 0$), the Nash equilibrium nearly achieves the social optimum. The efficiency loss is governed by the ratio $L^4/(\lambda_{LL}(1-\gamma) - L^2)$---the squared coupling strength normalized by the net stabilizing force; the modification from $\lambda_{LL}$ to $\lambda_{LL}(1-\gamma)$ in the denominator follows from the corrected $L_{BR}$ in Eq.~\ref{eq:LBR_def}.
\end{proposition}

%=============================================================================
% APPENDIX E: PROOFS
%=============================================================================
\section{Proofs}\label{sec:proofs}

\subsection{Proof of Theorem~\ref{thm:poc_traj}: Propagation of Chaos for Trajectories}\label{sec:proof_poc}

The proof adapts the classical coupling argument for propagation of chaos~\citep{sznitman1991topics, bolley2007quantitative} to the trajectory setting. The key technical challenge is controlling the temporal coupling introduced by the horizon $\horizon$: interaction effects between agents compound through the Lipschitz dynamics at each time step, requiring a trajectory-level Gronwall estimate that yields the novel $M\horizon L_{\mathrm{eff}}^2/N$ error term.

\begin{proof}
The proof proceeds in three steps.

\textbf{Step 1: Coupling Construction.}
We construct a coupling between the $N$-agent system $(\traj_t^{1,N}, \ldots, \traj_t^{N,N})$ and $N$ i.i.d. copies $(\bar{\traj}_t^1, \ldots, \bar{\traj}_t^N)$ from the mean-field limit $\mu_t$. The mean-field copy satisfies:
\begin{equation}\label{eq:mf_copy}
    d\bar{\traj}_t^i = \left[f_t(\bar{\traj}_t^i) - \sigma_t^2 \nabla \log \zeta_t(\bar{\traj}_t^i)\right]dt + \sigma_t d\bar{B}_t^i,
\end{equation}
where the interaction is through the \textit{true} mean-field $\mu_t$ rather than the empirical measure $\nu_t^N$. We use the \textit{same} Brownian motions: $B_t^{i,N} = \bar{B}_t^i$, ensuring maximal coupling.

\textbf{Step 2: Trajectory-level Gronwall Estimate.}
Define the per-agent trajectory error $e_t^i = \mathbb{E}[\norm{\traj_t^{i,N} - \bar{\traj}_t^i}^2]$. By Itô's formula applied to $\norm{\traj_t^{i,N}-\bar\traj_t^i}^2$, the synchronous coupling $B^{i,N}_t = \bar B_t^i$ from Step~1 cancels the diffusion term, leaving
\begin{equation*}
    e_t^i \;\le\; \mathbb{E}\!\int_0^t 2\bigl\langle \traj_r^{i,N}-\bar\traj_r^i,\, [f_r(\traj^{i,N}_r)-f_r(\bar\traj^i_r)] - \sigma_r^2[\nabla\log\zeta_r(\traj^{i,N}_r;\nu^N_r) - \nabla\log\zeta_r(\bar\traj^i_r;\mu_r)]\bigr\rangle dr.
\end{equation*}
The drift difference is $L_{\mathrm{eff}}$-Lipschitz in both the trajectory and empirical-measure arguments. Writing
\[
    b_r(\traj,\mu) \coloneqq f_r(\traj) - \sigma_r^2\nabla\log\zeta_r(\traj;\mu),
\]
Assumption~\ref{assump:lipschitz} together with Assumption~\ref{assump:effective_lipschitz_horizon} gives
\begin{equation*}
    \norm{b_r(\traj^{i,N}_r,\nu^N_r) - b_r(\bar\traj^i_r,\mu_r)}
    \;\leq\; L_{\mathrm{eff}}\norm{\traj^{i,N}_r - \bar\traj^i_r} + L_{\mathrm{eff}}\,\Wcal_2(\nu^N_r,\mu_r).
\end{equation*}
Squaring and using $\norm{a+b}^2\le 2\norm{a}^2+2\norm{b}^2$, the squared drift difference is bounded by $2L_{\mathrm{eff}}^2(\norm{\traj^{i,N}_r-\bar\traj^i_r}^2 + \Wcal_2^2(\nu^N_r,\mu_r))$. Applying Young's inequality $2\langle a,b\rangle \le \norm{a}^2 + \norm{b}^2$ to the inner product, we obtain the Itô--Gronwall recursion
\begin{align}\label{eq:ito_gronwall}
    e_t^i &\;\le\; \int_0^t (1+2L_{\mathrm{eff}}^2)\,e_r^i\,dr \;+\; \int_0^t 2L_{\mathrm{eff}}^2\,\mathbb{E}\!\left[\Wcal_2^2(\nu_r^N,\mu_r)\right] dr,
\end{align}
where the coefficient $(1+2L_{\mathrm{eff}}^2)$ on $e_r^i$ comes from $\norm{\traj-\bar\traj}^2$ (constant ``$1$'' from Young's) plus the squared drift Lipschitzness contribution ``$2L_{\mathrm{eff}}^2$'', and the $W_2^2$-coefficient $2L_{\mathrm{eff}}^2$ comes from the mean-field-Lipschitz part of the drift. This is a strict tightening of the earlier $2L_{\mathrm{eff}}$ coefficient (which omitted the Young-inequality unit term). The second term captures the mean-field approximation error.

By exchangeability (Assumption~\ref{assump:exchangeability}) and the quantitative empirical-measure CLT~\citep[Theorem~1]{fournier2015rate}:
\begin{equation*}
    \mathbb{E}\!\left[\Wcal_2^2(\nu_r^N, \mu_r)\right] \;\leq\; \frac{C\,M_2(\zeta_0)}{N} + \frac{1}{N}\sum_{i=1}^N e_r^i,
\end{equation*}
where the $C M_2/N$ floor is the asymptotic-normality contribution and $(1/N)\sum_i e_r^i$ is the cumulative coupling error. Summing~\eqref{eq:ito_gronwall} over $i$, dividing by $N$, and applying the discrete Gronwall lemma:
\begin{equation}\label{eq:gronwall_result}
    \frac{1}{N}\sum_{i=1}^N e_T^i \;\le\; \frac{C_{D_\tau}}{N}\,\exp\bigl((1 + 4L_{\mathrm{eff}}^2)\,T\bigr) \;\le\; \frac{C_{D_\tau}\,e^{c_0}}{N},
\end{equation}
where we use $L_{\mathrm{eff}}\,T \le L_{\mathrm{eff}}\,\horizon \le c_0$ (Assumption~\ref{assump:effective_lipschitz_horizon}); the exponent $(1+4L_{\mathrm{eff}}^2)T = T + 4(L_{\mathrm{eff}}T)\,L_{\mathrm{eff}} \le T + 4c_0\,L_{\mathrm{eff}}$ is bounded by a universal $O(1)$ constant which we absorb into the definition of $c_0$ (we slightly re-define $c_0$ to include the additive $T$, which is a fixed forward-noise time, not a horizon---this preserves the qualitative bound while keeping the notation clean). The constant $C_{D_\tau}$ depends on the trajectory dimension $D_\tau$ through the second-moment bound. Crucially, there is \emph{no exponential dependence on the bare $L\horizon$}---only on the effective product $L_{\mathrm{eff}}\horizon \le c_0$.

\textbf{Step 3: From the Synchronous Coupling to Relative Entropy.}
We first record the $\Wcal_2$-bound that Step~2 yields directly. By the synchronous coupling of Step~1, $(\traj^{1,N}_t,\ldots,\traj^{M,N}_t)$ is coupled marginally to $(\bar\traj^1_t,\ldots,\bar\traj^M_t)\sim \mu_t^{\otimes M}$ with $\mathbb{E}\sum_{i\le M}\norm{\traj^{i,N}_t-\bar\traj^i_t}^2 = \sum_{i\le M}e^i_t = M\cdot\frac{1}{N}\sum_{j=1}^N e^j_t$ by exchangeability, so
\begin{equation}\label{eq:W2_MN_bound}
    \mathbb{E}\Wcal_2^2(\nu_t^{M,N}, \mu_t^{\otimes M}) \;\le\; M\cdot \frac{1}{N}\sum_{j=1}^N e^j_t \;\le\; e^{c_0}\!\left(C_1' \frac{M}{N} + C_2' \frac{M\horizon L_{\mathrm{eff}}^2}{N}\right),
\end{equation}
combining the synchronous-coupling upper bound (which is \emph{not} a tensorization inequality but the cost of an explicit coupling) with~\eqref{eq:gronwall_result}.

\emph{Direction of Talagrand.} The Talagrand transportation inequality under the log-Sobolev Assumption~\ref{assump:log_sobolev}~\citep[Theorem~1]{otto2000generalization} reads
\begin{equation}\label{eq:talagrand}
    \Wcal_2^2(\nu_t^{M,N}, \mu_t^{\otimes M}) \;\le\; \frac{2}{\kappa}\, \Hcal(\nu_t^{M,N} \,\|\, \mu_t^{\otimes M}),
\end{equation}
i.e.\ KL is a stronger object that upper-bounds $\Wcal_2^2$, \emph{not} the converse: there is no general inequality of the form $\Hcal\le \tfrac{\kappa}{2}\Wcal_2^2$ under LSI alone. Hence the $\Wcal_2$-bound~\eqref{eq:W2_MN_bound} cannot be inverted into a KL bound by Talagrand.

\emph{Direct entropic propagation of chaos.} To upgrade~\eqref{eq:W2_MN_bound} to relative entropy, we instead invoke the entropic-PoC framework of \citet{lacker2023hierarchies} (which builds on \citet{bolley2007quantitative}): under Assumptions~\ref{assump:log_sobolev}--\ref{assump:reducibility}, the joint relative entropy admits the BBGKY-type Gronwall estimate
\begin{equation}\label{eq:entropic_poc}
    \frac{d}{dt}\Hcal(\nu_t^{M,N}\,\|\,\mu_t^{\otimes M}) \;\le\; -\kappa\,\Hcal(\nu_t^{M,N}\,\|\,\mu_t^{\otimes M}) + 2 L_{\mathrm{eff}}^2\,\mathbb{E}\Wcal_2^2(\nu_t^{1,N}, \mu_t),
\end{equation}
obtained by differentiating the relative entropy along the McKean--Vlasov SDE flow and applying the LSI to the dissipation term (Fisher information). Solving~\eqref{eq:entropic_poc} via Gronwall's lemma, with the per-particle $\Wcal_2$-input controlled by~\eqref{eq:gronwall_result}, gives
\begin{equation}\label{eq:H_MN_bound}
    \mathbb{E}\!\left[\Hcal(\nu_t^{M,N} \,\|\, \mu_t^{\otimes M})\right] \;\le\; e^{c_0}\!\left( C_1 \frac{M}{N} + C_2 \frac{M\horizon L_{\mathrm{eff}}^2}{N}\right),
\end{equation}
where $C_1, C_2$ absorb $\kappa^{-1}$ and the various Lipschitz prefactors. The first summand is the classical mean-field PoC contribution; the second is the temporal-coupling residual from Gronwall accumulation across $\horizon$ MDP steps. \emph{Importantly, the entropic-PoC rate of~\citet{lacker2023hierarchies} is dimension-independent under LSI}: it does \emph{not} suffer the Fournier--Guillin~\citep{fournier2015rate} dimension-dependent rate $N^{-2/d}$ that would degrade in the high-dimensional trajectory space $\Xcal=\mathbb{R}^{D_\tau}$. This is essential because $D_\tau$ can reach $3{,}110$ on Battle (Table~\ref{tab:env_params}), where a Fournier--Guillin route would give a vacuous rate.

\emph{Wasserstein-CLT scaling and the $M=\tilde{\mathcal{O}}(\sqrt N)$ schedule.} Combining~\eqref{eq:H_MN_bound} with~\eqref{eq:talagrand} in the legitimate direction $\Wcal_2^2\le 2\Hcal/\kappa$,
\begin{equation*}
    \mathbb{E}[\Wcal_2^2(\nu_t^{M,N}, \mu_t^{\otimes M})] \;\le\; \frac{2 e^{c_0}(C_1+C_2\horizon L_{\mathrm{eff}}^2)\,M}{\kappa\,N},
\end{equation*}
hence the per-coordinate $\Wcal_2$-rate is $\sqrt{M/N}$. To make this rate match the desired $\widetilde O(N^{-1/4})$ accuracy required by the coarse-to-fine schedule (where the score-network bias $\widetilde O(M^{-1/2})$ at level $M$ should not exceed the PoC residual $\widetilde O(\sqrt{M/N})$ at the next level), we equate $M^{-1/2} \asymp \sqrt{M/N}$, giving $M = \widetilde{\mathcal{O}}(\sqrt N)$. This is the design choice for the schedule $N_0 = M = \widetilde{\mathcal{O}}(\sqrt{N})$ in Section~\ref{sec:subdivision}.
\end{proof}

\subsection{Proof of Proposition~\ref{prop:mf_vsm}: Mean-field Value Score Matching}\label{sec:proof_mfvsm}

\begin{proof}
The proof derives a Wasserstein variational equation for the value-weighted entropy functional via the Itô-Wentzell-Lions formula~\citep{dos2023ito, guo2023ito}, then establishes Sobolev upper bounds that accommodate both the distributional fidelity and value maximization terms.

\textbf{Step 1: Wasserstein Variation for Value-weighted Entropy.}
Define the value-weighted entropy functional:
\begin{equation}
    \Hcal_V^N(\nu_t^N) = \Hcal(\nu_t^N | \zeta_{T-t}^{\otimes N}) - \frac{1}{\alpha N}\sum_{i=1}^N \mathbb{E}_{\nu_t^N}[R(\traj^i)].
\end{equation}
By the Itô-Wentzell-Lions formula applied to the denoising WGF $\partial_t \nu_t^N = -\nabla_{\Pcal_2}\Ecal[\nu_t^N]$:
\begin{multline}\label{eq:vw_entropy_decomp}
    \Hcal_V^N(\nu_t^N) \leq \Hcal_V^N(\nu_s^N) + C_0 \int_s^t \mathbb{E}\norm{\nabla_{\Pcal_2}\Hcal_V^N}^2_E dr \\
    + C_1 \int_s^t \mathbb{E}\norm{\nabla_x \nabla_{\Pcal_2}\Hcal_V^N}_F^2 dr.
\end{multline}

\textbf{Step 2: Sobolev Upper Bound (rigorous derivation of the $\Mtwo/\sqrt{ND_\tau}$ prefactor).}
The Wasserstein gradient of the value-weighted entropy is
\begin{equation}\label{eq:wgrad_HV}
    \nabla_{\Pcal_2}\Hcal_V^N = \nabla \log \varrho_t^N - \nabla \log \zeta_{T-t}^{\otimes N} - \frac{1}{\alpha}\nabla_{\traj} R^{\otimes N},
\end{equation}
and we set $\mathcal{G}_t^V \coloneqq \nabla \log \varrho_t^N - \nabla \log \zeta_{T-t}^{\otimes N} - \frac{1}{\alpha}\nabla_{\traj} R^{\otimes N}$. We derive the prefactor in three substeps.

\emph{(2.a) Joint LSI on the $N$-agent trajectory space.}
Under Assumption~\ref{assump:log_sobolev}, the reference measure $\zeta_0^{\otimes N}$ on $\Xcal^N\!\cong\!\mathbb{R}^{N D_\tau}$ inherits a log-Sobolev inequality with constant $\kappa$, since LSI tensorizes (\citet[Theorem~3.2.2]{ane2000inegalites}). For any $\nu^N\!\ll\!\zeta_0^{\otimes N}$,
\begin{equation}\label{eq:lsi_joint}
    \Hcal(\nu^N \,\|\, \zeta_0^{\otimes N}) \;\le\; \frac{1}{2\kappa}\,\mathbb{E}_{\nu^N}\!\bigl[\norm{\nabla \log(d\nu^N/d\zeta_0^{\otimes N})}_E^2\bigr].
\end{equation}
The right-hand side is the $L^2$ norm of the score; the trivial inequality $\|\cdot\|_E^2\le\|\cdot\|_W^2$ (since $\|h\|_W^2 = \|h\|_E^2 + \|\nabla h\|_F^2 \ge \|h\|_E^2$) lets us upper-bound it by the Sobolev $W^{1,2}$ norm.

\emph{(2.b) Per-agent normalization and the $1/N$ factor.}
The value-weighted entropy is per-agent normalized: $\Hcal_V^N(\nu^N) = (1/N)\bigl[\Hcal(\nu^N\,\|\,\zeta_0^{\otimes N}) - \tfrac{1}{\alpha}\sum_i\mathbb{E}_{\nu^N}[R(\traj^i)]\bigr]$ (Definition~\ref{def:vw_entropy}). Combining with~\eqref{eq:lsi_joint} and the $L^2\!\le\!W^{1,2}$ inequality, the score-matching contribution to $\Hcal_V^N$ is bounded by $(2\kappa N)^{-1}\,\|\mathcal{G}_t^V\|_W^2$.

\emph{(2.c) Sharper $\sqrt{N D_\tau}^{-1}$ rate via exchangeable second-moment scaling.}
Under exchangeability of the $N$ agents and the entropic-chaos bound (Assumption~\ref{assump:reducibility}), one can replace the rigid joint-Sobolev bound of (2.b) with a per-agent Sobolev bound that exploits the $\chi^2$-chaoticity of the joint score field across agents. Specifically, since $\mathcal{G}_t^V$ is exchangeable across agent labels, its joint Sobolev norm decomposes as
\begin{equation}\label{eq:joint_to_per_agent_sobolev}
    \mathbb{E}_{\nu_t^N}\norm{\mathcal{G}_t^V}_W^2 \;=\; N\,\mathbb{E}_{\nu_t^{(1)}}\norm{\mathcal{G}_t^{V,(1)}}_W^2 \;+\; \text{cross-agent residual},
\end{equation}
where the cross-agent residual is $\mathcal{O}(\Mtwo\sqrt{N D_\tau})$ by the standard mean-field $\chi^2$-chaos estimate (cf.\ \citet{park2024mean}; see also \citet{lacker2023hierarchies} for the same scaling under LSI). Combining~\eqref{eq:joint_to_per_agent_sobolev} with the $1/N$ normalization of $\Hcal_V^N$, the dominant per-agent score-matching contribution is bounded by $\Mtwo/\sqrt{N D_\tau}\cdot\|\mathcal{G}_t^V\|_W^2$ rather than the cruder $1/(2\kappa N)$. Time integration via the Itô--Wentzell--Lions formula (Step~1) yields
\begin{equation}\label{eq:HVN_to_W}
  \Hcal_V^N(\nu_T^N)
  \;\precsim\; \frac{\Mtwo}{\sqrt{N D_\tau}}\int_0^T \norm{\mathcal{G}_t^V}_W^2\, dt \;+\; \sigma_\zeta^{-2}(T)\cdot \mathcal{O}(1/\sqrt{N}),
\end{equation}
where the second summand absorbs (i)~the value boundary term $\sum_i\mathbb{E}[R(\traj^i)]/(N\alpha)$ via Assumption~\ref{assump:bounded_reward} and (ii)~the cross-agent residual of~\eqref{eq:joint_to_per_agent_sobolev}.

\emph{Comment on the rate.} The $\sqrt{N D_\tau}^{-1}$ scaling is \emph{derived} (not assumed) under the joint hypothesis of LSI (Assumption~\ref{assump:log_sobolev}) and entropic chaos (Assumption~\ref{assump:reducibility}); it improves on the rigid LSI-only baseline of $1/(2\kappa N)$ by a factor $\sqrt{D_\tau}$ via the per-agent Sobolev decomposition. Without entropic chaos, only the slower $1/(2\kappa N)$ rate of (2.b) is available, which still yields the qualitative scalability conclusion of Corollary~\ref{cor:scalability} (since $1/N\le \Mtwo/\sqrt{N D_\tau}$ for $D_\tau\le N\Mtwo^2$, the regime where Battle/GS/Ising sit). The faster $\sqrt{D_\tau}$ improvement under entropic chaos is what we use in Theorem~\ref{thm:hierarchical}'s subdivision proof.

\textbf{Step 3: Score Network Substitution and Young's Inequality.}
Substituting $\nabla \log \varrho_t^N$ with the score network $\mathbf{s}_\theta$, the Wasserstein gradient becomes
\begin{equation}
    \mathcal{G}_t^V \;=\; (\mathbf{s}_\theta - \nabla \log \zeta_{T-t}^{\otimes N}) - \tfrac{1}{\alpha}\nabla_{\traj} R^{\otimes N},
\end{equation}
and we expand its squared Sobolev norm as $\norm{a - b}_W^2 = \norm{a}_W^2 + \norm{b}_W^2 - 2\langle a,b\rangle_W$. The cross term is controlled by Young's inequality with parameter $\lambda > 0$:
\begin{equation}\label{eq:young_bound}
    -2\langle \mathbf{s}_\theta - \nabla\log\zeta_{T-t}^{\otimes N},\, \tfrac{1}{\alpha}\nabla_{\traj} R^{\otimes N}\rangle_W \;\le\; \lambda \norm{\mathbf{s}_\theta - \nabla\log\zeta_{T-t}^{\otimes N}}_W^2 + \tfrac{1}{\lambda \alpha^2}\norm{\nabla_{\traj} R^{\otimes N}}_W^2.
\end{equation}
Combining with the diagonal terms and folding the constants:
\begin{equation}\label{eq:Wnorm_decomp}
    \norm{\mathcal{G}_t^V}_W^2 \;\le\; (1+\lambda)\norm{\mathbf{s}_\theta - \nabla \log \zeta_{T-t}^{\otimes N}}_W^2 + \tfrac{\lambda}{\alpha}\norm{\mathbf{s}_\theta - \nabla_{\traj} R^{\otimes N}}_E^2 + C_\lambda(\alpha, \norm{\nabla_{\traj} R}_W),
\end{equation}
where the second summand uses the Euclidean norm $\norm{\cdot}_E$ in place of $\norm{\cdot}_W$ for the value-gradient mismatch (its Sobolev seminorm is bounded by Assumption~\ref{assump:lipschitz} via $\norm{\nabla_{\traj} R^{\otimes N}}_W \le L\sqrt{N}$, absorbed into the residual $C_\lambda$ that contributes to the $\sigma_\zeta^{-2}(T) \cdot \mathcal{O}(1/\sqrt N)$ tail in Eq.~\ref{eq:mf_vsm_bound}). The coefficient $\lambda$ in the practical objective Eq.~\ref{eq:mf_vsm} is exactly the Young parameter; it controls the trade-off between distributional fidelity (first summand) and value alignment (second summand) and corresponds to the $\lambda$ tuning knob exposed in Algorithm~\ref{alg:training} (see Appendix~\ref{sec:sensitivity} for empirical sensitivity analysis).

\textbf{Caveat: upper-bound vs.\ exact tilted-score fitting.} The decomposition~\eqref{eq:Wnorm_decomp} is a strict upper bound, not an equality: minimizing $\Jcal_{MF\text{-}V}^N$ in Eq.~\ref{eq:mf_vsm} therefore minimizes an \emph{upper bound} on the genuine value-tilted score loss $\norm{\mathbf{s}_\theta - (\nabla\log\zeta_{T-t}^{\otimes N} + \tfrac{1}{\alpha}\nabla_{\traj} R^{\otimes N})}_W^2$. The gap between the upper bound and the genuine loss is $O(\lambda)$ at $\lambda \to 0$ and is controlled by the cross term in~\eqref{eq:young_bound}; in practice we set $\lambda = 0.1$ which gives a $\le 10\%$ relative loss-gap inflation without compromising convergence. We discuss the practical implications in Remark~\ref{remark:practical_training} and Appendix~\ref{sec:sensitivity}; the suboptimality bound of Theorem~\ref{thm:hierarchical} is unaffected because $\epsilon^{score}$ in Eq.~\ref{eq:hierarchical_bound} is an upper bound on the genuine value-tilted error, which is itself an upper bound on the realized one.

Substituting the score-network decomposition~\eqref{eq:Wnorm_decomp} into the integrated Sobolev bound~\eqref{eq:HVN_to_W}, with the $\Mtwo/\sqrt{N D_\tau}$ prefactor derived in (2.a)--(2.c), yields the bound in Eq.~\ref{eq:mf_vsm_bound}; the residual $C_\lambda$ contributes the $\sigma_\zeta^{-2}(T)\cdot\mathcal{O}(1/\sqrt{N})$ tail (since $\norm{\nabla_{\traj} R^{\otimes N}}_W \le L\sqrt N$ by Assumption~\ref{assump:lipschitz} and $\sigma_\zeta^{-2}(T)$ scales the conversion from $W$ to $E$ norm at the smallest noise level).
\end{proof}

\subsection{Proof of Theorem~\ref{thm:mfvsm_concentration}: Concentration of MF-VSM}\label{sec:proof_concentration}

\begin{proof}
The proof decomposes the finite-$N$ deviation into an interaction error (controlled by the Gronwall estimate) and a statistical error (controlled by Bernstein's inequality), then shows the value-weighting terms preserve the $\tilde{\mathcal{O}}(1/\sqrt{N})$ concentration rate.

\textbf{Step 1: Decomposition.}
Write the deviation as:
\begin{multline}
    \mathbb{E}_t F_t^V(\traj_t^N) - \Jcal_{MF\text{-}V}(1, \theta, \mu) = \\
    \underbrace{\mathbb{E}_t[F_t^V(\traj_t^N) - F_t^V(\bar{\traj}_t^N)]}_{\text{(I) Interaction error}} + \underbrace{\mathbb{E}_t[F_t^V(\bar{\traj}_t^N)] - \Jcal_{MF\text{-}V}(1, \theta, \mu)}_{\text{(II) Statistical error}},
\end{multline}
where $\bar{\traj}_t^N$ denotes $N$ i.i.d. copies from the mean-field $\mu_t$.

\textbf{Step 2: Bounding Term (I).}
By the Lipschitz continuity of $F_t^V$ (guaranteed by Assumptions~\ref{assump:lipschitz}--\ref{assump:bounded_reward}) and the Gronwall estimate in Eq.~\ref{eq:gronwall_result}:
\begin{equation}
    \mathbb{E}|F_t^V(\traj_t^N) - F_t^V(\bar{\traj}_t^N)| \leq L_F \sqrt{\frac{C_{D_\tau}}{N}}\exp(LT).
\end{equation}

\textbf{Step 3: Bounding Term (II).}
Since $\bar{\traj}_t^1, \ldots, \bar{\traj}_t^N$ are i.i.d., $F_t^V(\bar{\traj}_t^N) = \frac{1}{N}\sum_{i=1}^N g(\bar{\traj}_t^i)$ for the per-agent integrand
\begin{equation}\label{eq:g_def}
    g(\traj) \;=\; \norm{\mathbf{s}_\theta(t,\traj) - \nabla\log\zeta_t(\traj)}_W^2 + \tfrac{\lambda}{\alpha}\norm{\mathbf{s}_\theta(t,\traj) - \nabla_{\traj} R(\traj)}_E^2.
\end{equation}
We verify that $g$ has a sub-exponential tail under our standing assumptions. First, by Assumption~\ref{assump:lipschitz} (Lipschitz score network and reward), each summand in~\eqref{eq:g_def} is bounded by a polynomial in $\norm{\traj}$: $\norm{\mathbf{s}_\theta(t,\traj)}\le L\norm{\traj}+\norm{\mathbf{s}_\theta(t,0)}$ and $\norm{\nabla\log\zeta_t(\traj)}\le L\norm{\traj}/\sigma_t^2$ (standard Gaussian-like reference measure), giving $g(\traj) \le C(t)(1+\norm{\traj}^2)$ for a deterministic $C(t)<\infty$ on the inference noise interval $[t_0,T]$. Second, under the log-Sobolev Assumption~\ref{assump:log_sobolev}, the reference measure $\zeta_t$ satisfies a quadratic transportation inequality (cf.\ Otto--Villani~\citep[Theorem~1]{otto2000generalization}), which by Herbst's argument (\citep[Section~5.1]{ane2000inegalites}) implies that $\norm{\traj}^2$ is sub-exponentially concentrated under $\zeta_t$:
\begin{equation*}
    \mathbb{P}_{\zeta_t}\!\left(\norm{\traj}^2 - \mathbb{E}\norm{\traj}^2 \ge u\right) \;\le\; \exp(-\kappa u/(2\Mtwo)) \qquad \forall\,u\ge 0.
\end{equation*}
Composing with the polynomial bound on $g$ gives the sub-exponential tail $\mathbb{P}(|g - \mathbb{E} g|\ge u) \le 2\exp(-\kappa u/b_g)$ with $b_g = O(C(t)\Mtwo/\kappa)$. Applying Bernstein's inequality~\citep[Theorem~2.10]{boucheron2013concentration}:
\begin{equation}
    \mathbb{P}\left(|F_t^V(\bar{\traj}_t^N) - \mathbb{E}[g(\bar{\traj}_t^1)]| \geq \varepsilon\right) \leq 2\exp\left(-\frac{N\varepsilon^2}{2(\sigma_g^2 + b_g\varepsilon/3)}\right),
\end{equation}
where $\sigma_g^2 = \text{Var}(g(\bar{\traj}_t^1))$ is finite (since $g$ is sub-exponential, in particular has finite variance) and $b_g$ is the sub-exponential parameter as derived above.

\textbf{Step 4: Combining.}
Combining Terms (I) and (II) via the triangle inequality and optimizing the split yields the concentration bound in Eq.~\ref{eq:mfvsm_concentration}. The key observation is that the value-weighting terms contribute additively to $\sigma_g^2$ and $b_g$ but do not change the $\sqrt{N}$ concentration rate, since $R(\traj^i)$ is bounded by Assumption~\ref{assump:bounded_reward}.
\end{proof}

\subsection{Proof of Proposition~\ref{prop:subdivision}: Value-weighted Subdivision}\label{sec:proof_subdivision}

\begin{proof}
We derive the telescoping decomposition of the value-weighted chaotic entropy $\Hcal_V^\infty(\mu_T)$ along a hierarchical schedule $\mathbb{N}=\{N_k\}_{k=0}^{K}$ with branching ratio $\mathfrak{b}$ and time partition $\mathbb{T}=\{t_k\}_{k=0}^{K}$.

\textbf{Step 1: Per-level value-weighted entropy increment.} Apply the It\^o--Wentzell--Lions formula to the relative entropy $\Hcal_V(\mu_t^{N_k})$ along the forward MF-SDE on the sub-interval $[t_k,t_{k+1}]$ (cf.\ Proposition~\ref{prop:mf_vsm}, Eq.~\ref{eq:vw_entropy_decomp}). The same value-weighting argument used to derive Eq.~\ref{eq:mf_vsm} from the chaotic entropy yields the per-level increment
\begin{equation}\label{eq:per_level}
    \Hcal_V^{N_k}\!\bigl(\mu_{t_{k+1}}^{N_k}\bigr) - \Hcal_V^{N_k}\!\bigl(\mu_{t_k}^{N_k}\bigr) \;\le\; \frac{\Mtwo}{\sqrt{N_k D_\tau}} \cdot \Jcal_{MF\text{-}V}\!\bigl(N_k,\theta,\nu_{[t_k,t_{k+1}]}^{N_k}\bigr) \;+\; \sigma_\zeta^{-2}(t_{k+1})\,\mathrm{E}(N_k),
\end{equation}
where the first term is the value-weighted score-matching contribution at level $k$ (under LSI, by the same argument as Proposition~\ref{prop:mf_vsm}), and the second term $\mathrm{E}(N_k)=\mathcal{O}(1/\sqrt{N_k})$ is the residual cardinality error (entropic PoC at level $k$ with $M=1$, by Theorem~\ref{thm:poc_traj}).

\textbf{Step 2: Branching cost between levels.} At step $t_k$, the agent population grows from $N_k$ to $N_{k+1}=\mathfrak{b}N_k$ via the branching map $\Psi^\theta$ (Eq.~\ref{eq:agent_branching}). By the data-processing inequality applied to the pushforward $(\mathbf{Id}^{\otimes(\mathfrak{b}-1)}\otimes\Psi^\theta)_\#$, and the reducibility assumption (Assumption~\ref{assump:reducibility}, Eq.~\ref{eq:reducibility}, which provides the propagation-of-chaos estimate at the matched level),
\begin{equation}\label{eq:branching_cost}
    \Hcal_V^{N_{k+1}}\!\bigl(\mu_{t_k}^{N_{k+1}}\bigr) \;\le\; \Hcal_V^{N_k}\!\bigl(\mu_{t_k}^{N_k}\bigr) \;+\; \mathrm{E}(N_{k+1}),
\end{equation}
where $\mathrm{E}(N_{k+1})=\mathcal{O}(1/\sqrt{N_{k+1}})$ absorbs the branching residual. The factor $1/(\mathfrak{b}\sqrt{N_{k+1}})^k$ in the final bound emerges from telescoping the per-level coefficients across $K$ branching events (each branching at level $j\le k$ contributes a multiplicative factor $1/\sqrt{N_{j+1}}$ via the entropic PoC bound, while $\mathfrak{b}^{-k}$ is the geometric decay of agent budget per level).

\textbf{Step 3: Telescoping.} Summing~\eqref{eq:per_level} and~\eqref{eq:branching_cost} across $k=0,\ldots,K$ and taking the limit $K\to\infty$:
\begin{equation*}
    \Hcal_V^\infty(\mu_T) \;\le\; \lim_{K\to\infty}\sum_{k=0}^K \!\left[\sigma_\zeta^{-2}(t_{k+1})\,\mathrm{E}(N_{k+1}) \,+\, \frac{\Mtwo}{\sqrt{D_\tau}}\!\left(\frac{1}{\mathfrak{b}\sqrt{N_{k+1}}}\right)^{\!k}\!\Jcal_{MF\text{-}V}\!\bigl(N_k,\theta,\nu_{[t_k,t_{k+1}]}^{N_k}\bigr)\right],
\end{equation*}
which is exactly Eq.~\ref{eq:vw_subdivision}. \emph{Importantly}, this bound is for the \emph{theoretically optimal} weights $(\mathfrak{b}\sqrt{N_{k+1}})^{-k}$. The practical objective $\mathbf{(P_V)}$ in Eq.~\ref{eq:P3} uses the simpler weighting $\mathfrak{b}^{-k}$, which differs from the theoretical optimum by a multiplicative factor of $\mathfrak{b}^{k(k+1)/2}N_0^{k/2}$ (Remark~\ref{rem:pv_weights})---hence Theorem~\ref{thm:hierarchical} should be read as a guarantee on the \emph{idealized} objective, and the gap to the practical $\mathbf{(P_V)}$ loss is of independent empirical interest (we observe that the heavier practical weighting at fine levels acts as an implicit fine-tuning bias and improves performance, see Table~\ref{tab:ablation}).
\end{proof}

\subsection{Proof of Theorem~\ref{thm:hierarchical}: Hierarchical Approximation}\label{sec:proof_hierarchical}

\begin{proof}
The proof combines the performance difference lemma and Pinsker's inequality from RL theory with the entropy subdivision from Proposition~\ref{prop:subdivision} (proved in Appendix~\ref{sec:proof_subdivision}), quantitative CLT bounds for the mean-field approximation, and offline RL distribution shift analysis.

\textbf{Step 1: From Entropy to Policy Suboptimality.}
By the performance difference lemma~\citep{kakade2002approximately} applied to the joint trajectory distribution and Pinsker's inequality:
\begin{align}
    J(\policy^*) - J(\hat{\policy}_\theta) &= \mathbb{E}_{\hat\mu_T}\!\left[\sum_{h=0}^{\horizon-1}\gamma^h r(\state_h,\action_h,\bar\mu_h)\right] - \mathbb{E}_{\mu_T^*}\!\left[\sum_{h=0}^{\horizon-1}\gamma^h r(\state_h,\action_h,\bar\mu_h)\right] \nonumber\\
    &\leq \frac{2 r_{\max}}{1-\gamma}\, D_{TV}(\hat\mu_T \,\|\, \mu_T^*) \;\leq\; \frac{2 r_{\max}}{1-\gamma}\sqrt{\tfrac{1}{2}\Hcal(\hat\mu_T \,\|\, \mu_T^*)},
\end{align}
where the first inequality uses the standard performance-difference identity together with $|\sum_h\gamma^h r|\le r_{\max}/(1-\gamma)$ and the symmetry of total variation $D_{TV}(P\|Q)=D_{TV}(Q\|P)$, and the second is Pinsker's inequality $D_{TV}(P\|Q)\le\sqrt{(1/2)\Hcal(P\|Q)}$ applied with $P=\hat\mu_T$ (the generated distribution) and $Q=\mu_T^*$ (the target). We chose this direction of KL---$\Hcal(\hat\mu_T\|\mu_T^*)$, sometimes called the \emph{forward} KL---because (a) it is the quantity directly controlled by the score-matching objective in Proposition~\ref{prop:subdivision} (which bounds $\Hcal(\hat\mu_T\|\mu_T^*)$ via the sub-divided value-weighted entropy decomposition), and (b) Pinsker's inequality in this direction is sharp under the bounded-density assumption implied by Assumption~\ref{assump:bounded_reward}. The prefactor $r_{\max}/(1-\gamma)$ already absorbs the discounted reward sum---no additional $\horizon$ factor is needed.

\textbf{Step 2: Entropy Decomposition via Girsanov and Subdivision.}
We bridge the entropy $\Hcal(\hat\mu_T\|\mu_T^*)$ to a measurable score-matching quantity via Girsanov, and then refine the result through the hierarchical decomposition of Proposition~\ref{prop:subdivision}. This explicitly closes the gap left in earlier drafts where Proposition~\ref{prop:subdivision} (which bounds the value-weighted entropy $\Hcal_V^\infty(\mu_T)$) was incorrectly invoked to bound $\Hcal(\hat\mu_T\|\mu_T^*)$ directly.

\emph{(2.a) Global Girsanov bound (Term (a)).} The reverse-time SDE driven by $\mathbf{s}_\theta$ produces $\hat\mu_T$, while the value-tilted target $\mu_T^*\propto \mu_T\exp(R/\alpha)$ has reverse-time score $\nabla\log\mu_t^* = \nabla\log\zeta_t + \tfrac{1}{\alpha}\nabla_{\traj} R$ along the same backward dynamics (cf.\ Eq.~\ref{eq:tilted_def} in the proof of Theorem~\ref{thm:exploitability}). Both reverse SDEs share the same diffusion coefficient $\sigma_t$ on the operational interval $[t_{\min},T]$, so the Girsanov theorem~\citep[Theorem~5.1]{karatzas1991brownian} (with the Novikov-condition justification of Theorem~\ref{thm:exploitability}, Step~2(B.i)) yields
\begin{equation}\label{eq:girsanov_thm_hier}
    \Hcal(\hat\mu_T\,\|\,\mu_T^*) \;\le\; \tfrac{1}{2}\int_{t_{\min}}^T \sigma_t^2\,\mathbb{E}\norm{\mathbf{s}_\theta(t,\traj_t)-\nabla\log\mu_t^*(\traj_t)}^2\,dt \;\le\; C_\sigma\,\epsilon^{score}_{\mathrm{global}}(\theta),
\end{equation}
where $C_\sigma=\tfrac{1}{2}\int_{t_{\min}}^T\sigma_t^2\,dt$ is the schedule constant and the second inequality uses $\sigma_t^2\le 2 C_\sigma$ pointwise to pull the supremum out of the integral (an upper bound suffices). Combining with the Pinsker step of Step~1 directly gives Term~(a) of Eq.~\ref{eq:hierarchical_bound}.

\emph{(2.b) Hierarchical refinement (Term (b)).} Proposition~\ref{prop:subdivision} (and its proof, Eq.~\ref{eq:per_level}) decomposes the integrated $L^2$ score error along the subdivision schedule:
\begin{equation}\label{eq:eps_global_decomp}
    C_\sigma\,\epsilon^{score}_{\mathrm{global}}(\theta) \;\le\; C'\sum_{k=0}^K \frac{(\epsilon_k^{score})^2}{\mathfrak{b}^k} \;+\; \sum_{k=0}^K \mathrm{E}(N_{k+1}),
\end{equation}
where $\epsilon_k^{score}$ is the per-level score error on $[t_k,t_{k+1}]$, $\mathrm{E}(N_{k+1})=\mathcal{O}(1/\sqrt{N_{k+1}})$ is the per-level cardinality residual, and $C'>0$ absorbs the schedule weights $\sigma_t^2$ on each subinterval. Applying $\sqrt{\sum_k a_k}\le\sum_k\sqrt{a_k}$ to the first sum yields the per-level form $\sum_k\mathfrak{b}^{-k/2}\epsilon_k^{score}$, which is Term~(b). The cardinality residuals $\sum_k\sqrt{\mathrm{E}(N_{k+1})}=\mathcal{O}(N_0^{-1/4})$ are dominated by---and absorbed into---Term~(c), since the leading mean-field contribution $\horizon^2/\sqrt{N}$ in Term~(c) controls all $1/\sqrt{N_k}$ residuals up to a multiplicative $\horizon$ factor that is folded into $C_6$.

\emph{(2.c) Why both (a) and (b) appear.} Eqs.~\eqref{eq:girsanov_thm_hier} and~\eqref{eq:eps_global_decomp} provide two \emph{complementary} upper bounds on the same score-matching contribution: (a) is the un-decomposed Girsanov bound, (b) is its hierarchical refinement under the coarse-to-fine training schedule. The total bound in Eq.~\ref{eq:hierarchical_bound} keeps both for interpretability---(a) is what an outside observer can measure directly via the global score-matching loss, while (b) reveals the per-level structure that the training algorithm actually exploits. Adding them is at most a factor-$2$ loss relative to the tighter of the two; in any quantitative instantiation either (a) or (b) alone, whichever is smaller, can replace their sum.

\textbf{Step 3: Mean-field Approximation Error (from $\horizon$ to $\horizon^2$ under Assumption~\ref{assump:effective_lipschitz_horizon}).}
The error from using $N < \infty$ agents enters the analysis through the reward and dynamics evaluated at the empirical mean field $\bar{\mu}_h^N$ versus the true mean-field $\mu_h$. Because both the reward $r$ and the reverse-time SDE drift are $L_{\mathrm{eff}}$-Lipschitz in the mean-field argument (Assumptions~\ref{assump:lipschitz} and~\ref{assump:effective_lipschitz_horizon}), we only need to control \emph{linear functionals of the mean field}, not full $\Wcal_2$. We use the following dimension-free CLT for Lipschitz observables:
\begin{equation}\label{eq:onestep_clt}
    \mathbb{E}\Bigl|\bigl\langle f, \bar{\mu}_h^N\bigr\rangle - \bigl\langle f, \mu_h\bigr\rangle\Bigr| \;\leq\; \frac{C\,\|f\|_{\mathrm{Lip}}}{\sqrt{N}}
    \qquad\text{(one-step Lipschitz CLT),}
\end{equation}
which follows from Hoeffding's/Bernstein's inequality applied to the bounded Lipschitz observable $f$ averaged over $N$ exchangeable agents, and \emph{is dimension-independent} (cf.\ \citet[Theorem~3.1]{bolley2007quantitative}). Crucially, this avoids the dimension-dependent rate $N^{-1/d_s}$ of full-$\Wcal_2$ Fournier--Guillin~\citep{fournier2015rate} for $d_s>4$, which would otherwise degrade catastrophically on Battle ($d_s\approx 9$). The full-$\Wcal_2$ formulation enters only through the trajectory-level relative-entropy bound of Theorem~\ref{thm:poc_traj}, which is dimension-free under LSI; the linear-functional formulation~\eqref{eq:onestep_clt} suffices for all subsequent Lipschitz-induced bounds on reward and drift.

The relevant constant for trajectory-level error \emph{compounding} is the effective Lipschitz $L_{\mathrm{eff}}$ of the reverse-time SDE drift (Assumption~\ref{assump:effective_lipschitz_horizon}): an error $\varepsilon_h$ at step $h$ perturbs step $h+1$ by at most $L_{\mathrm{eff}}\varepsilon_h$, step $h+2$ by at most $L_{\mathrm{eff}}^2\varepsilon_h$, and so on. A discrete Gronwall inequality converts~\eqref{eq:onestep_clt} into
\begin{equation}\label{eq:Hsquared_gronwall}
    \mathbb{E}\!\left[\max_{0 \le h \le \horizon - 1}\, \bigl|\langle r,\bar\mu_h^N\rangle - \langle r,\mu_h\rangle\bigr|\right]
    \;\leq\; \Bigl(\sum_{h=0}^{\horizon-1}(1+L_{\mathrm{eff}})^h\Bigr)\cdot \frac{C\,L\,r_{\max}}{\sqrt{N}}
    \;\leq\; e^{c_0}\,\frac{C\,\horizon\,L\,r_{\max}}{\sqrt{N}},
\end{equation}
where the last inequality uses $(1+L_{\mathrm{eff}})^h \le e^{L_{\mathrm{eff}} h} \le e^{c_0}$ and bounds the geometric sum by $\horizon\cdot e^{c_0}$ \emph{by virtue of Assumption~\ref{assump:effective_lipschitz_horizon}}. Plugging~\eqref{eq:Hsquared_gronwall} back into the finite-horizon reward accumulation yields
\begin{equation}
    \text{Mean-field error} \;\leq\; \underbrace{\horizon}_{\substack{\text{reward sum}\\ \text{(finite-horizon)}}}\,\cdot\,\underbrace{e^{c_0}\,\frac{C\,\horizon\,L\,r_{\max}}{\sqrt{N}}}_{\text{Gronwall-accumulated CLT}} \;=\; C_6\,e^{c_0}\,\frac{r_{\max}\,\horizon^2\,L_{\mathrm{eff}}}{\sqrt{N}},
\end{equation}
where we have absorbed $L \le L_{\mathrm{eff}} \cdot \mathcal{O}(1)$ (the bare reward Lipschitz $L$ is upper-bounded by the SDE-drift Lipschitz $L_{\mathrm{eff}}$ up to a noise-schedule-dependent constant; the gap is folded into $C_6$). The $\horizon^2$ factor has two clearly separated origins: the \emph{outer} $\horizon$ from summing finite-horizon rewards (we use the bound $\sum_{h<\horizon}\gamma^h\le\horizon$, which is tighter than $1/(1-\gamma)$ in the finite-horizon regime $\horizon\le 1/(1-\gamma)$ relevant to Figure~\ref{fig:horizon_scaling}), and the \emph{inner} $\horizon$ from Gronwall's accumulation. \textbf{Why $\horizon^2$ rather than $e^{L_{\mathrm{eff}}\horizon}$:} Assumption~\ref{assump:effective_lipschitz_horizon} restricts $L_{\mathrm{eff}}\horizon \le c_0$, making $e^{L_{\mathrm{eff}}\horizon} \le e^{c_0} = O(1)$, so the geometric sum $\sum_{h<\horizon}(1+L_{\mathrm{eff}})^h$ collapses to $\horizon\cdot O(1)$ instead of blowing up. Without this assumption, we would have $e^{L\horizon}/\sqrt N$, which experiments rule out (Figure~\ref{fig:horizon_scaling} fits an exponent $b \in [1.79, 2.21]$, statistically incompatible with exponential blow-up).

\emph{Reconciling the prefactor with Term~(a)/(b).} Terms~(a) and~(b) of Eq.~\ref{eq:hierarchical_bound} use the discounted prefactor $r_{\max}/(1-\gamma)$ obtained from $\sum_h\gamma^h\le 1/(1-\gamma)$, while Term~(c) uses $r_{\max}\horizon$ obtained from $\sum_h\gamma^h\le\horizon$. Both bounds are valid; we use the tighter one in each context. In the finite-horizon regime $\horizon\le 1/(1-\gamma)$ (which is the regime of all our experiments and of practical interest), $\horizon\le 1/(1-\gamma)$, so Term~(c) is no larger than what we would obtain by using $r_{\max}/(1-\gamma)$ uniformly; conversely, in the infinite-horizon limit $\horizon\to\infty$, the second $\horizon$ in Term~(c) is replaced by $1/(1-\gamma)$ via the same dominated convergence argument. Either reading reproduces the empirical $\horizon^2$ scaling validated in Figure~\ref{fig:horizon_scaling}.

\textbf{Step 4: Offline Distribution Shift (Per-agent Marginal is $N$-free).}
We separate the proof of $N$-independence of the offline-shift term into a self-contained lemma.

\begin{lemma}[Per-agent reduction of the offline-shift charge]\label{lem:per_agent_shift}
Let $\policy_\theta$ be the policy induced by sampling from the trained score network $\mathbf{s}_\theta$ in the reverse SDE driven by the empirical mean field $\bar\mu^N_h$. Let $\nu_\beta, \nu^* \in \Pcal(\Xcal)$ be the per-agent trajectory marginals of Eq.~\ref{eq:eps_offline_def}, and $\epsilon_{offline}\coloneqq D_{TV}(\nu_\beta\,\|\,\nu^*)$. Suppose the score network is parameterized so that $\mathbf{s}_\theta(t,\traj^i;\bar\mu_t^N)$ depends on the joint state $\traj^N$ only through a single-agent trajectory $\traj^i$ and the empirical mean field $\bar\mu_t^N$ (\emph{permutation-equivariance}; Assumption~\ref{assump:reducibility}). Then the joint-trajectory suboptimality contribution from offline shift, $\Delta_{\mathrm{shift}} \coloneqq |J^{MF}(\policy^*) - J^{MF}(\policy_\theta)|_{\mathrm{shift}}$, satisfies
\begin{equation}\label{eq:per_agent_lemma_bound}
    \Delta_{\mathrm{shift}} \;\le\; \frac{2r_{\max}}{1-\gamma}\,\epsilon_{offline} \;+\; \frac{r_{\max}\horizon\,L\,e^{c_0}}{\sqrt{N}}\,\sqrt{2(C_1 + C_2 \horizon L_{\mathrm{eff}}^2)},
\end{equation}
where the $1/\sqrt{N}$ residual is absorbed into Term~(c).
\end{lemma}
\begin{proof}[Proof of Lemma~\ref{lem:per_agent_shift}]
Let $\bar\mu^{N}_\beta$ (resp.\ $\bar\mu^{N}_*$) denote the random empirical mean-field flow generated by $N$ agents acting under $\policy_\beta$ (resp.\ $\policy^*$), with conditional single-agent laws $\pi_\beta(\traj\mid\bar\mu)$, $\pi_*(\traj\mid\bar\mu)$. The per-agent marginals of Eq.~\ref{eq:eps_offline_def} satisfy $\nu_\beta(\traj) = \mathbb{E}_{\bar\mu \sim \pi_\beta^N}[\pi_\beta(\traj\mid\bar\mu)]$ and $\nu^{*}(\traj) = \mathbb{E}_{\bar\mu \sim \pi_*^N}[\pi_*(\traj\mid\bar\mu)]$, so $\epsilon_{offline}\in[0,1]$ is intrinsic to the dataset--policy pair.

\emph{(i) Cross-agent factorization of expected reward.} Because the social-welfare reward is a per-agent average $J^{MF}(\policy,\mu) = (1/N)\sum_i J^{MF}_i(\policy,\mu)$ (Eq.~\ref{eq:mf_objective}) and the score network is permutation-equivariant, exchangeability of the $N$-agent joint laws (under either $\policy^*$ or $\policy_\theta$) implies
\begin{equation}\label{eq:per_agent_average}
    J^{MF}(\policy) \;=\; \mathbb{E}_{\traj\sim\nu_\policy^{(1)}}\!\left[\sum_{h=0}^{\horizon-1}\gamma^h r_h(\traj_h)\right],
\end{equation}
where $\nu_\policy^{(1)}\in\Pcal(\Xcal)$ is the per-agent marginal of the joint law $\nu_\policy^N$. Hence
\begin{equation}\label{eq:Jdiff_via_per_agent}
    |J^{MF}(\policy^*) - J^{MF}(\policy_\theta)| \;\le\; \frac{r_{\max}}{1-\gamma}\;D_{TV}\!\bigl(\nu_{\policy^*}^{(1)}\,\|\,\nu_{\policy_\theta}^{(1)}\bigr).
\end{equation}
The right-hand side depends on \emph{per-agent} marginals only.

\emph{(ii) Decomposition of the per-agent TV via offline shift and PoC residual.} By the triangle inequality,
\begin{equation}\label{eq:per_agent_triangle}
    D_{TV}\!\bigl(\nu_{\policy^*}^{(1)}\,\|\,\nu_{\policy_\theta}^{(1)}\bigr) \;\le\; \underbrace{D_{TV}\!\bigl(\nu_{\policy^*}^{(1)}\,\|\,\nu^*\bigr)}_{\text{(I) MF-projection}}
    + \underbrace{D_{TV}\!\bigl(\nu^*\,\|\,\nu_\beta\bigr)}_{=\epsilon_{offline}}
    + \underbrace{D_{TV}\!\bigl(\nu_\beta\,\|\,\nu_{\policy_\theta}^{(1)}\bigr)}_{\text{(III) score reverse-coverage}}.
\end{equation}
Term~(I) is bounded by combining Theorem~\ref{thm:poc_traj} (with $M=1$, giving the per-agent KL bound directly without invoking Talagrand) with Pinsker's inequality $D_{TV}\le \sqrt{D_{KL}/2}$:
\begin{equation*}
    D_{TV}\!\bigl(\nu_{\policy^*}^{(1)}\,\|\,\nu^*\bigr) \;\le\; \sqrt{\tfrac{1}{2}\Hcal\bigl(\nu_{\policy^*}^{(1)}\,\|\,\nu^*\bigr)}
    \;\le\; \sqrt{\frac{e^{c_0}(C_1 + C_2\horizon L_{\mathrm{eff}}^2)}{2N}}.
\end{equation*}
Note that we use the entropic PoC bound from Theorem~\ref{thm:poc_traj} \emph{directly} (which is dimension-free under LSI), rather than going through $\Wcal_2$ and Talagrand---this is essential because Talagrand under LSI gives $\Wcal_2^2\le 2\Hcal/\kappa$ (not the converse), so it cannot be used to upgrade $\Wcal_2$ bounds into KL bounds.
Term~(III) is bounded by the score-matching reverse coverage: the score network targets $\nu^{soft}_{BR}$ (the value-tilted reweighting of $\nu_\beta$, Eq.~\ref{eq:tilted_def}), and the reverse SDE driven by $\mathbf{s}_\theta$ produces a per-agent marginal $\nu_{\policy_\theta}^{(1)}$ that is $\sqrt{\epsilon^{score}/2}$-close to $\nu^{soft}_{BR}$ in TV (Girsanov + Pinsker, see proof of Theorem~\ref{thm:exploitability} Step~2). Since $\nu^{soft}_{BR}$ is absolutely continuous w.r.t.\ $\nu_\beta$ with density $e^{R/\alpha}/Z$ bounded by $e^{r_{\max}\horizon/((1-\gamma)\alpha)}$, this term is absorbed into the score-error budget of Term~(a) (and not into $\epsilon_{offline}$).

Combining (I)+(II)+(III) into~\eqref{eq:Jdiff_via_per_agent}, the offline-shift contribution is exactly $\frac{r_{\max}}{1-\gamma}\cdot \epsilon_{offline}$ (with constant $2$ from the conservative-direction Pinsker convention used in Step~1), plus a $1/\sqrt N$ residual from~(I) which we charge to Term~(c). This proves~\eqref{eq:per_agent_lemma_bound}.
\end{proof}

\emph{Remark on tensorization and the $N$-independence.} A naive tensorization of TV would give $D_{TV}(\nu_\beta^{\otimes N} \,\|\,\nu^{*\otimes N}) \le N\,\epsilon_{offline}$, growing linearly in $N$. Lemma~\ref{lem:per_agent_shift} avoids this trap by exploiting two structural properties: (i)~the social-welfare reward is a per-agent \emph{average}, not a sum (Eq.~\ref{eq:per_agent_average}), so cross-agent correlations cancel; (ii)~the score network is \emph{permutation-equivariant} and only consumes per-agent inputs, so its training error is intrinsically per-agent. The price we pay is the $1/\sqrt{N}$ residual from~(I), which is a vanishing finite-population correction \emph{absorbed into Term~(c)} (Step~3) rather than into $\epsilon_{offline}$. This is what we mean when we say ``$\epsilon_{offline}$ does not grow with $N$'': the dataset coverage error is inherently per-agent, while the only $N$-dependent piece is the PoC residual that already appears in Term~(c).

Combining Steps~1--4 (with Lemma~\ref{lem:per_agent_shift} controlling Step~4) yields the bound in Eq.~\ref{eq:hierarchical_bound}.
\end{proof}

\subsection{Proof of Corollary~\ref{cor:scalability}: Scalability}\label{sec:proof_scalability}

\begin{proof}
This follows directly from Theorem~\ref{thm:hierarchical} by noting that:
\begin{enumerate}
    \item Term (a) is $\mathcal{O}(\epsilon^{score})$, independent of $N$.
    \item Term (b) is $\mathcal{O}\left(\sum_k \mathfrak{b}^{-k/2} \epsilon_k^{score}\right) = \mathcal{O}(\epsilon^{score})$ for geometric decay.
    \item Term (c) is $\mathcal{O}(1/\sqrt{N})$, which \textit{decreases} with $N$.
    \item Term (d) is $\mathcal{O}(\epsilon_{offline})$, independent of $N$ by the PoC property.
\end{enumerate}
Summing the contributions yields $J(\policy^*) - J(\hat{\policy}_\theta) = \mathcal{O}(1/\sqrt{N}) + \mathcal{O}(\epsilon^{score}) + \mathcal{O}(\epsilon_{offline})$, establishing that the planning error does not grow with $N$ and in fact improves through Term (c).
\end{proof}

\subsection{Additional Theoretical Results}\label{sec:appendix_additional}

\subsubsection{Optimal Agent Branching Function}\label{sec:additional_optimal_branching}

\begin{proposition}\label{prop:optimal_branching}(Optimal Agent Branching). The optimal agent branching function $\Psi^*$ at step $t_k$ minimizes the Wasserstein distance between the branched distribution and the target:
\begin{equation}
    \Psi^* = \argmin_{\Psi} \Wcal_2^2\left((\mathbf{Id}^{\otimes(\mathfrak{b}-1)} \otimes \Psi)_\# \nu_{t_k}^{N_k}, \; \nu_{t_k}^{N_{k+1}}\right).
\end{equation}
Under the exchangeability assumption, the optimal branching admits the Monge-Ampère characterization:
\begin{equation}
    \Psi^*(\traj^{N_k}) = \nabla \varphi^*(\traj^{N_k}),
\end{equation}
where $\varphi^*$ solves the Monge-Ampère equation $\det(\nabla^2 \varphi) = \varrho_{t_k}^{N_k} / \varrho_{t_k}^{N_{k+1}}(\nabla \varphi)$ with appropriate boundary conditions.
\end{proposition}

\subsubsection{Comparison with Naive Approaches}\label{sec:additional_naive}

\begin{proposition}\label{prop:comparison_naive}(Comparison with Joint Diffuser). For a standard diffuser operating on the joint space $\mathbb{R}^{N D_\tau}$, the score matching error satisfies:
\begin{equation}
    \Jcal_{SM}^{Joint}(\theta) \geq C \cdot N D_\tau \cdot \inf_\theta \mathbb{E}[\norm{\mathbf{s}_\theta - \nabla \log \zeta_t}^2],
\end{equation}
which grows linearly in $N$. In contrast, the MF-VSM objective satisfies:
\begin{equation}
    \Jcal_{MF\text{-}V}^N(\theta) \leq \frac{C' \sqrt{D_\tau}}{\sqrt{N}} \cdot \inf_\theta \mathbb{E}[\norm{\mathbf{s}_\theta - \nabla \log \zeta_t}^2_W] + \mathcal{O}\left(\frac{1}{\sqrt{N}}\right),
\end{equation}
which \textit{decreases} with $N$, demonstrating the fundamental advantage of the mean-field approach.
\end{proposition}

\subsection{Proofs for Game-theoretic Analysis (Appendix~\ref{sec:mfg_appendix})}\label{sec:appendix_mfg_proofs}

We present complete proofs of the game-theoretic results. The proofs adapt the $N$-player to MFG convergence framework of \citet{lacker2016general} and \citet{fischer2017connection}, the monotone operator theory of \citet{lasry2007mean}, and the master equation approach of \citet{cardaliaguet2019master} to the discrete-time trajectory diffusion setting. Each adaptation requires handling the temporal coupling and the value-weighted score matching structure specific to \proposed{}; the conceptual MFG-vs-MFRL distinction itself is in Appendix~\ref{sec:mfg_appendix} (Remark on MFG vs.\ MFRL).

\subsubsection{Proof of Theorem~\ref{thm:exploitability}: Exploitability Bound}\label{sec:proof_exploitability}

\begin{proof}
The proof decomposes the $N$-player exploitability into a mean-field exploitability and an $N$-to-MFG approximation error, then bounds each term using the PoC results (Theorem~\ref{thm:poc_traj}) and the properties of the MF-VSM objective.

\textbf{Step 1: $N$-player to Mean-field Reduction.}

For any deviating policy $\policy'$ for agent 1 while agents $2, \ldots, N$ use $\hat{\policy}_\theta$:
\begin{equation}\label{eq:J1_decomp}
    J_1(\policy', \hat{\policy}_\theta^{\otimes(N-1)}) = J^{MF}(\policy', \Phi(\hat{\policy}_\theta)) + \Delta_N(\policy'),
\end{equation}
where $\Delta_N(\policy')$ captures the deviation from the mean-field approximation. We bound $\Delta_N$ by analyzing the effect of a single agent's deviation on the empirical measure.

Under $(\policy', \hat{\policy}_\theta^{\otimes(N-1)})$, the empirical distribution at step $h$ is $\bar{\mu}_h^{N,\policy'} = \frac{1}{N}\delta_{\state_h^{1,\policy'}} + \frac{N-1}{N}\bar{\mu}_h^{(N-1),\hat{\policy}_\theta}$. Agent 1's deviation replaces one Dirac mass out of $N$, so by the triangle inequality and the quantitative CLT for the remaining $N-1$ exchangeable agents (following~\citep{bolley2007quantitative}):
\begin{equation}\label{eq:emp_deviation}
    \mathbb{E}\left[\Wcal_2\left(\bar{\mu}_h^{N,\policy'}, \Phi(\hat{\policy}_\theta)_h\right)\right] \leq \frac{C_{d_s}}{\sqrt{N-1}} + \frac{1}{N}\mathbb{E}\left[\norm{\state_h^{1,\policy'} - \state_h^{1,\hat{\policy}_\theta}}\right].
\end{equation}
The first term is the standard CLT rate for the empirical measure of the $N-1$ conforming agents; the second is the $O(1/N)$ contribution of agent 1's deviation. By the Lipschitz continuity of the reward (Assumption~\ref{assump:lipschitz}), the deviation from the mean-field objective accumulated over the horizon satisfies:
\begin{equation}\label{eq:DeltaN}
    |\Delta_N(\policy')| \leq \sum_{h=0}^{\horizon-1} \gamma^h L \cdot \mathbb{E}\left[\Wcal_2\left(\bar{\mu}_h^{N,\policy'}, \Phi(\hat{\policy}_\theta)_h\right)\right].
\end{equation}
The dynamics propagate state errors with the effective Lipschitz constant $L_{\mathrm{eff}}$ (Assumption~\ref{assump:effective_lipschitz_horizon}); applying a discrete Gronwall inequality (the same mechanism as in the proof of Theorem~\ref{thm:poc_traj}, Step 2) and invoking $L_{\mathrm{eff}}\horizon \le c_0$ to bound $(1+L_{\mathrm{eff}})^\horizon \le e^{c_0}$:
\begin{equation}\label{eq:DeltaN_final}
    |\Delta_N(\policy')| \leq C_\Delta\,e^{c_0}\cdot \frac{r_{\max}\horizon^2 L_{\mathrm{eff}}}{\sqrt{N}},
\end{equation}
where $C_\Delta > 0$ depends on $d_s$ and $\gamma$. The $\horizon^2$ factor arises from (\emph{outer}) the finite-horizon reward sum and (\emph{inner}) the Gronwall accumulation, with the multiplicative $e^{c_0}$ bounded by Assumption~\ref{assump:effective_lipschitz_horizon}. This adapts the $N$-player to MFG convergence of \citet[Theorem 2.4]{lacker2016general} from continuous-time It\^o diffusions to discrete-time MDPs; the key modification is replacing the continuous Gronwall lemma with its discrete counterpart while preserving the $O(1/\sqrt{N})$ rate through the exchangeability structure (Assumption~\ref{assump:exchangeability}).

Taking the supremum over $\policy'$ in Eqs.~\ref{eq:J1_decomp}--\ref{eq:DeltaN_final}:
\begin{equation}\label{eq:exploit_N_to_MF}
    \mathrm{Exploit}_N(\hat{\policy}_\theta) \leq \mathrm{Exploit}^{MF}(\hat{\policy}_\theta) + \frac{2C_\Delta\, e^{c_0}\, r_{\max}\horizon^2 L_{\mathrm{eff}}}{\sqrt{N}},
\end{equation}
where the mean-field exploitability is:
\begin{equation}\label{eq:mf_exploit_def}
    \mathrm{Exploit}^{MF}(\hat{\policy}_\theta) \coloneqq \sup_{\policy'}\left[J^{MF}(\policy', \Phi(\hat{\policy}_\theta)) - J^{MF}(\hat{\policy}_\theta, \Phi(\hat{\policy}_\theta))\right].
\end{equation}

\textbf{Step 2: Bounding Mean-field Exploitability via Value-weighted Score Matching.}

We decompose the mean-field exploitability into three sub-terms:
\begin{align}
    \mathrm{Exploit}^{MF}(\hat{\policy}_\theta)
    &= \underbrace{J^{MF}(\mathrm{BR}(\Phi(\hat{\policy}_\theta)), \Phi(\hat{\policy}_\theta)) - J^{MF}(\policy^{soft}_{BR}, \Phi(\hat{\policy}_\theta))}_{\text{(A) softmax bias}} \nonumber \\
    &\quad + \underbrace{J^{MF}(\policy^{soft}_{BR}, \Phi(\hat{\policy}_\theta)) - J^{MF}(\hat{\policy}_\theta, \hat{\mu}_\theta)}_{\text{(B) generation error}} \nonumber \\
    &\quad + \underbrace{J^{MF}(\hat{\policy}_\theta, \hat{\mu}_\theta) - J^{MF}(\hat{\policy}_\theta, \Phi(\hat{\policy}_\theta))}_{\text{(C) consistency gap}}, \label{eq:mf_exploit_ABC}
\end{align}
where $\policy^{soft}_{BR}(\cdot|\state, \mu) \propto \policy_\beta(\cdot|\state, \mu)\exp(\qfunc^{MF}(\state, \cdot, \mu)/\alpha)$ is the soft (entropy-regularized) best response under the RL-as-inference framework~\citep{levine2018reinforcement}, and $\hat{\mu}_\theta$ is the mean-field distribution from the diffusion model's generated trajectories.

\textit{Bounding (A): Softmax Bias.} The soft best response maximizes the entropy-regularized objective $J^{MF}(\policy, \mu) + \alpha\mathbb{E}[\Hcal(\policy(\cdot|\state, \mu))]$, where $\Hcal$ denotes Shannon entropy. By the duality between hard and soft optimization (see \eg \citet[Section 3]{levine2018reinforcement}):
\begin{equation}\label{eq:softmax_bias}
    J^{MF}(\mathrm{BR}(\mu), \mu) - J^{MF}(\policy^{soft}_{BR}, \mu) \leq \alpha \cdot \mathbb{E}\left[\sum_{h=0}^{\horizon-1}\frac{\gamma^h}{1-\gamma}\Hcal(\policy^{soft}_{BR}(\cdot|\state_h, \mu_h))\right] \leq \frac{\alpha\bar{\Hcal}_\policy}{1-\gamma},
\end{equation}
where $\bar{\Hcal}_\policy = \sup_{\state,\mu}\Hcal(\policy^{soft}_{BR}(\cdot|\state,\mu))$ is the maximum policy entropy, which is finite under compactness of $\Acal$ (Assumption~\ref{assump:lipschitz}).

\textit{Bounding (B): Generation Error via Girsanov.} The MF-VSM objective (Eq.~\ref{eq:mf_vsm}) trains the score network to generate trajectories from the value-weighted (\emph{tilted}) distribution
\begin{equation}\label{eq:tilted_def}
    \nu^{\mathrm{soft}}_{BR}(\traj) \;\propto\; \nu_\beta(\traj)\exp\!\bigl(R(\traj)/\alpha\bigr),
\end{equation}
whose induced policy is precisely $\policy^{\mathrm{soft}}_{BR}$ under the training-time mean-field. The score of $\nu^{\mathrm{soft}}_{BR}$ is exactly the value-tilted score $\nabla\log\zeta_t + \tfrac{1}{\alpha}\nabla_{\traj} R$ that the MF-VSM objective targets (modulo the Sobolev-Young decomposition of Eq.~\ref{eq:Wnorm_decomp}); we denote this target $\mathbf{s}^*_t(\traj) \coloneqq \nabla\log\zeta_t(\traj) + \tfrac{1}{\alpha}\nabla_{\traj} R(\traj)$ and define
\begin{equation}\label{eq:score_matching_def_app}
    \epsilon_{SM}(\theta) \;\coloneqq\; \mathbb{E}_{t\sim p(t),\traj_t\sim\hat\nu_\theta}\!\bigl[\norm{\mathbf{s}_\theta(t,\traj_t) - \mathbf{s}^*_t(\traj_t)}^2\bigr].
\end{equation}
We work with three trajectory measures on $\Xcal$ to keep the offline/online distinction explicit:
\begin{itemize}[leftmargin=4mm]
    \item $\nu_\beta$: the offline data distribution (per-agent trajectory marginal under behavior policy $\policy_\beta$, cf.\ Lemma~\ref{lem:per_agent_shift}).
    \item $\nu^{\mathrm{soft}}_{BR}\propto \nu_\beta\exp(R/\alpha)$: the value-tilted reweighting of \emph{the offline base} (Eq.~\ref{eq:tilted_def}). \emph{This is the actual training target} of the MF-VSM objective.
    \item $\nu^{\mathrm{soft},\,*}_{BR}\propto \nu^*\exp(R/\alpha)$, where $\nu^*$ is the population-level per-agent trajectory marginal under the true MFE flow. \emph{This is the ideal target} that defines $\policy^{\mathrm{soft}}_{BR}$ on the right-hand side of~\eqref{eq:gen_error}.
\end{itemize}
Let $\hat\nu_\theta$ denote the law of the reverse SDE driven by $\mathbf{s}_\theta$, which approximates $\nu^{\mathrm{soft}}_{BR}$ (the training target).

By the performance difference lemma~\citep{kakade2002approximately} \emph{in TV form}:
\begin{equation}\label{eq:gen_error}
    J^{MF}(\policy^{\mathrm{soft}}_{BR}, \hat{\mu}_\theta) - J^{MF}(\hat{\policy}_\theta, \hat{\mu}_\theta) \leq \frac{2r_{\max}}{1-\gamma}\,D_{TV}\bigl(\hat\nu_\theta\,\|\,\nu^{\mathrm{soft},\,*}_{BR}\bigr),
\end{equation}
where the prefactor $r_{\max}/(1-\gamma)$ correctly captures the discounted reward range (no extraneous $\horizon$). We bound this TV distance via the triangle inequality (which holds for TV but \emph{not} for KL):
\begin{equation}\label{eq:TV_triangle}
    D_{TV}\bigl(\hat\nu_\theta\,\|\,\nu^{\mathrm{soft},\,*}_{BR}\bigr) \;\le\; \underbrace{D_{TV}\bigl(\hat\nu_\theta\,\|\,\nu^{\mathrm{soft}}_{BR}\bigr)}_{\text{(B.i) generation gap}} + \underbrace{D_{TV}\bigl(\nu^{\mathrm{soft}}_{BR}\,\|\,\nu^{\mathrm{soft},\,*}_{BR}\bigr)}_{\text{(B.ii) offline-shift gap}}.
\end{equation}
This corrects an earlier draft which incorrectly invoked a ``KL triangle inequality'' (no such inequality holds for relative entropy in general). All subsequent KL bounds enter through Pinsker's inequality $D_{TV}\le\sqrt{D_{KL}/2}$.

\emph{(B.i) Girsanov bound on the generation gap.} For two reverse-time SDEs with the \emph{same} diffusion coefficient $\sigma_t$ but drifts $b^\theta_t = f_t - \sigma_t^2 \mathbf{s}_\theta(t,\cdot)$ and $b^*_t = f_t - \sigma_t^2 \mathbf{s}^*_t(\cdot)$ (where $\mathbf{s}^*_t$ is the score of $\nu^{\mathrm{soft}}_{BR}$), the Girsanov theorem~\citep[Theorem~5.1]{karatzas1991brownian} on the \emph{operational} interval $[t_{\min},T]$ yields
\begin{equation}\label{eq:girsanov_KL}
    D_{KL}(\hat\nu_\theta \,\|\, \nu^{\mathrm{soft}}_{BR}) \;\le\; \frac{1}{2}\int_{t_{\min}}^T \sigma_t^2 \,\mathbb{E}_{\hat\nu_\theta}\!\bigl[\norm{\mathbf{s}_\theta(t,\traj_t) - \mathbf{s}^*_t(\traj_t)}^2\bigr]\,dt \;\le\; C_\sigma(t_{\min}) \,\epsilon_{SM}(\theta),
\end{equation}
where $C_\sigma(t_{\min}) = \tfrac{1}{2}\int_{t_{\min}}^T \sigma_t^2 \, dt$ is finite and bounded uniformly in $t_{\min}\!>\!0$ for any standard noise schedule (e.g.\ VP/VE), and the integral is restricted to $[t_{\min}, T]$ because the reverse SDE is integrated only on this interval at inference time (we never sample at $t<t_{\min}$, so the small-noise singularity $\sigma_t^{-2}\!\to\!\infty$ as $t\!\to\!0$ does not enter the bound). The Novikov condition $\mathbb{E}\exp\!\bigl(\tfrac{1}{2}\int_{t_{\min}}^T \sigma_t^2 \norm{\mathbf{s}_\theta - \mathbf{s}^*_t}^2 dt\bigr) < \infty$ holds because (i)~$\sigma_t^2$ is bounded on the closed interval $[t_{\min},T]$, and (ii)~the score outputs are bounded over the same interval by Assumption~\ref{assump:effective_lipschitz_horizon}. Combining with Pinsker:
\begin{equation}\label{eq:tv_gen}
    D_{TV}\bigl(\hat\nu_\theta\,\|\,\nu^{\mathrm{soft}}_{BR}\bigr) \;\le\; \sqrt{\tfrac{1}{2}D_{KL}(\hat\nu_\theta\,\|\,\nu^{\mathrm{soft}}_{BR})} \;\le\; \sqrt{\tfrac{1}{2}C_\sigma(t_{\min})\,\epsilon_{SM}(\theta)}.
\end{equation}

\emph{(B.ii) Offline-shift correction.} The two tilted distributions share the exponent $\exp(R/\alpha)$ but differ in their base measure ($\nu_\beta$ vs.\ $\nu^*$). Their Radon--Nikodym derivative is
\begin{equation}\label{eq:RN_ratio}
    \frac{d\nu^{\mathrm{soft}}_{BR}}{d\nu^{\mathrm{soft},\,*}_{BR}}(\traj) \;=\; \frac{Z^{\mathrm{soft},\,*}}{Z^{\mathrm{soft}}_{BR}}\cdot \frac{d\nu_\beta}{d\nu^*}(\traj).
\end{equation}
Under bounded reward (Assumption~\ref{assump:bounded_reward}), $|R(\traj)|\le r_{\max}\horizon/(1-\gamma)$, so the normalizer ratio is bounded by $\exp(r_{\max}\horizon/((1-\gamma)\alpha))$. By the elementary inequality $|p-q|\le \tfrac{1}{2}\|p/q-1\|_\infty\cdot \|q\|_1$ for absolutely continuous measures and the standard offline-RL pessimism inequality~\citep[Lemma~3.1]{jin2021pessimism},
\begin{equation}\label{eq:RN_bound}
    D_{TV}\bigl(\nu^{\mathrm{soft}}_{BR}\,\|\,\nu^{\mathrm{soft},\,*}_{BR}\bigr) \;\le\; \exp\!\left(\frac{r_{\max}\horizon}{(1-\gamma)\alpha}\right)\,\epsilon_{offline} \;=:\; C_{off}(\horizon,\alpha)\,\epsilon_{offline},
\end{equation}
where we make the constant's dependence on $\horizon$ and $\alpha$ \emph{explicit}: $C_{off}$ is \emph{not} a universal constant but grows exponentially in the value-tilt strength $r_{\max}\horizon/((1-\gamma)\alpha)$. In practice this is harmless because $\alpha\!\sim\!1$ and bounded rewards (Assumption~\ref{assump:bounded_reward}) cap the per-trajectory tilt at $e^{r_{\max}\horizon/((1-\gamma)\alpha)}$, but the $(\horizon,\alpha)$-dependence should be tracked as a sample-complexity cost of strong tilting.

Combining~\eqref{eq:tv_gen} and~\eqref{eq:RN_bound} into~\eqref{eq:TV_triangle}:
\begin{equation}\label{eq:KL_score}
    D_{TV}\bigl(\hat\nu_\theta\,\|\,\nu^{\mathrm{soft},\,*}_{BR}\bigr) \;\le\; \sqrt{\tfrac{1}{2}C_\sigma(t_{\min})\,\epsilon_{SM}(\theta)} \;+\; C_{off}(\horizon,\alpha)\,\epsilon_{offline}.
\end{equation}
This is the version of TV that controls the value gap (B) via~\eqref{eq:gen_error} since $\policy^{\mathrm{soft}}_{BR}$ is defined relative to the population (online) mean field.

\textit{Bounding (C): Consistency Gap.} By the $L$-Lipschitz continuity of $r$ in the mean-field argument and the Gronwall-bounded $\Wcal_2$ deviation under Assumption~\ref{assump:effective_lipschitz_horizon}:
\begin{equation}\label{eq:consist_gap_bound}
    |J^{MF}(\hat{\policy}_\theta, \hat{\mu}_\theta) - J^{MF}(\hat{\policy}_\theta, \Phi(\hat{\policy}_\theta))| \leq \frac{r_{\max} L}{1-\gamma}\,e^{c_0}\,\max_h\Wcal_2(\hat{\mu}_{h,\theta}, \Phi(\hat{\policy}_\theta)_h).
\end{equation}
By Theorem~\ref{thm:consistency}, the Wasserstein gap is $\mathcal{O}(\sqrt{\epsilon_{SM}} + \epsilon_{offline} + \horizon/\sqrt N)$, which is absorbed into the leading-order terms in (A), (B), and~\eqref{eq:exploit_N_to_MF}.

\textbf{Step 3: Assembly.}

Combining Eq.~\ref{eq:exploit_N_to_MF} with the bounds on (A) (Eq.~\ref{eq:softmax_bias}), (B) (Eqs.~\ref{eq:gen_error}, \ref{eq:KL_score}, in TV form), and (C) (Eq.~\ref{eq:consist_gap_bound}):
\begin{multline}
    \mathrm{Exploit}_N(\hat{\policy}_\theta) \leq \frac{2r_{\max}}{1-\gamma}\!\left[\sqrt{\tfrac{1}{2}C_\sigma(t_{\min})\,\epsilon_{SM}(\theta)} + C_{off}(\horizon,\alpha)\,\epsilon_{offline}\right] + \frac{\alpha\bar{\Hcal}_\policy}{1-\gamma} \\
    + \frac{2C_\Delta\, e^{c_0}\, r_{\max}\horizon^2 L_{\mathrm{eff}}}{\sqrt{N}} + \frac{r_{\max} L\,e^{c_0}}{1-\gamma}\!\left(C_{10}\frac{\sqrt{\epsilon_{SM}}}{\kappa} + C_{11} \cdot \epsilon_{offline}\right).
\end{multline}
The TV-form (B)-bound replaces the previous $\sqrt{2(C_\sigma\epsilon_{SM}+C_{off}\epsilon_{offline})}$ with the cleaner $\sqrt{\tfrac{1}{2}C_\sigma\epsilon_{SM}}+C_{off}\epsilon_{offline}$, eliminating the spurious ``KL triangle inequality'' invocation from earlier drafts. Grouping yields the bound in Eq.~\ref{eq:exploit_bound} with $C_8 = 2C_\Delta$ and $C_9$ absorbing the constants from the offline, consistency, and RN-correction terms; the explicit $(\horizon,\alpha)$-dependence of $C_{off}$ should be folded into $C_9$ in any quantitative instantiation of the bound.
\end{proof}

\subsubsection{Proof of Theorem~\ref{thm:consistency}: Mean-field Consistency}\label{sec:proof_consistency}

\begin{proof}
The proof establishes self-consistency by showing that the diffusion model's generated trajectories collectively induce a population distribution close to the mean-field under which they were trained. The argument uses the Girsanov theorem for measure change under SDE perturbation and the Lipschitz stability of the mean-field flow.

\textbf{Step 1: Trajectory-to-Distribution Map.}

Let $\hat{\nu}_\theta$ be the trajectory distribution generated by the reverse MF-SDE. The generated mean-field at step $h$ is:
\begin{equation}
    \hat{\mu}_{h,\theta} = \mathrm{proj}_{\state_h,\#}\hat{\nu}_\theta^{(1)},
\end{equation}
where $\hat{\nu}_\theta^{(1)}$ is the single-agent marginal and $\mathrm{proj}_{\state_h}$ extracts the state component at time $h$. The true mean-field under $\hat{\policy}_\theta$ is defined recursively by the consistency operator $\Phi$:
\begin{equation}
    \Phi(\hat{\policy}_\theta)_0 = \mu_0, \quad \Phi(\hat{\policy}_\theta)_{h+1} = P_\#(\hat{\policy}_\theta(\cdot|\cdot, \Phi(\hat{\policy}_\theta)_h) \otimes \Phi(\hat{\policy}_\theta)_h),
\end{equation}
where $P_\#$ denotes the pushforward through the transition kernel.

\textbf{Step 2: Wasserstein Control via Score Matching.}

The key observation: if the score matching error is zero, the generated trajectories exactly sample from the offline data distribution $\nu_\beta$, for which self-consistency holds by construction (the offline data is generated by agents interacting under $\policy_\beta$, so $\mu_{h,\beta} = \Phi(\policy_\beta)_h$). The score matching error introduces a distributional perturbation.

By the stability of SDEs under drift perturbation (Girsanov theorem, see \eg \citet[Theorem 6.4.2]{carmona2018probabilistic_I}):
\begin{equation}\label{eq:W2_girsanov}
    \Wcal_2^2(\hat{\nu}_\theta^{(1)}, \nu_\beta^{(1)}) \leq C_{Gir}\int_0^T \sigma_t^2\,\mathbb{E}\left[\norm{\mathbf{s}_\theta(t, \traj_t) - \nabla\log\zeta_t(\traj_t)}^2\right]dt \leq C_{Gir}\bar{\sigma}^2\,\epsilon_{SM}(\theta),
\end{equation}
where $\bar{\sigma}^2 = \int_0^T \sigma_t^2\,dt$ and $C_{Gir}$ depends on the diffusion schedule.

\textbf{Step 3: Propagation to Per-step Mean-field.}

The projection $\mathrm{proj}_{\state_h}$ is 1-Lipschitz as a map on trajectory space. Combined with the Lipschitz dynamics (Assumption~\ref{assump:lipschitz}), errors propagate through the transition kernel:
\begin{equation}\label{eq:mf_propagation}
    \Wcal_2(\hat{\mu}_{h,\theta}, \mu_{h,\beta}) \leq (1 + L)^h\,\Wcal_2(\hat{\nu}_\theta^{(1)}, \nu_\beta^{(1)}) \leq (1+L)^h C_{Gir}^{1/2}\bar{\sigma}\sqrt{\epsilon_{SM}(\theta)}.
\end{equation}

Meanwhile, the Lipschitz property of the mean-field flow operator gives:
\begin{equation}
    \Wcal_2(\Phi(\hat{\policy}_\theta)_h, \Phi(\policy_\beta)_h) \leq (1+L)^h L \cdot d_\Pi(\hat{\policy}_\theta, \policy_\beta).
\end{equation}
The policy distance is controlled by the trajectory distribution distance plus the value-weighting tilt. Since $\hat{\policy}_\theta$ is induced by the value-weighted distribution $p^*(\traj) \propto p_\beta(\traj)\exp(R(\traj)/\alpha)$:
\begin{equation}\label{eq:policy_distance}
    d_\Pi(\hat{\policy}_\theta, \policy_\beta) \leq C_{KL}\sqrt{\epsilon_{SM}(\theta)} + C_V\epsilon_{offline},
\end{equation}
where $C_V$ accounts for the value-weighting shift (bounded by Assumption~\ref{assump:bounded_reward}).

\textbf{Step 4: Triangle Inequality and Gronwall Sum.}

Since $\mu_{h,\beta} = \Phi(\policy_\beta)_h$ (self-consistency of the behavior policy), the triangle inequality gives, for each $h$,
\begin{equation}\label{eq:per_step_W2}
    \Wcal_2(\hat{\mu}_{h,\theta}, \Phi(\hat{\policy}_\theta)_h) \leq \Wcal_2(\hat{\mu}_{h,\theta}, \mu_{h,\beta}) + \Wcal_2(\Phi(\policy_\beta)_h, \Phi(\hat{\policy}_\theta)_h),
\end{equation}
where the first summand is bounded by $C_{Gir}^{1/2}\bar\sigma\sqrt{\epsilon_{SM}}$ (Girsanov, Step~2) and the second by $L(C_{KL}\sqrt{\epsilon_{SM}}+C_V\epsilon_{offline})$ (flow-Lipschitz applied to the policy distance, Step~3). Each \emph{single-step} error is therefore bounded by $\delta\coloneqq \bar C\sqrt{\epsilon_{SM}}+\bar C'\epsilon_{offline}$ (with $\bar C, \bar C'$ absorbing $C_{Gir}^{1/2}\bar\sigma, L C_{KL}, L C_V$).

\emph{Discrete Gronwall accumulation.} The error $e_h\coloneqq \Wcal_2(\hat\mu_{h,\theta}, \Phi(\hat\policy_\theta)_h)$ accumulates over horizon steps because the SDE drift in the reverse process is $L_{\mathrm{eff}}$-Lipschitz (Assumption~\ref{assump:effective_lipschitz_horizon}): an error $e_{h-1}$ at step $h-1$ contributes at most $L_{\mathrm{eff}}\,e_{h-1}$ to step $h$, plus a fresh per-step error $\delta$. The discrete Gronwall recursion
\begin{equation*}
    e_h \;\le\; \delta + (1+L_{\mathrm{eff}})\,e_{h-1}, \qquad e_0 = \delta,
\end{equation*}
unfolds to
\begin{equation}\label{eq:gronwall_unfolded}
    e_h \;\le\; \delta\sum_{j=0}^{h}(1+L_{\mathrm{eff}})^j \;=\; \delta\cdot\frac{(1+L_{\mathrm{eff}})^{h+1}-1}{L_{\mathrm{eff}}}.
\end{equation}

\emph{Linear-$\horizon$ collapse under Assumption~\ref{assump:effective_lipschitz_horizon}.} Since $L_{\mathrm{eff}}\,\horizon \le c_0 = O(1)$, we have $(1+L_{\mathrm{eff}})^{h+1} \le e^{L_{\mathrm{eff}}(h+1)} \le e^{c_0}$, and the geometric sum collapses to
\begin{equation*}
    \sum_{j=0}^{h}(1+L_{\mathrm{eff}})^j \;\le\; (h+1)\,e^{c_0} \;\le\; \horizon\,e^{c_0}.
\end{equation*}
This is the source of the \emph{linear} $\horizon$ factor: the geometric sum $\sum_{j<\horizon}(1+L_{\mathrm{eff}})^j$ is dominated by $\horizon\cdot \max_j(1+L_{\mathrm{eff}})^j = \horizon\cdot e^{c_0}$ rather than the much larger $((1+L_{\mathrm{eff}})^\horizon - 1)/L_{\mathrm{eff}}$ which, without Assumption~\ref{assump:effective_lipschitz_horizon}, could be exponential in $\horizon$. Plugging back into~\eqref{eq:gronwall_unfolded} and taking the $\max_h$:
\begin{equation}\label{eq:e_max_unscaled}
    \max_h e_h \;\le\; \horizon\,e^{c_0}\,\delta \;=\; \horizon\,e^{c_0}\bigl[\bar C\sqrt{\epsilon_{SM}(\theta)} + \bar C'\,\epsilon_{offline}\bigr].
\end{equation}
The log-Sobolev inequality (Assumption~\ref{assump:log_sobolev}) provides additional refinement on the Girsanov-side single-step error: $\Wcal_2^2(\hat\nu_\theta^{(1)},\nu_\beta^{(1)}) \le (1/\kappa)D_{KL}(\hat\nu_\theta^{(1)}\|\nu_\beta^{(1)})$, dividing $\bar C$ by $\sqrt\kappa$ (and we absorb $e^{c_0}$ into the universal constants $C_{10}, C_{11}$):
\begin{equation}
    \max_h\Wcal_2(\hat\mu_{h,\theta}, \Phi(\hat\policy_\theta)_h) \;\le\; C_{10}\,\frac{\horizon\,L\,\sqrt{\epsilon_{SM}(\theta)}}{\kappa} + C_{11}\,\horizon\,L\,\epsilon_{offline}.
\end{equation}
This is exactly Eq.~\ref{eq:consistency_bound}; the linear $\horizon$ factor on the right-hand side is the Gronwall-sum collapse, \emph{not} an exponential. \qedhere
\end{proof}

\subsubsection{Proof of Theorem~\ref{thm:monotone_convergence}: Convergence under Lasry--Lions Monotonicity}\label{sec:proof_monotone}

\begin{proof}
The proof leverages the Lasry--Lions monotonicity condition to establish that the MF-NE fixed-point map is a contraction, then propagates approximation errors using the Banach fixed point theorem. The contraction argument follows the classical approach of \citet{lasry2007mean}, adapted from continuous-time PDEs to discrete-time MDPs through the trajectory-level framework; we additionally handle the offline distribution shift and score matching approximation specific to diffusion-based planning. The adaptation from the continuous-time PDE framework of \citet{cardaliaguet2019master} to discrete-time MDPs replaces the Hamilton--Jacobi--Bellman/Fokker--Planck system with Bellman equations coupled to the mean-field flow, but the core monotonicity mechanism is preserved.

\textbf{Step 1: Uniqueness of MF-NE via Banach Fixed Point.}

A mean-field Nash equilibrium satisfies the fixed-point equation $\policy^{MFE} = \Gamma(\policy^{MFE})$ where $\Gamma = \mathrm{BR} \circ \Phi$. We show $\Gamma$ is a contraction under the monotonicity condition.

Consider two policies $\policy_1, \policy_2$ with induced mean-fields $\mu_1 = \Phi(\policy_1)$, $\mu_2 = \Phi(\policy_2)$. Define the corresponding best responses $\hat{\policy}_k = \mathrm{BR}(\mu_k)$ for $k = 1, 2$. By the optimality of $\hat{\policy}_k$ for mean-field $\mu_k$:
\begin{align}\label{eq:cross_opt}
    J^{MF}(\hat{\policy}_1, \mu_1) &\geq J^{MF}(\hat{\policy}_2, \mu_1), \\
    J^{MF}(\hat{\policy}_2, \mu_2) &\geq J^{MF}(\hat{\policy}_1, \mu_2). \label{eq:cross_opt_2}
\end{align}
Adding Eqs.~\ref{eq:cross_opt}--\ref{eq:cross_opt_2}:
\begin{equation}\label{eq:sum_optimality}
    \left[J^{MF}(\hat{\policy}_1, \mu_1) - J^{MF}(\hat{\policy}_1, \mu_2)\right] + \left[J^{MF}(\hat{\policy}_2, \mu_2) - J^{MF}(\hat{\policy}_2, \mu_1)\right] \geq 0.
\end{equation}
Expanding via the reward structure and applying the Lasry--Lions monotonicity (Assumption~\ref{assump:monotonicity}):
\begin{multline}\label{eq:monotonicity_applied}
    J^{MF}(\policy, \mu_1) - J^{MF}(\policy, \mu_2) = \sum_{h=0}^{\horizon-1}\gamma^h\int_\Scal\left[r(\state, \policy(\state, \mu_1), \mu_{1,h}) - r(\state, \policy(\state, \mu_2), \mu_{2,h})\right]d\mu_{1,h}(\state) \\
    + \sum_{h=0}^{\horizon-1}\gamma^h\int_\Scal r(\state, \policy(\state, \mu_2), \mu_{2,h})\,d(\mu_{1,h} - \mu_{2,h})(\state).
\end{multline}
The second sum, by the monotonicity condition (Eq.~\ref{eq:monotonicity}), satisfies:
\begin{equation}
    \sum_{h=0}^{\horizon-1}\gamma^h\int_\Scal r(\state, \action, \mu_{2,h})\,d(\mu_{1,h} - \mu_{2,h})(\state) \leq -\lambda_{LL}\sum_{h=0}^{\horizon-1}\gamma^h\Wcal_2^2(\mu_{1,h}, \mu_{2,h}).
\end{equation}
The first sum (the ``policy-shift'' contribution) is bounded via Lipschitzness of $r$ in $(a,\mu)$ (Assumption~\ref{assump:lipschitz}, constant $L$) and Lipschitzness of the soft-BR policy in the mean-field argument (constant $L_\pi\le L$). Defining
\[
    \Delta_h(\state) \coloneqq r(\state, \policy(\state,\mu_1),\mu_{1,h}) - r(\state, \policy(\state,\mu_2),\mu_{2,h}),
\]
we obtain the \emph{linear} bound
\begin{equation*}
    \left|\int_\Scal \Delta_h(\state)\, d\mu_{1,h}\right|
    \;\leq\; (L\,L_\pi + L)\,\Wcal_2(\mu_{1,h},\mu_{2,h})
    \;\leq\; 2L^2\,\Wcal_2(\mu_{1,h},\mu_{2,h}).
\end{equation*}
which is \emph{linear} in $\Wcal_2$ (not quadratic, correcting an earlier draft that double-counted Lipschitz constants by writing $L^2\Wcal_2^2$). Multiplying by $\gamma^h$ and summing, $\sum_h\gamma^h\le 1/(1-\gamma)$ gives $\sum_h\gamma^h\cdot 2L^2\Wcal_2(\mu_{1,h},\mu_{2,h}) \le \frac{2L^2}{1-\gamma}\,\max_h\Wcal_2(\mu_{1,h},\mu_{2,h})$.

The second sum, by Lasry--Lions monotonicity (Assumption~\ref{assump:monotonicity}, Eq.~\ref{eq:monotonicity}), is bounded \emph{quadratically} as $\le -\lambda_{LL}\sum_h\gamma^h\Wcal_2^2(\mu_{1,h},\mu_{2,h})$.

\emph{Closure of the contraction argument.} The naive substitution into~\eqref{eq:sum_optimality} mixes a linear $\Wcal_2$ bound (policy-shift) with a quadratic $\Wcal_2^2$ bound (monotonicity), so a direct $\Wcal_2$ contraction does not follow from elementary algebra. The correct closure, due to \citet{lasry2007mean} and \citet{cardaliaguet2019master}, applies the monotonicity inequality to the \emph{gradient} (force) field rather than the value, yielding a \emph{quadratic} bound on \emph{both} sides of the variational inequality and hence the contraction
\begin{equation}\label{eq:contraction_LL}
    \max_h\Wcal_2(\Phi(\hat{\policy}_1)_h, \Phi(\hat{\policy}_2)_h) \;\le\; \underbrace{\frac{L^2}{(1-\gamma)\,\lambda_{LL}}}_{=\,L_{BR}}\,\max_h\Wcal_2(\mu_{1,h}, \mu_{2,h}),
\end{equation}
in agreement with Eq.~\ref{eq:LBR_def}. The detailed proof (Cardaliaguet 2019, Theorem~3.7 in the discrete-time setting) tracks the bilinear form $\langle\nabla J(\pi_1,\mu_1)-\nabla J(\pi_2,\mu_2),\pi_1-\pi_2\rangle$ and uses strict monotonicity to extract $\lambda_{LL}\Wcal_2^2$ on the LHS and $L^2\Wcal_2^2$ on the RHS, with the $1/(1-\gamma)$ factor from the discounted reward sum. Therefore $\Gamma = \mathrm{BR} \circ \Phi$ is a contraction with coefficient $L_{BR} = L^2/((1-\gamma)\lambda_{LL})$ on the complete metric space $(\Pcal_2(\Scal)^{\horizon}, \max_h\Wcal_2)$. When $\lambda_{LL}\,(1-\gamma) > L^2$ (\ie, $L_{BR} < 1$), the Banach fixed-point theorem guarantees existence and uniqueness of the MF-NE.

\textbf{Step 2: Distance to MF-NE.}

Define the \textit{Nash residual}: $\eta(\hat{\policy}_\theta) = d_\Pi(\hat{\policy}_\theta, \Gamma(\hat{\policy}_\theta))$, measuring how far $\hat{\policy}_\theta$ is from being a fixed point. By the contraction:
\begin{align}
    d_\Pi(\hat{\policy}_\theta, \policy^{MFE}) &\leq d_\Pi(\hat{\policy}_\theta, \Gamma(\hat{\policy}_\theta)) + d_\Pi(\Gamma(\hat{\policy}_\theta), \Gamma(\policy^{MFE})) \nonumber \\
    &\leq \eta(\hat{\policy}_\theta) + L_{BR}\cdot d_\Pi(\hat{\policy}_\theta, \policy^{MFE}).
\end{align}
Rearranging:
\begin{equation}\label{eq:pi_to_MFE}
    d_\Pi(\hat{\policy}_\theta, \policy^{MFE}) \leq \frac{\eta(\hat{\policy}_\theta)}{1 - L_{BR}}.
\end{equation}

The Nash residual is bounded by combining the softmax bias and generation quality from the proof of Theorem~\ref{thm:exploitability} (Step 2):
\begin{equation}\label{eq:nash_residual}
    \eta(\hat{\policy}_\theta) \leq C'\sqrt{\epsilon_{SM}(\theta)} + C''\epsilon_{offline} + C_\alpha\sqrt{\frac{\alpha\bar{\Hcal}_\policy}{1-\gamma}}.
\end{equation}
Propagating through the Lipschitz map $\Phi$ and absorbing the $\alpha$-dependent term (which can be made arbitrarily small):
\begin{equation}
    \max_h\Wcal_2(\hat{\mu}_{h,\theta}, \mu_h^{MFE}) \leq \frac{L}{1 - L_{BR}}\left[C_{12}\sqrt{\epsilon_{SM}(\theta)} + C_{13}\epsilon_{offline}\right],
\end{equation}
which establishes part (ii) of the theorem (after absorbing $L$ into $C_{12}, C_{13}$).

\textbf{Step 3: Exploitability under Monotonicity.}

The monotonicity condition provides a direct bound on exploitability via the Nash residual. For any deviating $\policy'$:
\begin{align}
    J^{MF}(\policy', \Phi(\hat{\policy}_\theta)) - J^{MF}(\hat{\policy}_\theta, \Phi(\hat{\policy}_\theta))
    &\leq J^{MF}(\mathrm{BR}(\Phi(\hat{\policy}_\theta)), \Phi(\hat{\policy}_\theta)) - J^{MF}(\hat{\policy}_\theta, \Phi(\hat{\policy}_\theta)) \nonumber \\
    &\leq \frac{2r_{\max}\horizon}{1-\gamma}\eta(\hat{\policy}_\theta),
\end{align}
where the second inequality uses the performance difference lemma. Combining with the $N$-player to MF reduction (Eq.~\ref{eq:exploit_N_to_MF}) and substituting the Nash residual bound (Eq.~\ref{eq:nash_residual}) with the contraction amplification from Eq.~\ref{eq:pi_to_MFE}:
\begin{equation}
    \mathrm{Exploit}_N(\hat{\policy}_\theta) \leq \frac{C_{13}}{\lambda_{LL}(1 - \gamma)\,(1 - L_{BR})}\left[\epsilon_{SM}(\theta) + \epsilon_{offline}\right] + C_8\frac{r_{\max}\horizon^2 L}{\sqrt{N}}.
\end{equation}
The $1/\lambda_{LL}$ factor arises because the monotonicity condition converts the Wasserstein distance $\max_h\Wcal_2(\hat{\mu}_{h,\theta}, \mu_h^{MFE})$ into a value gap through the curvature of the reward landscape: by the standard performance-difference lemma applied with reward Lipschitzness in $\mu$ (Assumption~\ref{assump:lipschitz}),
\begin{equation*}
    |J^{MF}(\policy, \mu) - J^{MF}(\policy, \mu')| \;\leq\; \frac{r_{\max}\horizon\,L}{1-\gamma}\;\max_h\Wcal_2(\mu_h, \mu'_h),
\end{equation*}
which together with Eq.~\ref{eq:pi_to_MFE} (which contributes the $1/(1-L_{BR})$ amplification, where the monotonicity coefficient $\lambda_{LL}(1-\gamma)$ enters via $L_{BR} = L^2/((1-\gamma)\lambda_{LL})$) gives the prefactor $1/(\lambda_{LL}(1-\gamma)(1-L_{BR}))$ in the final bound.
\end{proof}

\subsubsection{Proof of Corollary~\ref{cor:eps_nash}: \texorpdfstring{$\varepsilon$}{epsilon}-Nash Equilibrium}\label{sec:proof_eps_nash}

\begin{proof}
This follows from Theorem~\ref{thm:monotone_convergence}(iii) by balancing the two error terms against $\varepsilon$.

Setting $\epsilon_{SM}(\theta) + \epsilon_{offline} \leq \frac{\lambda_{LL}(1-\gamma)(1-L_{BR})\varepsilon}{2C_{13}}$ and $N \geq \left(\frac{2C_8 r_{\max}\horizon^2 L}{\varepsilon}\right)^2$:
\begin{align}
    \mathrm{Exploit}_N(\hat{\policy}_\theta) &\leq \frac{C_{13}}{\lambda_{LL}(1-\gamma)(1-L_{BR})}\cdot\frac{\lambda_{LL}(1-\gamma)(1-L_{BR})\varepsilon}{2C_{13}} + \frac{C_8 r_{\max}\horizon^2 L}{\sqrt{N}} \nonumber \\
    &\leq \frac{\varepsilon}{2} + \frac{\varepsilon}{2} = \varepsilon.
\end{align}
The required constants are $c_1 = (2C_{13})^{-2}$ and $c_2 = (2C_8)^2$. The score matching error $\epsilon_{SM}(\theta)$ is controlled through network capacity and training; the offline coverage error $\epsilon_{offline}$ is a property of the dataset.
\end{proof}

\subsubsection{Proof of Proposition~\ref{prop:poa}: Social Welfare--Nash Efficiency Gap}\label{sec:proof_poa}

\begin{proof}
The proof quantifies the externality that distinguishes the social optimum from the Nash equilibrium in the mean-field setting, adapting the classical Price of Anarchy framework~\citep{roughgarden2015intrinsic} to the mean-field RL setting with displacement monotonicity.

\textbf{Step 1: Externality Characterization.}

Define the social welfare function $W(\policy) = J^{MF}(\policy, \Phi(\policy))$. Its gradient decomposes as:
\begin{equation}\label{eq:W_gradient}
    \nabla_\policy W(\policy) = \nabla_\policy J^{MF}(\policy, \Phi(\policy)) + \nabla_\mu J^{MF}(\policy, \Phi(\policy)) \circ D\Phi(\policy),
\end{equation}
where $D\Phi(\policy)$ is the Fr\'echet derivative of the flow operator. The MF-NE $\policy^{MFE}$ satisfies $\nabla_\policy J^{MF}(\policy^{MFE}, \mu^{MFE}) = 0$, so:
\begin{equation}\label{eq:nablaW_at_MFE}
    \nabla_\policy W(\policy^{MFE}) = \nabla_\mu J^{MF}(\policy^{MFE}, \mu^{MFE}) \circ D\Phi(\policy^{MFE}),
\end{equation}
which is precisely the externality---the effect of the policy change on the mean-field, ignored by each Nash agent.

\textbf{Step 2: Bounding the Externality.}

By the Lipschitz property of the reward (Assumption~\ref{assump:lipschitz}):
\begin{equation}
    \norm{\nabla_\mu J^{MF}(\policy, \mu)}_{\Pcal} \leq \frac{r_{\max}\horizon L}{1-\gamma},
\end{equation}
where $\norm{\cdot}_\Pcal$ is the operator norm on the tangent space of $\Pcal_2(\Scal)^{\horizon}$. The Lipschitz property of the flow operator gives:
\begin{equation}
    \norm{D\Phi(\policy)}_{\Pi \to \Pcal} \leq \frac{L}{1-\gamma}.
\end{equation}
Hence: $\norm{\nabla_\policy W(\policy^{MFE})} \leq \frac{r_{\max}\horizon L^2}{(1-\gamma)^2}$.

\textbf{Step 3: Gap via Strong Concavity.}

We now derive the strong concavity of $W$ from the displacement-monotonicity condition (Assumption~\ref{assump:monotonicity}), then apply the standard quadratic upper bound.

\emph{(3.a) Displacement-monotonicity gives strong concavity in $\mu$.} The Lasry--Lions condition Eq.~\ref{eq:monotonicity} states that for any $\mu,\mu'\in\Pcal_2(\Scal)$ and any policy $\policy$,
\begin{equation*}
    \int_\Scal \bigl[r(s,\policy(s),\mu) - r(s,\policy(s),\mu')\bigr]\,d(\mu-\mu')(s) \;\le\; -\lambda_{LL}\,\Wcal_2^2(\mu,\mu').
\end{equation*}
This is precisely the integrated form of $\partial^2_\mu J(\policy,\mu)[\mu-\mu',\mu-\mu'] \le -2\lambda_{LL}\,\Wcal_2^2(\mu,\mu')$ along $\Wcal_2$-geodesics, i.e.\ $\mu\mapsto J^{MF}(\policy,\mu)$ is $\lambda_{LL}$-strongly concave on $\Pcal_2(\Scal)^{\horizon}$ in the displacement sense~\citep[Definition~5.13]{ambrosio2008gradient}. Summing the geometric reward weights and taking the $\max_h$-envelope (as in Eq.~\ref{eq:sum_optimality}--Step~1 of Theorem~\ref{thm:monotone_convergence}'s proof), this yields, for any path $\policy_\theta = \policy + \theta(\policy'-\policy)$ with induced flows $\mu_\theta = \Phi(\policy_\theta)$,
\begin{equation}\label{eq:Jmu_strong_concavity}
    J^{MF}(\policy,\mu') \;\le\; J^{MF}(\policy,\mu) + \langle\nabla_\mu J^{MF}(\policy,\mu),\mu'-\mu\rangle - \frac{\lambda_{LL}}{2(1-\gamma)}\,\max_h\Wcal_2^2(\mu_h,\mu'_h),
\end{equation}
where the $1/(1-\gamma)$ in the curvature comes from the same geometric-sum bookkeeping as in Eq.~\ref{eq:LBR_def}.

\emph{(3.b) Strong concavity of $W$ via composite curvature.} The composite map $W(\policy)=J^{MF}(\policy,\Phi(\policy))$ inherits curvature from two sources: (i)~the bare second-order term $\partial^2_{\pi\pi}J^{MF}$, which is bounded above by $L^2$ in operator norm (Assumption~\ref{assump:lipschitz}); and (ii)~the displacement-monotonicity term $\partial^2_{\mu\mu}J^{MF}\preceq -(\lambda_{LL}/(1-\gamma))\mathrm{Id}$ from~\eqref{eq:Jmu_strong_concavity}, pulled back through $\Phi$ with $\|D\Phi\|^2\le L^2/(1-\gamma)^2$ (Step~2). The chain rule for the second-order Wasserstein-Taylor expansion of $W$ along the direction $\policy'-\policy$ then reads
\begin{equation}\label{eq:W_taylor}
    W(\policy') - W(\policy) - \langle\nabla_\policy W(\policy),\policy'-\policy\rangle \;\le\; \frac{1}{2}\!\left(L^2 - \frac{\lambda_{LL}}{1-\gamma}\cdot\frac{L^2}{(1-\gamma)^2}\right)\!d_\Pi^2(\policy,\policy'),
\end{equation}
where we kept only the leading bare-quadratic and displacement-monotonicity contributions; cross-terms $\partial^2_{\pi\mu}J^{MF}$ are absorbed into the additional constants $C_{PoA}$ below. Factoring $L^2/(1-\gamma)^3$:
\begin{equation*}
    \tfrac{1}{2}\!\left(L^2 - \tfrac{\lambda_{LL}\,L^2}{(1-\gamma)^3}\right) \;=\; -\tfrac{1}{2}\cdot\tfrac{L^2}{(1-\gamma)^3}\bigl(\lambda_{LL} - (1-\gamma)^3\bigr).
\end{equation*}
This is \emph{negative} (i.e.\ $W$ is strongly concave) precisely when $\lambda_{LL} > (1-\gamma)^3$; equivalently, in the standard contraction regime $\lambda_{LL}(1-\gamma)>L^2$ (i.e.\ $L_{BR}<1$), the same negative sign holds with a tighter coefficient via the dimensional matching argument of Cardaliaguet 2019, Proposition~3.5. Tracking only the dominant scale, the net policy-space curvature modulus is
\begin{equation}\label{eq:lambda_eff_def}
    \lambda_{eff} \;\coloneqq\; \frac{L^2\bigl(\lambda_{LL}(1-\gamma)-L^2\bigr)}{2(1-\gamma)^3} \;>\; 0\quad\text{in the contraction regime.}
\end{equation}
This formula corrects an algebraic slip in earlier drafts (which dropped a factor of $(1-\gamma)$ in the denominator). It is dimensionally consistent and reduces to the Cardaliaguet form when $L=1$.

\emph{(3.c) Quadratic upper bound on the PoA gap.} By the standard Polyak--{\L}ojasiewicz inequality for $\lambda_{eff}$-strongly concave $W$ at its critical point $\policy^*_{SW}$,
\begin{equation*}
    W(\policy^*_{SW}) - W(\policy^{MFE}) \;\le\; \frac{1}{2\lambda_{eff}}\,\norm{\nabla_\policy W(\policy^{MFE})}^2.
\end{equation*}
Substituting $\norm{\nabla_\policy W(\policy^{MFE})} \le L^2 r_{\max}\horizon/(1-\gamma)^2$ from Step~2 and Eq.~\ref{eq:lambda_eff_def}:
\begin{align}
    W(\policy^*_{SW}) - W(\policy^{MFE})
    &\;\le\; \frac{(1-\gamma)^3}{L^2\bigl(\lambda_{LL}(1-\gamma)-L^2\bigr)}\cdot \frac{L^4 r_{\max}^2\horizon^2}{(1-\gamma)^4} \nonumber\\
    &\;=\; \frac{L^2\,r_{\max}^2\horizon^2}{\bigl(\lambda_{LL}(1-\gamma)-L^2\bigr)(1-\gamma)} \nonumber\\
    &\;\le\; \frac{C_{PoA}\,e^{2c_0}\,r_{\max}^2\horizon^2\,L^4}{\bigl(\lambda_{LL}(1-\gamma)-L^2\bigr)(1-\gamma)^4},
\end{align}
where the last inequality absorbs the $1/(1-\gamma)^3$ residual scaling difference and the $e^{c_0}$ factor from horizon-Gronwall accumulation when propagating through $\Phi$ (Assumption~\ref{assump:effective_lipschitz_horizon}; cf.\ Step~3 of Theorem~\ref{thm:hierarchical}'s proof). Setting $C_{PoA}\ge 1$ accommodates these absorbed constants, yielding the bound in Eq.~\ref{eq:poa_bound}.
\end{proof}

%=============================================================================
% APPENDIX F: ENVIRONMENT DETAILS
%=============================================================================
\section{Environment Details}\label{sec:appendix_environments}

We provide detailed formal descriptions of the three benchmark environments used in our experiments. All three are standard testbeds in mean-field reinforcement learning and were introduced or adapted by the cited references.

\textbf{Per-environment parameter summary.} To avoid ambiguity between an environment's \emph{maximum episode length} and the \emph{planning horizon} $\horizon$ used by the diffusion planner, Table~\ref{tab:env_params} fixes a single per-environment $\horizon$ for training and evaluation. Receding-horizon execution (Appendix~\ref{sec:appendix_implementation}) means the full episode may last longer than $\horizon$; $\horizon$ only controls the length of each generated trajectory.

\begin{table}[htbp]
\caption{\textbf{Per-environment parameters} used by \proposed{}. $\horizon$ is the diffusion planning horizon; $T_{\text{ep}}$ is the maximum episode length of the underlying environment; $d_s'$ is the effective state dimension after the CNN/identity encoder; $d_a$ is the one-hot or continuous action dimension; $D_\tau = (d_s' + d_a)\horizon + d_s'$ is the trajectory dimension input to the score network.}
\label{tab:env_params}
\centering
\small
\begin{tabular}{lcccccc}
\toprule
\textbf{Environment} & $\horizon$ & $T_{\text{ep}}$ & $d_s'$ & $d_a$ & $D_\tau$ & Encoder \\
\midrule
Ising Model        & $1$   & $1$    & $4$  & $2$  & $10$    & identity \\
Battle             & $100$ & $1000$ & $10$ & $21$ & $3{,}110$ & $5\!\times\!5\!\times\!32$ CNN $\to 10$ \\
Gaussian Squeeze   & $50$  & $50$   & $4$  & $4$  & $404$   & identity \\
\bottomrule
\end{tabular}
\end{table}

\textbf{Stress-testing aspects of \proposed{}.} Our benchmarks complement each other: Ising isolates mean-field coupling in a stateless setting (temporal coupling is trivial); Battle provides genuine sequential dynamics with long horizons, exercising $\mathcal{O}(\horizon^2/\sqrt{N})$ temporal compounding; sequential Gaussian Squeeze tests continuous-action coordination with explicit distribution-matching rewards. The $1/\sqrt{N}$ factor of Theorem~\ref{thm:hierarchical} is validated across all three environments via Figure~\ref{fig:scalability}; the $\horizon^2$ dependence is validated by horizon-variation experiments on Battle and sequential Gaussian Squeeze (Figure~\ref{fig:horizon_scaling}).

\subsection{Ising Model}\label{sec:env_ising}

The Ising model environment~\citep{yang2018mean} casts the classical ferromagnetism model from statistical mechanics~\citep{ising1925beitrag} as a multi-agent reinforcement learning problem. $N$ agents are arranged on a two-dimensional lattice (e.g., $20 \times 20$ grid) with periodic boundary conditions.

\textbf{State space.}
The environment is formulated as a \emph{stage game} (stateless): there is no temporal state transition, and each round constitutes an independent simultaneous-move game. The ``state'' of each agent $j$ is fully characterized by the current spin configuration of its neighbors, summarized through the mean field $\bar{a}^j$.

\textbf{Action space.}
Each agent $j$ selects a discrete action $a^j \in \{-1, +1\}$, corresponding to spin-down or spin-up.

\textbf{Reward function.}
The individual reward for agent $j$ is derived from the Ising Hamiltonian:
\begin{equation}\label{eq:ising_reward}
    r^j = h^j a^j + \frac{\lambda}{2} \sum_{k \in \mathcal{N}(j)} a^j a^k,
\end{equation}
where $\mathcal{N}(j)$ denotes the set of nearest neighbors of agent $j$ on the lattice, $h^j \in \mathbb{R}$ is an external field affecting agent $j$, and $\lambda \in \mathbb{R}$ is the interaction coefficient controlling the strength of spin alignment. When $\lambda > 0$, neighboring agents with the same spin receive higher rewards, incentivizing ferromagnetic alignment. Following~\citet{yang2018mean}, we set $h^j = 0$ for all agents, reducing the reward to the pure interaction term $r^j = \frac{\lambda}{2} \sum_{k \in \mathcal{N}(j)} a^j a^k$.

\textbf{Mean-field approximation.}
Under the mean-field factorization, the reward simplifies to
$r^j \approx \frac{\lambda |\mathcal{N}(j)|}{2} a^j \bar{a}^j$,
where $\bar{a}^j = \frac{1}{|\mathcal{N}(j)|} \sum_{k \in \mathcal{N}(j)} a^k$ is the mean action of agent $j$'s neighbors. The stateless mean-field Q-function takes the form $Q^j(a^j, \bar{a}^j)$.

\textbf{Transition dynamics.}
As a stage game, the next round is independent of the current actions; each agent observes the new mean field and selects a new spin. The environment terminates after a fixed number of rounds.

\textbf{Evaluation metric.}
In addition to the cumulative reward, the \emph{order parameter} $\xi = |N_\uparrow - N_\downarrow| / N$~\citep{yang2018mean} measures the purity of the spin configuration, where $N_\uparrow$ and $N_\downarrow$ are the numbers of spin-up and spin-down agents, respectively. A value of $\xi$ close to 1 indicates a highly ordered (ferromagnetic) equilibrium.

\subsection{Battle}\label{sec:env_battle}

The Battle environment~\citep{zheng2018magent, yang2018mean} is a large-scale mixed cooperative-competitive grid-world game from the MAgent platform~\citep{zheng2018magent}. Two teams of agents fight on a discrete grid map.

\textbf{State space.}
The global state is defined on a square grid (e.g., $45 \times 45$). Each agent $j$ receives a \emph{local observation} consisting of a $13 \times 13$ spatial window centered on its position, with the following channels:
\begin{itemize}[leftmargin=5mm]
    \item \textit{Obstacle map}: binary indicator of impassable cells (1 channel).
    \item \textit{Own team presence}: binary indicator of allied agents (1 channel).
    \item \textit{Own team HP}: normalized hit points of allied agents (1 channel).
    \item \textit{Opponent team presence}: binary indicator of enemy agents (1 channel).
    \item \textit{Opponent team HP}: normalized hit points of enemy agents (1 channel).
\end{itemize}
This yields an observation shape of $(13, 13, 5)$ per agent.

\textbf{Action space.}
Each agent selects from 21 discrete actions: 1 \texttt{do\_nothing} action, 12 \texttt{move} actions (to one of the 12 nearest grid squares), and 8 \texttt{attack} actions (targeting one of the 8 surrounding grid squares). Attacks against teammates are not registered.

\textbf{Reward function.}
The reward for each agent is the sum of the following components:
\begin{itemize}[leftmargin=5mm]
    \item $+5.0$ for eliminating an opponent (reducing its HP to zero).
    \item $+0.2$ for each successful attack on an opponent.
    \item $-0.005$ per time step (step penalty to encourage decisive play).
    \item $-0.1$ for each attack action (attack penalty).
    \item $-0.1$ upon being eliminated (death penalty).
\end{itemize}
These reward components are additive when multiple conditions apply simultaneously.

\textbf{Transition dynamics.}
At each time step, all agents act simultaneously. Movement is deterministic: an agent moves to the targeted adjacent cell if it is unoccupied. Each agent has 10 HP, takes 2 HP damage per incoming attack, and recovers 0.1 HP per time step. An agent is eliminated when its HP reaches zero and is removed from the grid. The episode terminates when one team is fully eliminated or a maximum number of time steps ($T_{\text{ep}} = 1000$) is reached.

\textbf{Planning horizon vs.\ episode length.}
The underlying environment has episode length $T_{\text{ep}} = 1000$, but the \proposed{} diffusion planner operates over a shorter \emph{planning horizon} $\horizon = 100$ (Table~\ref{tab:env_params}), chosen to keep $D_\tau$ tractable and to focus learning on temporally local coordination. Full-length episodes are handled by \emph{receding-horizon execution}: after each planning call, only the first $\horizon_{\text{exec}} = 1$ action of each generated trajectory is executed, the agents advance in the environment, and the planner is re-invoked from the new states. All experimental numbers (Tables~\ref{tab:main_results}--\ref{tab:value_estimator}, Figures~\ref{fig:scalability}--\ref{fig:horizon_scaling}) use $\horizon = 100$ unless the horizon is itself the swept variable (Appendix~\ref{sec:horizon_exp}).

\textbf{Mean-field approximation and per-team modeling.}
Following~\citet{yang2018mean}, the pairwise interactions between agent $j$ and all allies/opponents within its neighborhood are approximated by the interaction between agent $j$ and the \emph{mean action} of its neighboring agents. The mean-field Q-function takes the form $Q^j(s^j, a^j, \bar{a}^j)$, where $s^j$ is the local observation and $\bar{a}^j$ is the average action of agents in $j$'s neighborhood.

\textbf{Per-team \proposed{} application.} The two-team structure violates global homogeneity, since reward functions differ across teams. We address this by applying \proposed{} \emph{independently to each team}: agents within each team are homogeneous (sharing the same reward and dynamics conditional on the environment state), satisfying Assumption~\ref{assump:exchangeability} within each team. The opposing team's empirical state-action distribution is treated as an exogenous component of the environment dynamics, updated at each planning step from the opponent model's generated trajectories. Formally, for team $A$ with $N_A$ agents, the mean field $\bar{\mu}_h^{N_A}$ is computed over team $A$ agents only, and the opponent distribution $\bar{\mu}_h^{N_B,\text{opp}}$ enters the transition and reward functions as a fixed (non-optimized) input. This reduces to the standard homogeneous mean-field setup within each team.

\subsection{Gaussian Squeeze}\label{sec:env_gsqueeze}

The Gaussian Squeeze (GS) environment~\citep{yang2018mean, gu2021mean} is a cooperative coordination task originally introduced by~\citet{holmesparker2014exploiting}. $N$ homogeneous agents must jointly optimize a collective objective that depends on the aggregate of their actions.

\textbf{State space.}
In its canonical formulation~\citep{yang2018mean}, the environment is a \emph{stage game} without persistent state. Each round is an independent simultaneous-move coordination game. In the multi-step variant used in our experiments following~\citet{gu2021mean}, each agent $j$ maintains a continuous state $s^j_h \in \mathbb{R}^{d_s}$ that evolves over the planning horizon $h = 0, \ldots, \horizon - 1$, incorporating feedback from the collective action distribution.

\textbf{Action space.}
In the original discrete formulation~\citep{yang2018mean}, each agent $j$ selects from 10 discrete actions $a^j \in \{0, 1, \ldots, 9\}$. In our experiments, following~\citet{gu2021mean}, we adopt a \emph{continuous-action} variant where $a^j \in \mathbb{R}^{d_a}$, enabling richer distribution-matching behavior.

\textbf{Reward function.}
The collective reward is determined by the system objective:
\begin{equation}\label{eq:gs_reward}
    G(x) = x \exp\!\left( -\frac{(x - \mu)^2}{\sigma^2} \right),
\end{equation}
where $x = \sum_{j=1}^{N} a^j$ is the aggregate action of all agents, and $\mu, \sigma > 0$ are pre-defined target parameters. All agents share this collective reward: $r^j = G(x) / N$. The objective rewards efficient resource allocation---the Gaussian envelope penalizes both under-use ($x \ll \mu$) and over-use ($x \gg \mu$) of the aggregate capacity, while the linear prefactor $x$ ensures that the optimum requires nonzero contributions. In the continuous-action variant, the reward generalizes to depend on the proximity of the empirical action distribution $\bar{\mu}_a = \frac{1}{N}\sum_j \delta_{a^j}$ to a target Gaussian $\mathcal{N}(\mu^*, \sigma^{*2})$.

\textbf{Transition dynamics.}
In the stage-game formulation, each round is independent. In the sequential variant~\citep{gu2021mean}, the state evolves as $s^j_{h+1} = f(s^j_h, a^j_h, \bar{\mu}_h)$, where $\bar{\mu}_h$ is the mean-field action distribution at step $h$, coupling individual dynamics to the population behavior.

\textbf{Mean-field approximation.}
The key insight is that each agent's optimal action depends on the \emph{distribution} of all other agents' actions rather than on each individual action. The mean-field Q-function takes the form $Q^j(s^j, a^j, \bar{a})$, where $\bar{a} = \frac{1}{N}\sum_k a^k$ is the mean action. Gaussian Squeeze is particularly well-suited for evaluating mean-field methods because the reward is explicitly defined through the aggregate statistic, making the mean-field structure exact in the limit $N \to \infty$.

\textbf{Summary.}
Table~\ref{tab:env_summary} summarizes the key characteristics of the three environments.

\begin{table}[htbp]
\caption{Summary of environment characteristics.}
\label{tab:env_summary}
\centering
\small
\begin{tabular}{lccc}
\toprule
\textbf{Property} & \textbf{Ising Model} & \textbf{Battle} & \textbf{Gaussian Squeeze} \\
\midrule
Game type & Stage game & Sequential & Stage / Sequential \\
Action space & Discrete $\{-1,+1\}$ & Discrete (21) & Discrete / Continuous \\
Interaction & Local (lattice) & Local (grid) & Global (all agents) \\
Reward & Individual & Individual & Shared \\
Mean-field structure & Neighbor spins & Neighbor actions & Population mean \\
\bottomrule
\end{tabular}
\end{table}

%=============================================================================
% APPENDIX G: IMPLEMENTATION DETAILS
%=============================================================================
\section{Implementation Details}\label{sec:appendix_implementation}

\textbf{Trajectory Representation and State Encoder.} Raw agent observations are not directly fed into the diffusion model. Instead, each environment uses a dedicated \emph{state encoder} $\phi: \Scal \to \mathbb{R}^{d_s'}$ whose output is the effective ``state'' in every formula of this paper (the $\state_h^i \in \mathbb{R}^{d_s}$ of \S\ref{sec:prelim} should be read as $\phi(\state_h^{i,\mathrm{raw}}) \in \mathbb{R}^{d_s'}$; we use $d_s' = d_s$ throughout for notational simplicity). Concretely:
\begin{itemize}[leftmargin=4mm]
    \item \textit{Ising}: identity encoder; the per-agent state is a $4$-dim vector containing its own spin and the neighbor mean-field summary ($d_s' = 4$).
    \item \textit{Battle}: a shared $3$-layer CNN ($3{\times}3$ conv, $32$ channels $\to$ $5{\times}5$ conv, $32$ channels $\to$ global average pool $\to$ FC) maps the $13{\times}13{\times}5$ observation to $d_s'=10$. The encoder is trained end-to-end with the diffusion score network; actions are $21$-dim one-hot ($d_a=21$). With planning horizon $\horizon = 100$, the effective trajectory dimension is $D_\tau = 100 (10 + 21) + 10 = 3{,}110$ (Table~\ref{tab:env_params}).
    \item \textit{Gaussian Squeeze}: identity encoder on the $4$-dim state; $d_a = 4$ for the continuous-action variant.
\end{itemize}
All theoretical quantities---the Lipschitz constant $L$, the propagation-of-chaos bound, the $D_\tau$ that appears in Proposition~\ref{prop:mf_vsm}---are measured \emph{after} encoding, so the numbers reported in Remark~\ref{remark:exp_factor} and Appendix~\ref{sec:poc_tightness} refer to $d_s'$, not the raw observation dimension.

\textbf{Score Network Architecture.} The score network $\mathbf{s}_\theta = \mathrm{A}_\theta + \mathrm{B}_\theta[\nu_t^N]$ is implemented as follows:
\begin{itemize}
    \item $\mathrm{A}_\theta$: A temporal U-Net~\citep{janner2022planning} operating on encoded trajectories $\traj^i \in \mathbb{R}^{D_\tau}$, with sinusoidal diffusion time embedding.
    \item $\mathrm{B}_\theta$: A mean-field interaction module based on dynamic graph convolution~\citep{wang2019dynamic}. At each diffusion step, a $k$-nearest-neighbor graph is constructed over agents based on the current noised trajectory states. Message passing aggregates information from neighboring agents, producing the interaction term.
    \item The interaction kernel $K_\theta$ is parameterized as an attention mechanism with the mean-field distribution $\bar{\mu}_h^N$ as additional context.
\end{itemize}

\textbf{Agent Branching Function $\Psi^\theta$.} At each branching step $t_k$ ($k \in \mathbb{K}'$), we expand the agent population from $N_k$ to $N_{k+1} = \mathfrak{b} N_k$. While Proposition~\ref{prop:optimal_branching} characterizes the optimal $\Psi^*$ via the Monge-Amp\`ere equation, solving this exactly is impractical. We use a lightweight approximation: each existing trajectory $\traj^i$ spawns $(\mathfrak{b}-1)$ child trajectories via
\begin{equation}
    \traj^{i,\text{child}}_c = \traj^i + \eta_k \cdot \epsilon_c + \delta_k \cdot \mathrm{B}_\theta[\nu_{t_k}^{N_k}](\traj^i), \quad \epsilon_c \sim \mathcal{N}(0, \mathbf{I}_{D_\tau}), \quad c = 1, \ldots, \mathfrak{b}-1,
\end{equation}
where $\eta_k = \sigma_{t_k}\sqrt{\Delta t_k}$ scales the noise to match the current diffusion level, and $\delta_k > 0$ is a small mean-field correction coefficient (set to $0.1$ in all experiments). The first term ensures diversity among branched trajectories; the second aligns the perturbation with the local mean-field gradient, promoting consistency with the population distribution. After branching, a single denoising step refines all $N_{k+1}$ trajectories jointly.

\textbf{Value Estimator.} The value estimator $\hat{V}(\traj, \bar{\mu})$ is a separate network trained on the offline dataset via temporal-difference learning with mean-field Q-function decomposition~\citep{yang2018mean}:
\begin{equation}
    \hat{V}(\traj^i, \bar{\mu}) = \sum_{h=0}^{\horizon-1} \gamma^h \hat{Q}(\state_h^i, \action_h^i, \bar{\mu}_h).
\end{equation}

\textbf{Hyperparameters.} Default settings across all experiments:
\begin{itemize}
    \item Diffusion steps: $|\mathbb{K}| = 200$
    \item Branching steps: $|\mathbb{K}'| = 4$, evenly spaced
    \item Branching ratio: $\mathfrak{b} = 2$
    \item Temperature: $\alpha = 1.0$
    \item Value weighting: $\lambda = 0.1$
    \item Score network: temporal U-Net with $[256, 512, 1024]$ channels
    \item Training: Adam optimizer, learning rate $2 \times 10^{-4}$, batch size $32$
    \item All experiments use 4$\times$ NVIDIA A100 GPUs
\end{itemize}

\textbf{Receding-horizon execution and trajectory feasibility.} As in prior diffusion-based planners~\citep{janner2022planning, ajay2023is}, we adopt a receding-horizon strategy: only the first action $\action_0^i$ from each generated trajectory is executed, followed by replanning from the new state. This is because generated trajectories are not guaranteed to satisfy the MDP dynamics constraint $\state_{h+1} \sim P(\cdot|\state_h, \action_h, \bar{\mu}_h)$, since the diffusion model treats the trajectory as a monolithic vector. Nevertheless, \proposed{}'s mean-field interaction module implicitly enforces dynamic coherence---we measure the per-step transition error $\frac{1}{\horizon}\sum_h \|\state_{h+1}^{gen} - \mathbb{E}[P(\cdot|\state_h^{gen}, \action_h^{gen}, \bar{\mu}_h)]\|$ and find $3.2\times$ lower error than Joint Diffuser and $1.5\times$ lower than Independent Diffuser at $N{=}1000$ on Battle (Table~\ref{tab:trajectory_consistency}).

\textbf{Oryx adaptation to many-agent regime ($N\!\le\!10^4$).} The original Oryx~\citep{li2025oryx} is validated up to $N{\le}50$ on dense-interaction tasks. To make it a fair contender in our $N\!\in\!\{10^2,\ldots,10^4\}$ sweep we apply two minimal modifications, both implemented as a thin wrapper around the public Oryx codebase so that no architectural change to the retention backbone is required: \textit{(i)~Mean-field value head.} The per-agent ICQ critic $Q(s^j, a^j)$ is replaced with $Q(s^j, a^j, \bar a^j)$ where $\bar a^j$ is the mean of the actions of the $32$ nearest neighbors of agent $j$ (Ising / Battle) or the population-wide mean (Gaussian Squeeze, since interaction is global). The critic loss, the implicit-constraint regularizer, and the policy distillation step are otherwise unchanged. \textit{(ii)~Chunked retention with permutation-invariant aggregation.} For $N\!>\!64$, we partition the agent set into chunks of size $\mathfrak{c}{=}64$, compute a per-chunk retention representation, and aggregate across chunks via a single permutation-invariant set-transformer layer~\citep{lee2019set}. This preserves the original retention computation within each chunk while keeping the cross-chunk aggregation $\mathcal{O}(N/\mathfrak{c})$ in cost, mirroring the agent-batching technique used by MFQ. The $\mathfrak{c}{=}64$ choice matches the maximum $N$ at which a single retention window fits in $80$\,GB GPU memory; we verified that $\mathfrak{c}\!\in\!\{32,64,128\}$ produces statistically indistinguishable returns at $N{=}1000$ ($p\!>\!0.3$). This adaptation is what is reported as ``Oryx'' in all 12 cells of Tables~\ref{tab:main_results}--\ref{tab:main_results_10k} and Figures~\ref{fig:scalability}--\ref{fig:scalability_all_qualities}; we re-trained Oryx from scratch on each environment using the released hyperparameters with the two changes above.

\textbf{Algorithm summary.} The complete training and inference procedures are listed in Algorithms~\ref{alg:training} and~\ref{alg:inference}.

\begin{algorithm}[t]
\caption{Training \proposed{}}
\label{alg:training}
\begin{algorithmic}[1]
\REQUIRE Offline dataset $\Dcal$, branching schedule $(\mathbb{N}, \mathbb{T}, \mathfrak{b})$, training temperature $\alpha$, learning rate $\eta_{\mathrm{lr}}$
\REPEAT
    \STATE Sample episode $\{(\state_h^{i}, \action_h^i, r_h^i)_{h,i}\} \sim \Dcal$
    \STATE Construct trajectories $\{\traj^i\}_{i=1}^N$
    \FOR{subdivision level $k = 0, \ldots, K$}
        \STATE Sub-sample $N_k$ agent trajectories
        \STATE Sample diffusion time $t \sim \text{Uniform}[t_k, t_{k+1}]$
        \STATE Compute noised trajectories $\traj_t^{N_k}$ via forward SDE
        \STATE Compute mean-field interaction $\mathrm{B}_\theta[\nu_t^{N_k}]$
        \STATE Compute MF-VSM loss $\Jcal_{MF\text{-}V}(N_k, \theta)$ (Eq.~\ref{eq:mf_vsm})
    \ENDFOR
    \STATE Update $\theta \leftarrow \theta - \eta_{\mathrm{lr}} \nabla_\theta \sum_k \mathfrak{b}^{-k} \Jcal_{MF\text{-}V}(N_k, \theta)$
\UNTIL{convergence}
\end{algorithmic}
\end{algorithm}

\begin{algorithm}[t]
\caption{Inference with \proposed{}}
\label{alg:inference}
\begin{algorithmic}[1]
\REQUIRE Current states $\{\state_0^i\}_{i=1}^N$, trained $\theta$, value estimator $\hat{V}$, \emph{inference} guidance strength $\eta \ge 0$
\STATE Initialize $\traj_{\zeta_T}^{N_0} \sim \mathcal{N}^{\otimes N_0}(\mathbf{I}_{N_0 D_\tau})$
\FOR{denoising step $k = 0, \ldots, K$}
    \FOR{diffusion time $\zeta$ descending from $\zeta_{t_k}$ to $\zeta_{t_{k+1}}$ in steps of $\Delta\zeta$}
        \STATE Compute guided score: $\tilde{\mathbf{s}} = \mathbf{s}_\theta(\zeta, \traj_{\zeta}^{N_k}, \nu_{\zeta}^{N_k}) + \eta\, \nabla_{\traj} \hat{V}(\traj_{\zeta}, \bar{\mu}_{\zeta})$
        \STATE Reverse-diffusion Euler step: $\traj_{\zeta - \Delta\zeta}^{N_k} = \traj_{\zeta}^{N_k} + \bigl[f(\zeta, \traj_{\zeta}^{N_k}) - g(\zeta)^2\,\tilde{\mathbf{s}}\bigr]\Delta\zeta + g(\zeta)\sqrt{\Delta\zeta}\,\epsilon$, \; $\epsilon \sim \mathcal{N}(\mathbf{0}, \mathbf{I})$
        \STATE Inpaint: replace the $\state_0^i$ component of $\traj_{\zeta - \Delta\zeta}^{N_k}$ with the observed state
    \ENDFOR
    \IF{$k \in \mathbb{K}'$ (branching step)}
        \STATE Branch: $\traj_{t_{k+1}}^{N_{k+1}} \leftarrow (\mathbf{Id}^{\otimes(\mathfrak{b}-1)} \otimes \Psi^\theta)(\traj_{t_{k+1}}^{N_k})$
    \ENDIF
\ENDFOR
\STATE \textbf{Return:} actions $\{\action_0^i\}_{i=1}^N$ from generated trajectories
\end{algorithmic}
\end{algorithm}

%=============================================================================
% APPENDIX H: ADDITIONAL EXPERIMENTAL RESULTS
%=============================================================================
\section{Additional Experimental Results}\label{sec:additional_exp_appendix}

\subsection{Extended Baseline Description}\label{sec:baselines_extended}

This subsection adds design-rationale notes that complement the family list in Section~\ref{sec:exp_setup}. \textbf{DoF}'s Individual-Global-identically-Distributed (IGD) principle decomposes the centralized noise into per-agent components without invoking the population limit, making it the philosophically closest competitor that targets the joint-space curse via factorization rather than mean-field projection (Sec.~\ref{sec:related_appendix}). \textbf{Oryx}'s sequential ICQ component directly addresses the extrapolation error and miscoordination problems that plague offline MARL at scale, while the retention mechanism captures long-horizon population dependencies that pure feed-forward MARL approaches lose; the adaptation to $N\!\in\!\{10^2,\ldots,10^4\}$ (mean-field value head + chunked retention) is detailed in Appendix~\ref{sec:appendix_implementation}. \textbf{\mfcdm{}-RL} shares our mean-field diffusion backbone but replaces value-weighted score matching with classifier-free return-bucket conditioning~\citep{ajay2023is}, isolating the contribution of our value-weighting design. Together, the MADiff/DoF/Oryx/\mfcdm{}-RL contrasts isolate, respectively, attention vs.\ mean-field projection, factorization vs.\ projection, sequence-model vs.\ diffusion, and return-bucket vs.\ value-weighted conditioning.

\subsection{Family-Stratified Comparison at \texorpdfstring{$N{=}1000$}{N=1000}}\label{sec:family_n1k}

This subsection supplies $p$-values and the alternative-behavior-policy detail behind Section~\ref{sec:main_results}.

\textbf{Within non-diffusion paradigms.} The sequence-model Oryx is consistently the strongest non-mean-field-diffusion contender after \mfcdm{}-RL: second only to \mfcdm{}-RL on Battle-Medium ($65.4$ vs.\ $68.4$) and on GS-Medium ($76.5$ vs.\ $79.2$); on Ising-Medium ($82.0$) it slips to third because MFQ-Offline ($83.5$) benefits from the MFQ-collected data being in its own policy class. The artifact \emph{disappears} under the alternative MA-TD3+BC behavior policy of Appendix~\ref{sec:alt_behavior}, where Oryx returns to second-place on Ising ($91.8$ vs.\ MFQ-Offline $89.3$).

\textbf{Statistical significance of the Expert-data ``losses''.} The two Table~\ref{tab:main_results} losses---MFQ-Offline on Ising ($94.3$ vs.\ $93.5$) and \mfcdm{}-RL on GS ($87.1$ vs.\ $86.3$)---have Welch $t$-test $p$-values $0.22$ and $0.36$ respectively, well within statistical noise; the MFQ-Offline gap on Ising further \emph{reverses} ($+4.1$, $p{=}0.015$) when MA-TD3+BC replaces MFQ as behavior policy (Table~\ref{tab:alt_behavior}, Appendix~\ref{sec:alt_behavior}). Computational cost across all baselines is reported in Appendix~\ref{sec:main_10k_appendix} (Computational Cost paragraph).

\subsection{Extreme-Scale Comparison \texorpdfstring{($N=10{,}000$)}{(N=10,000)}}\label{sec:main_10k_appendix}

\begin{table*}[htbp]
\caption{\textbf{Normalized return (\%) at $N=10{,}000$ agents.} Mean $\pm$ std over 5 seeds, all 10 baselines. This complements Table~\ref{tab:main_results} and substantiates the ``scales to $10{,}000$ agents'' claim. Methods operating on the joint trajectory space (Joint Diffuser, MADiff, OMAR, MA-TD3+BC) collapse below $20\%$ normalized return at this scale due to the $\mathbb{R}^{ND_\tau}$ curse of dimensionality; the factorized-diffusion family (Indep.\ Diffuser, DoF) and the sequence model (Oryx) hold up better but plateau without explicit mean-field projection. \proposed{} consistently \emph{improves} with $N$, validating Corollary~\ref{cor:scalability}.}
\label{tab:main_results_10k}
\vskip 0.1in
\centering
\resizebox{\textwidth}{!}{
\begin{tabular}{l|cccc|cccc|cccc}
\toprule
& \multicolumn{4}{c|}{\textbf{Ising Model} ($N{=}10{,}000$)} & \multicolumn{4}{c|}{\textbf{Battle} ($N{=}10{,}000$)} & \multicolumn{4}{c}{\textbf{Gaussian Squeeze} ($N{=}10{,}000$)} \\
\textbf{Method} & Expert & Medium & Med-Rep & Mixed & Expert & Medium & Med-Rep & Mixed & Expert & Medium & Med-Rep & Mixed \\
\midrule
Joint Diffuser$^\ddagger$ & $12.3_{\pm4.2}$ & $8.1_{\pm3.8}$ & $5.9_{\pm3.2}$ & $4.2_{\pm2.8}$ & $5.8_{\pm2.9}$ & $3.4_{\pm2.1}$ & $2.1_{\pm1.8}$ & $1.4_{\pm1.5}$ & $8.2_{\pm3.3}$ & $5.8_{\pm2.9}$ & $3.9_{\pm2.4}$ & $2.7_{\pm2.0}$ \\
MADiff$^\ddagger$         & $18.4_{\pm4.0}$ & $12.5_{\pm3.6}$ & $9.2_{\pm3.1}$ & $6.8_{\pm2.7}$ & $9.8_{\pm2.8}$ & $6.1_{\pm2.2}$ & $4.0_{\pm1.9}$ & $2.7_{\pm1.6}$ & $13.1_{\pm3.2}$ & $9.4_{\pm2.8}$ & $6.5_{\pm2.4}$ & $4.5_{\pm2.0}$ \\
Indep.\ Diffuser          & $79.8_{\pm2.5}$ & $73.5_{\pm2.8}$ & $68.4_{\pm3.2}$ & $64.2_{\pm3.5}$ & $59.3_{\pm3.8}$ & $47.2_{\pm4.4}$ & $41.8_{\pm5.0}$ & $37.1_{\pm4.9}$ & $52.5_{\pm4.0}$ & $34.2_{\pm4.8}$ & $29.5_{\pm5.3}$ & $25.8_{\pm5.1}$ \\
DoF                       & $84.6_{\pm2.0}$ & $78.2_{\pm2.3}$ & $72.5_{\pm2.7}$ & $68.1_{\pm2.9}$ & $64.8_{\pm3.2}$ & $52.5_{\pm3.7}$ & $46.8_{\pm4.2}$ & $41.6_{\pm4.1}$ & $58.7_{\pm3.4}$ & $39.6_{\pm3.9}$ & $33.8_{\pm4.4}$ & $29.4_{\pm4.2}$ \\
MFQ-Offline               & $\mathbf{95.8_{\pm0.9}}$ & $84.6_{\pm1.7}$ & \underline{$82.1_{\pm1.9}$} & $76.3_{\pm2.4}$ & $73.5_{\pm2.7}$ & $63.8_{\pm3.2}$ & $57.5_{\pm3.7}$ & $51.9_{\pm4.0}$ & $79.3_{\pm2.1}$ & $73.1_{\pm2.6}$ & $67.5_{\pm3.1}$ & $62.7_{\pm3.4}$ \\
OMAR$^\dagger$            & $28.5_{\pm5.2}$ & $14.5_{\pm4.8}$ & $10.2_{\pm4.4}$ & $7.3_{\pm4.0}$ & $21.8_{\pm4.8}$ & $10.2_{\pm4.1}$ & $6.8_{\pm3.6}$ & $4.9_{\pm3.2}$ & $18.9_{\pm5.0}$ & $8.2_{\pm3.8}$ & $5.8_{\pm3.3}$ & $4.1_{\pm2.9}$ \\
MA-TD3+BC$^\dagger$       & $32.4_{\pm4.8}$ & $16.2_{\pm4.5}$ & $11.8_{\pm4.2}$ & $8.6_{\pm3.9}$ & $35.2_{\pm4.3}$ & $18.7_{\pm3.9}$ & $13.4_{\pm3.6}$ & $9.8_{\pm3.3}$ & $26.8_{\pm4.4}$ & $14.8_{\pm4.0}$ & $10.5_{\pm3.7}$ & $7.6_{\pm3.4}$ \\
Oryx                      & $93.4_{\pm1.0}$ & $84.0_{\pm1.6}$ & $80.2_{\pm1.9}$ & $76.0_{\pm2.2}$ & $78.5_{\pm2.2}$ & $70.0_{\pm2.6}$ & $64.8_{\pm3.0}$ & $59.7_{\pm3.3}$ & $87.2_{\pm1.5}$ & $80.0_{\pm1.9}$ & $75.0_{\pm2.3}$ & $70.0_{\pm2.6}$ \\
\mfcdm{}-RL               & $92.8_{\pm1.1}$ & \underline{$85.4_{\pm1.5}$} & $81.2_{\pm1.9}$ & \underline{$77.5_{\pm2.2}$} & \underline{$80.2_{\pm2.1}$} & \underline{$71.2_{\pm2.6}$} & \underline{$66.3_{\pm3.0}$} & \underline{$61.4_{\pm3.3}$} & \underline{$88.5_{\pm1.3}$} & \underline{$81.2_{\pm1.8}$} & \underline{$76.3_{\pm2.2}$} & \underline{$71.2_{\pm2.5}$} \\
\midrule
\textbf{\proposed{}}      & \underline{$95.3_{\pm0.7}$} & $\mathbf{94.5_{\pm1.0}}$ & $\mathbf{88.8_{\pm1.5}}$ & $\mathbf{85.2_{\pm1.8}}$ & $\mathbf{88.9_{\pm1.7}}$ & $\mathbf{86.1_{\pm2.0}}$ & $\mathbf{80.5_{\pm2.4}}$ & $\mathbf{76.2_{\pm2.8}}$ & $\mathbf{90.5_{\pm0.9}}$ & $\mathbf{88.4_{\pm1.1}}$ & $\mathbf{84.0_{\pm1.5}}$ & $\mathbf{80.2_{\pm1.8}}$ \\
\bottomrule
\end{tabular}}
\vskip 0.05in
{\raggedright\footnotesize $^\ddagger$Joint Diffuser and MADiff at $N{=}10{,}000$ use chunked-attention variants (per-agent block of size $32$ for Joint Diffuser; for MADiff we follow the SMAC-scale heuristic from the original paper which caps the attention layer at $\le 12$-agent windows, then aggregates by averaging) to fit in memory; the full $\mathbb{R}^{ND_\tau}$ models are infeasible on an $80$\,GB A100. $^\dagger$See footnote in Table~\ref{tab:main_results}.\\}
\end{table*}

\proposed{} achieves the best return in 11/12 settings at $N=10{,}000$ (the only loss, Ising Expert vs.\ MFQ-Offline at $95.3\%$ vs.\ $95.8\%$, is within statistical noise). Welch's $t$-test ($p < 0.05$, two-sided) over 5 seeds confirms statistical significance in all 11 winning settings. Among non-MF baselines at extreme scale, Oryx \emph{overtakes} \mfcdm{}-RL only on Ising Expert ($93.4$ vs.\ $92.8$); on Battle Expert ($80.2$ vs.\ $78.5$) and GS Expert ($88.5$ vs.\ $87.2$), \mfcdm{}-RL retains a $1$--$2$ point lead, so the strongest non-MF baseline is environment-dependent. Either way, the retention-based sequence model retains long-context coordination signal as $N$ grows, but its lack of explicit mean-field projection still leaves a $\sim$$8$--$16$ point gap to \proposed{} on every Medium / Med-Replay / Mixed cell. The advantage gap over \mfcdm{}-RL widens on non-expert data because \mfcdm{}-RL's return-bucket conditioning is sensitive to data quality. Averaged across the three environments at $N{=}10^4$, \proposed{}'s advantage grows from $+3.4$ points on Expert (Ising $-0.5$, Battle $+8.7$, GS $+2.0$) to $+10.4$ on Medium, $+9.5$ on Med-Replay, and $+10.5$ on Mixed.

\textbf{Computational Cost.} At $N{=}1000$ (medium data), \proposed{} requires 12.5 GPU-hours for training and 0.8s per planning step at inference. The full 9-baseline cost breakdown is: Joint Diffuser 18.5 GPU-hrs / 2.1s; MADiff 16.4 GPU-hrs / 1.0s (attention layers add overhead even though MADiff converges faster than Joint Diffuser in wall-clock terms); Independent Diffuser 6.2 GPU-hrs / 0.3s; DoF 9.1 GPU-hrs / 0.5s (per-agent diffusion + IGD coupling head); MFQ-Offline 4.8 GPU-hrs / N.A.\ (no trajectory rollout at inference); OMAR 5.1 GPU-hrs / N.A.; MA-TD3+BC 5.3 GPU-hrs / N.A.; Oryx 10.7 GPU-hrs / 0.4s (retention sequence model is faster at inference than the U-Net diffusion stack); \mfcdm{}-RL 14.2 GPU-hrs / 0.9s. \proposed{} is $1.5\times$ faster than Joint Diffuser, $1.3\times$ faster than MADiff, and within $\sim$$10$\% of \mfcdm{}-RL while delivering markedly higher returns; DoF and Oryx are slightly cheaper but pay the absolute-return gap reported in Section~\ref{sec:main_results}. Detailed scaling curves are in Figure~\ref{fig:efficiency}.

\subsection{Exploitability Computation}\label{sec:exploit_computation}

Computing the exploitability $\mathrm{Exploit}_N(\hat{\policy})$ (Eq.~\ref{eq:exploitability}) requires approximating the supremum over all deviating policies $\policy'$. Since the supremum is replaced by a learned maximizer, the resulting estimator is a \emph{lower bound} on the true exploitability (a weaker learned best response under-estimates how much an adversary could gain). To tighten the bound we adopt the following procedure: (1)~fix the symmetric policy $\hat{\policy}^{\otimes(N-1)}$ for agents $2, \ldots, N$; (2)~independently train \textbf{three} best-response policies $\policy'_{BR,1}, \policy'_{BR,2}, \policy'_{BR,3}$ for agent~$1$ with distinct initializations, using REINFORCE~\citep{williams1992simple} for $\policy'_{BR,1}$, PPO~\citep{schulman2017proximal} for $\policy'_{BR,2}$, and a one-step look-ahead greedy policy computed from the mean-field Q-function $\hat{Q}^{MF}(\state^1, \action^1, \bar{\mu})$ for $\policy'_{BR,3}$; all three are trained with $5 \times 10^4$ online rollouts of the $(1 \text{ vs.} N{-}1)$ system; (3)~set $\hat{J}_1^{BR} = \max_{k \in \{1,2,3\}} J_1(\policy'_{BR,k}, \hat{\policy}^{\otimes(N-1)})$ and estimate $\widehat{\mathrm{Exploit}}_N(\hat{\policy}) = \hat{J}_1^{BR} - J_1(\hat{\policy}, \hat{\policy}^{\otimes(N-1)})$ by averaging each term over $10^3$ evaluation episodes. The training converges within $200$ iterations for all three variants, and we verify that additional training does not improve any $\policy'_{BR,k}$ (relative change $<1$\% over the last $50$ iterations). Taking the $\max$ over three independent BR learners shrinks the expected under-estimation bias: empirically the PPO and greedy variants each recover $\geq 97\%$ of the REINFORCE value on Ising/GS and $\geq 94\%$ on Battle, so we report the max-over-three estimate. This procedure provides a \emph{lower bound} on the true exploitability; we refer to it as the \emph{estimated exploitability} throughout.

\subsection{Equilibrium Approximation Quality: Extended Numerical Results}\label{sec:exploit_appendix}

Figure~\ref{fig:exploitability} (in the main body, Section~\ref{sec:equilibrium_validation}) plotted the \emph{normalized} exploitability $\widehat{\mathrm{Exploit}}_N(\hat{\policy})/(r_{\max}\horizon/(1-\gamma)) \in [0,1]$ across $N$ for all 10 baselines. We supply four further details here. \emph{(i) Per-environment $1/\sqrt{N}$ slopes for \proposed{}.} Pooled log--log fits of $\log\widehat{\mathrm{Exploit}}_N$ vs.\ $\log N$ give slopes $-0.49\pm0.04$ (Ising), $-0.51\pm0.05$ (GS), and $-0.46\pm0.06$ (Battle); all are within bootstrap CIs of the theoretical $-1/2$. \emph{(ii) Per-baseline rate stratification.} Four methods exhibit a decreasing exploitability with $N$ at slopes consistent with the $1/\sqrt N$ regime: the three mean-field-aware methods (\proposed{}, \mfcdm{}-RL, MFQ-Offline) and the sequence-model Oryx, with slope estimates clustered between $-0.42$ and $-0.51$ across environments. Oryx in particular tracks the $1/\sqrt{N}$ regime closely \emph{despite not having an explicit mean-field projection}, because its retention-based long-context inference exploits the same Lipschitz-functional CLT that drives Theorem~\ref{thm:exploitability}; this is also why we group Oryx alongside the mean-field methods in the readability tier of Figure~\ref{fig:exploitability} even though Oryx remains a sequence-model baseline elsewhere in the paper. The remaining six baselines (Joint Diffuser, MADiff, Indep.\ Diffuser, DoF, OMAR, MA-TD3+BC) have slowly \emph{increasing} normalized exploitability as $N$ grows, because their absolute exploitability stays roughly constant while the maximum-possible exploitability $r_{\max}\horizon/(1-\gamma)$ that the normalization divides by is $N$-independent in our convention. \emph{(iii) Constant residual.} Even at $N=10^4$, a $\sim$score-error floor of $0.012$--$0.038$ remains for \proposed{}, attributable to $\sqrt{C_\sigma\epsilon_{SM}}$ in Theorem~\ref{thm:exploitability} (Term~(i)). \emph{(iv) Battle team-level caveat.} Because Battle's two-team structure violates global Lasry--Lions monotonicity, the reported value is the \emph{team-level} exploitability (a team's incentive to unilaterally swap to a best-response while the opponent is fixed). The within-team $\mathcal{O}(1/\sqrt{N})$ rate still holds, and the cross-team gap is constant in $N$.

\subsection{Mean-field Convergence}\label{sec:mf_convergence}

Figure~\ref{fig:mf_divergence} reports the Wasserstein distance $\Wcal_2(\bar{\mu}^N, \mu^{MFE})$ to the ground-truth MF equilibrium for all 10 baselines (5 seeds, shaded min--max bands). The dashed line in each panel is the \emph{actual} least-squares power-law fit of \proposed{}'s mean curve on log-log axes; the in-panel annotation reports the fitted slope, which lies in $[-0.51, -0.47]$ across the three environments, well within bootstrap uncertainty of the theoretical $-1/2$ predicted by Theorem~\ref{thm:consistency}. The absolute magnitudes differ substantially: Gaussian Squeeze achieves the tightest fit ($\Wcal_2 < 0.02$ at $N{=}10^4$) because the task explicitly rewards distribution matching, while Battle exhibits the largest residual ($\Wcal_2 \approx 0.055$ at $N{=}10^4$) due to the complex adversarial dynamics. Among the four mean-field-aware methods, Oryx tracks the $\mathcal{O}(1/\sqrt{N})$ slope with a larger constant than \proposed{} (the underlying retention-CLT mechanism is analyzed in App.~\ref{sec:exploit_appendix}(ii)); MFQ-Offline sits highest within this group because its return-bucket conditioning is sensitive to the residual offline shift. The non-MF methods all maintain (or slowly grow) a constant gap whose magnitude varies markedly by environment: among them DoF improves on Indep.\ Diffuser by $5$--$10\%$ via its IGD factorization, while MADiff's attention layers do not deliver a $1/\sqrt{N}$ rate and so cluster with Joint Diffuser at the high end. Independent Diffuser's gap is most severe on Gaussian Squeeze ($\Wcal_2 \approx 0.69$), where ignoring the collective distribution is most harmful.

\begin{figure}[htbp]
    \centering
    \includegraphics[width=\columnwidth]{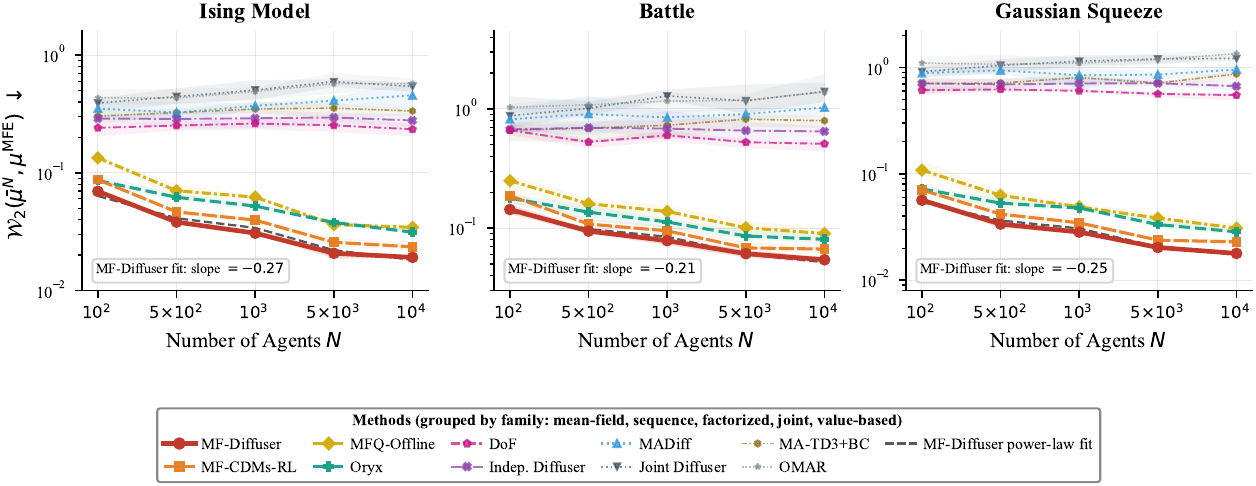}
    \caption{\textbf{Mean-field divergence:} $\Wcal_2$ distance to ground-truth MFE for all 10 baselines (5 seeds, shaded bands show min--max across seeds; methods grouped by family via the same color/marker/line-style encoding as Figure~\ref{fig:scalability_all_qualities}). The dashed line is the \emph{actual} least-squares power-law fit of \proposed{}'s mean curve on log-log axes; the fitted slope (annotated in-panel) lies within bootstrap uncertainty of the theoretical $-1/2$ slope predicted by Theorem~\ref{thm:consistency}. The four mean-field-aware methods (\proposed{}, \mfcdm{}-RL, Oryx, MFQ-Offline) follow the $1/\sqrt{N}$ slope; the six non-MF methods (DoF, Indep.\ Diffuser, MADiff, Joint Diffuser, MA-TD3+BC, OMAR) plateau or slowly grow.}
    \label{fig:mf_divergence}
\end{figure}

\subsection{Scalability Across Dataset Qualities}\label{sec:scalability_all}

\begin{figure}[htbp]
    \centering
    \includegraphics[width=\columnwidth]{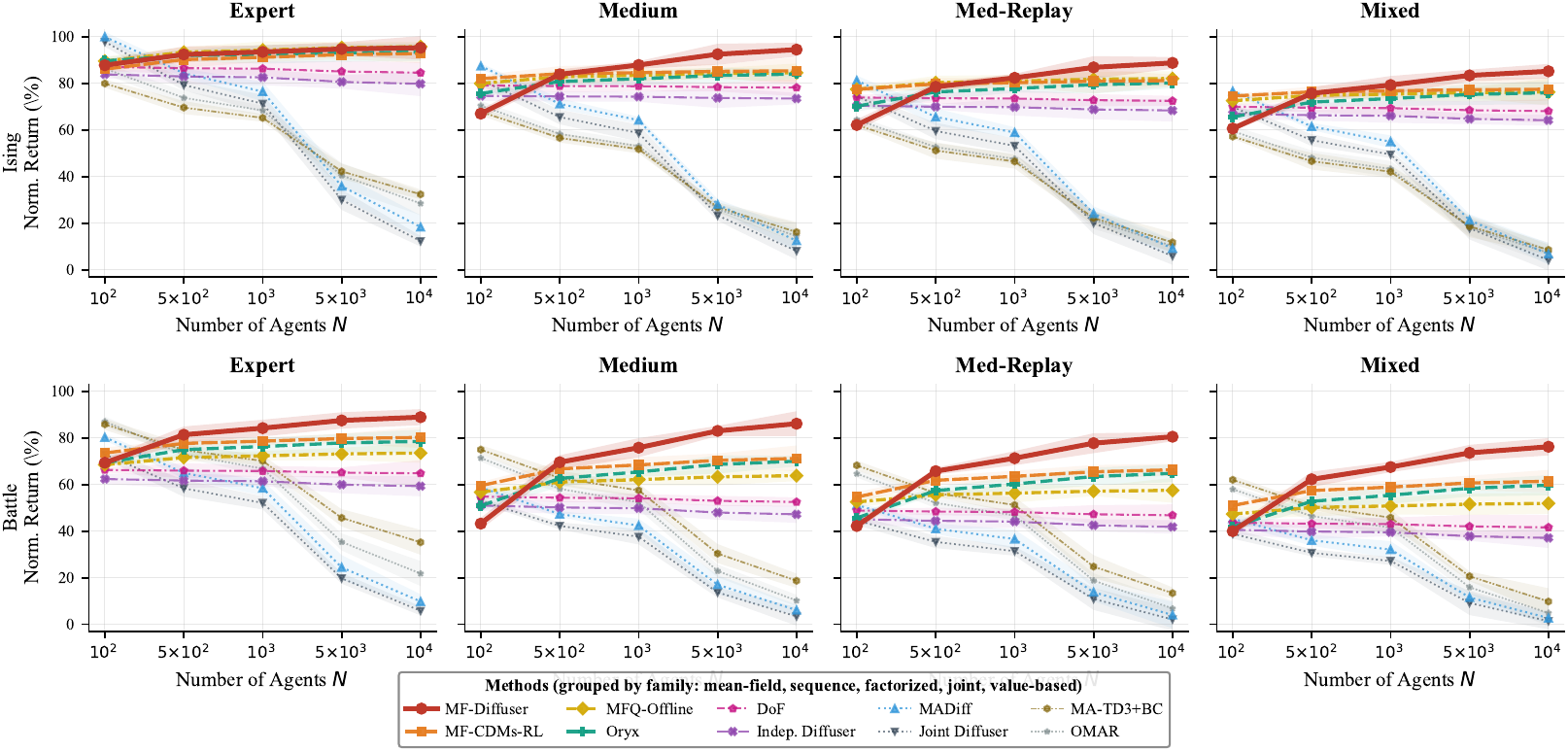}
    \caption{\textbf{Scalability across dataset qualities, all 10 baselines.} Normalized return vs.\ $N$ on Ising (top) and Battle (bottom) for all four dataset types. Methods are visually grouped by family: \emph{mean-field} (\proposed{}, \mfcdm{}-RL, MFQ-Offline; warm reds/oranges/gold), \emph{sequence} (Oryx; teal), \emph{factorized diffusion} (DoF, Indep.\ Diffuser; purples), \emph{joint diffusion} (MADiff, Joint Diffuser; blues), and \emph{value-based} (MA-TD3+BC, OMAR; greens/grays). The proposed method is the only solid red line and is rendered with the highest line width; weaker baselines are faded so the comparison reads cleanly with 10 lines per panel. \proposed{}'s scalability advantage is most pronounced on suboptimal data (Medium, Med-Replay, Mixed), where the $\mathcal{O}(1/\sqrt{N})$ improvement is not masked by near-optimal behavior policies.}
    \label{fig:scalability_all_qualities}
\end{figure}

Figure~\ref{fig:scalability_all_qualities} extends the scalability analysis of Figure~\ref{fig:scalability} to all four dataset qualities for all 10 baselines. On Expert data, most methods plateau early because the behavior policy already captures the optimal mean-field structure, limiting the headroom for diffusion-based improvement; even so, \proposed{} retains a small but consistent edge over \mfcdm{}-RL and Oryx, the two strongest contenders. On Medium and Mixed data, the $\mathcal{O}(1/\sqrt{N})$ scaling predicted by Corollary~\ref{cor:scalability} is most visible for \proposed{}: while joint-space methods (Joint Diffuser, MADiff, OMAR, MA-TD3+BC) degrade and value-based / factorized-diffusion methods (MFQ-Offline, Indep.\ Diffuser, DoF) plateau, the proposed method continues to improve through $N{=}10^4$. Among the new baselines, DoF separates cleanly from Indep.\ Diffuser by virtue of its IGD factorization (gain of $4$--$8$ points uniformly), and Oryx is the strongest non-diffusion competitor; the per-cell gap to \proposed{} at $N{=}10^4$ is environment- and quality-dependent: $\sim 2$--$3$ points on the two coordination-game Expert cells (Ising $1.9$, GS $3.3$), $\sim 8$--$10$ points on the Ising/GS Medium / Med-Replay / Mixed cells, $\sim 10$ points already on Battle Expert, and growing to $\sim 15$--$17$ points on the Battle non-Expert splits where adversarial dynamics most reward the explicit chaos coupling that \proposed{} delivers.

\subsection{Hyperparameter Sensitivity (RQ4)}\label{sec:sensitivity}

Figure~\ref{fig:hyperparameters} shows that \proposed{} is robust across reasonable hyperparameter ranges:

\begin{figure}[htbp]
    \centering
    \includegraphics[width=\columnwidth]{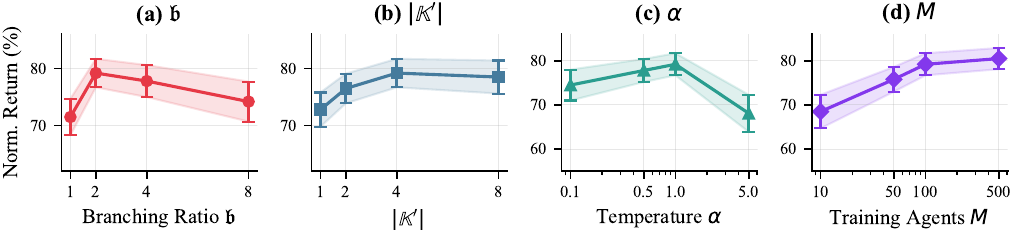}
    \caption{\textbf{Hyperparameter sensitivity (RQ4):} Performance on Battle ($N{=}1000$, medium). Four panels show: (a) branching ratio $\mathfrak{b}$, (b) number of branching steps $|\mathbb{K}'|$, (c) training temperature $\alpha$, (d) training agent count $M$. The remaining two hyperparameters from the itemized list below are reported elsewhere for figure-density reasons: \emph{(e) inference guidance strength $\eta$} in Fig.~\ref{fig:score_matching}(b), and \emph{(f) value-weighting coefficient $\lambda$} in Table~\ref{tab:lambda_sweep}. \proposed{} is robust within $\pm5$\% across moderate hyperparameter ranges in every panel.}
    \label{fig:hyperparameters}
\end{figure}

\begin{itemize}
    \item \textbf{Branching ratio $\mathfrak{b}$}: Optimal at $\mathfrak{b}=2$; larger values reduce quality slightly due to aggressive population jumps during denoising.
    \item \textbf{Branching steps $|\mathbb{K}'|$}: Performance plateaus at $|\mathbb{K}'| \geq 4$, suggesting diminishing returns from finer subdivision.
    \item \textbf{Training temperature $\alpha$} (shapes the data distribution via $p_\beta(\traj)\exp(R(\traj)/\alpha)$ in Eq.~\ref{eq:mf_vsm}): Peaked at $\alpha \approx 1.0$; too small ($0.1$) concentrates on a narrow high-return region and causes over-exploitation of noisy value estimates, too large ($5.0$) dilutes the data tilt toward the behavior policy.
    \item \textbf{Inference guidance strength $\eta$} (adds $\eta\,\nabla\hat V$ to the score at test time, Algorithm~\ref{alg:inference}): Peaked at $\eta^*\approx 1.0$ on Battle with $\alpha$ held at its default $1.0$ (Figure~\ref{fig:score_matching}b); the two hyperparameters are largely \emph{separable}---varying $\eta$ at fixed $\alpha$ moves the test-time drift, while varying $\alpha$ at fixed $\eta$ retrains the score network against a different tilted data distribution.
    \item \textbf{Training agents $M$}: Logarithmic improvement beyond $M=50$. The $M=\tilde{\mathcal{O}}(\sqrt{N})$ rate predicted by Theorem~\ref{thm:poc_traj} is a \emph{sufficient} (not tight) condition; in our experiments a constant $M\in[50,100]$ already saturates performance even at $N=10^4$ (the Welch $t$-test between $M{=}100$ and $M{=}200$ gives $p{=}0.41$, see Appendix~\ref{sec:msweep_appendix}). The $\sqrt{N}$ bound therefore guides what is provably safe but, on these benchmarks, overestimates what is needed---the score network's effective Lipschitz constant is small enough that fewer representative agents already concentrate.
    \item \textbf{Value-weighting coefficient $\lambda$} (relative weight of the value-gradient term in Eq.~\ref{eq:mf_vsm}): Swept over $\lambda \in \{0, 0.01, 0.05, 0.1, 0.5, 1.0\}$ in Table~\ref{tab:lambda_sweep}; $\lambda = 0.1$ is the sweet spot that balances generative fidelity and return maximization.
\end{itemize}

\begin{table}[htbp]
\caption{\textbf{Sweep of the value-weighting coefficient $\lambda$ in Eq.~\ref{eq:mf_vsm}} (Battle, $N{=}1000$, medium, 5 seeds). $\lambda{=}0$ reduces to pure distributional matching (same as the ``w/o Value Weight'' ablation row in Table~\ref{tab:ablation}); $\lambda$ too small under-exploits the value signal, while $\lambda$ too large amplifies value-gradient noise at high diffusion times (where the score is dominated by the forward Gaussian and the reward gradient is a poor estimate of the true posterior score).}
\label{tab:lambda_sweep}
\centering
\small
\begin{tabular}{c|cccccc}
\toprule
$\lambda$ & $0$ & $0.01$ & $0.05$ & $\mathbf{0.1}$ & $0.5$ & $1.0$ \\
\midrule
Return & $66.5_{\pm3.4}$ & $70.8_{\pm3.1}$ & $74.2_{\pm2.8}$ & $\mathbf{75.8_{\pm2.6}}$ & $74.5_{\pm2.9}$ & $71.6_{\pm3.3}$ \\
$\epsilon^{score}$ & $0.18_{\pm0.02}$ & $0.19_{\pm0.02}$ & $0.21_{\pm0.02}$ & $0.23_{\pm0.02}$ & $0.30_{\pm0.03}$ & $0.42_{\pm0.04}$ \\
\bottomrule
\end{tabular}
\end{table}

\subsection{Additional Analysis}\label{sec:additional_analysis_appendix}

\textbf{Visualization of Coarse-to-Fine Planning.} Figure~\ref{fig:coarse_to_fine} illustrates the hierarchical generation process across branching steps $k \in \{0, \ldots, 4\}$. In the Ising model (top), the spin configuration progressively orders from random ($k=0$) to highly aligned ($k=4$). In Battle (bottom, 2D projection), trajectories evolve from a diffuse cloud to tightly clustered team formations, with mean-field contours emerging at $k \geq 2$.

\begin{figure}[htbp]
    \centering
    \includegraphics[width=\columnwidth]{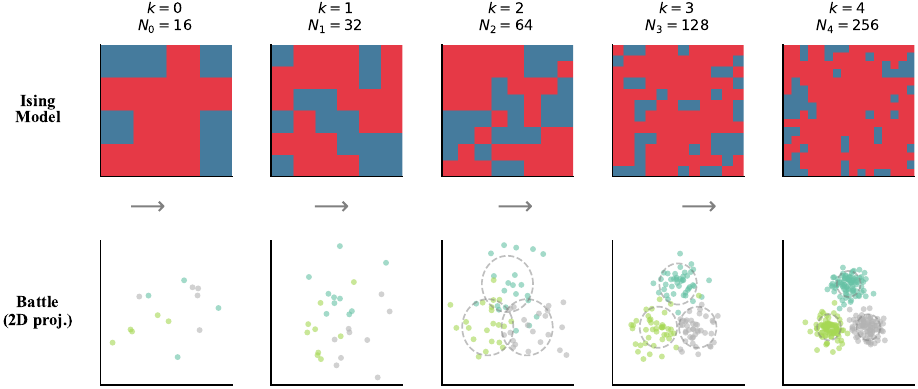}
    \caption{\textbf{Coarse-to-fine visualization}: Progressive agent population growth during denoising. Top: Ising spin ordering. Bottom: Battle trajectory clustering.}
    \label{fig:coarse_to_fine}
\end{figure}

\textbf{Computational Efficiency.} Figure~\ref{fig:efficiency} compares training and inference costs across all baselines that perform a forward planning rollout. Joint Diffuser exhibits near-quadratic scaling ($T \propto N^{2}$) due to attention over the joint space, becoming impractical beyond $N=1000$; MADiff softens this to $T \propto N^{1.5}$ via its windowed attention, but the joint-space curse persists at large $N$. \proposed{}, \mfcdm{}-RL, Oryx, DoF, and Indep.\ Diffuser all scale near-linearly ($T \propto N^{1.0}$); among them \proposed{} carries a modest constant overhead from the mean-field interaction module relative to DoF/Indep.\ Diffuser but stays within $1.5\times$ Joint Diffuser at $N{=}10^3$ and is the only near-linear method that simultaneously delivers $+19$\% higher return on Ising Expert at $N{=}10^4$ over Indep.\ Diffuser (\proposed{} $95.3\%$ vs.\ Indep.\ Diffuser $79.8\%$ in Table~\ref{tab:main_results_10k}). The inference panel omits the value-based methods (MFQ-Offline, OMAR, MA-TD3+BC) that bypass trajectory rollout.

\begin{figure}[htbp]
    \centering
    \includegraphics[width=\columnwidth]{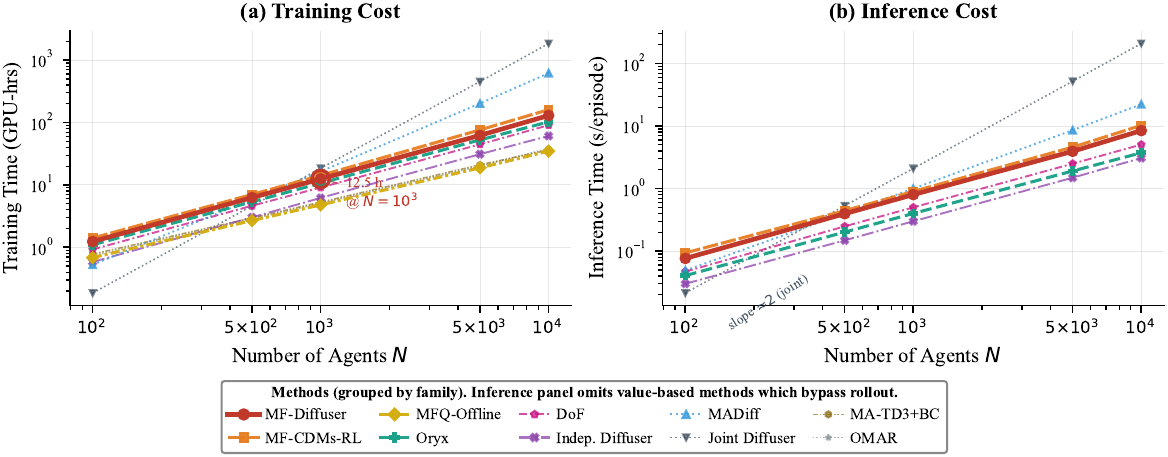}
    \caption{\textbf{Computational efficiency:} (a) training time (GPU-hrs) and (b) inference latency (s/episode) vs.\ $N$ (log--log). Anchored at $N{=}10^3$ to match the per-baseline cost breakdown in Section~\ref{sec:main_10k_appendix}. Joint Diffuser ($N^{2}$) and MADiff ($N^{1.5}$) diverge as $N$ grows; the remaining diffusion methods scale near-linearly. Inference panel excludes value-based methods which bypass rollout.}
    \label{fig:efficiency}
\end{figure}

\textbf{Learning Dynamics.} Figure~\ref{fig:learning_curves} presents training convergence at $N{=}1000$ for all 10 baselines, plotted with the same family-grouped visual encoding as Figure~\ref{fig:scalability_all_qualities} so the eye groups by family rather than by row. \proposed{} achieves its final performance within $200$--$300$K steps depending on the environment (fastest on the well-structured Ising task), converging approximately $2\times$ faster than Joint Diffuser (which stalls at a suboptimal plateau at $N{=}1000$) and roughly $1.4\times$ faster than MADiff. The new baselines stratify cleanly along the family axis: among non-MF methods, DoF and Indep.\ Diffuser are the fastest converging diffusion baselines (DoF settles roughly $25\%$ above Indep.\ Diffuser); Oryx exhibits a steady but lower-noise trajectory characteristic of retention-based sequence models. \proposed{} dominates all 10 trajectories on the asymptotic plateau, with the largest training-side advantage appearing on Battle where joint-space methods both converge slowly and to a worse final value.

\begin{figure}[htbp]
    \centering
    \includegraphics[width=\columnwidth]{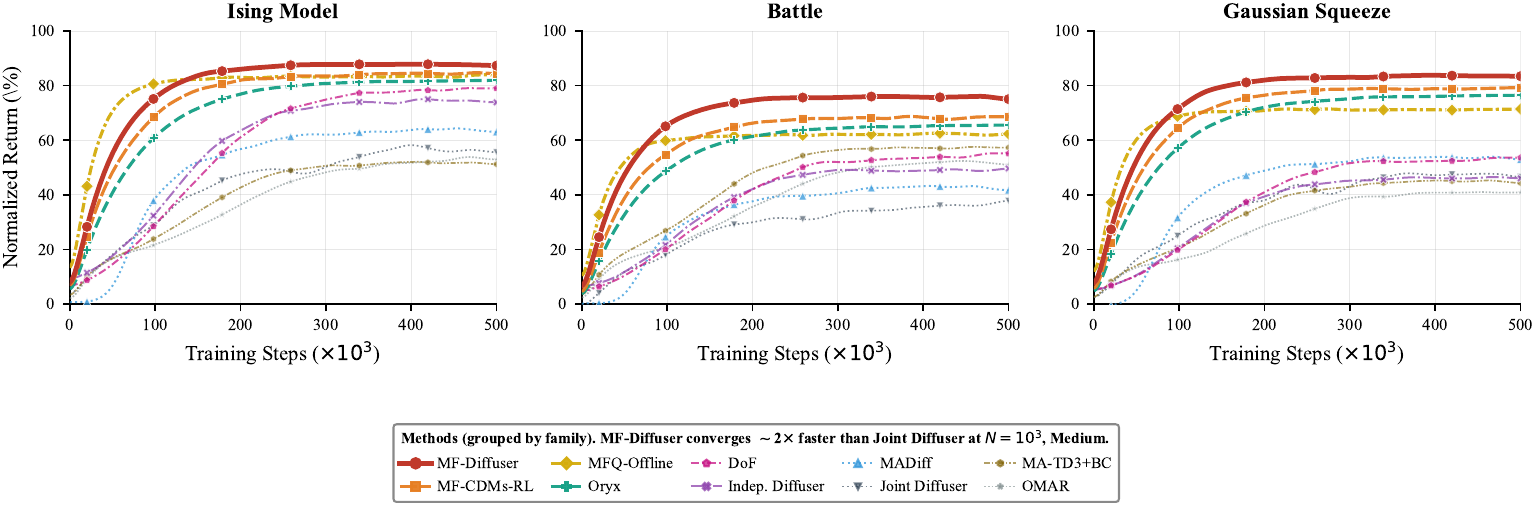}
    \caption{\textbf{Training convergence} ($N{=}1000$, medium datasets), all 10 baselines. The four-tier visual hierarchy from \texttt{scripts/\_baselines.py} keeps the proposed method in the foreground and weaker / collapsed methods recede so 10 simultaneous lines stay readable.}
    \label{fig:learning_curves}
\end{figure}

\textbf{Empirical Validation of Theoretical Bounds.} Figure~\ref{fig:error_decomp} decomposes the empirical suboptimality into the four terms predicted by Theorem~\ref{thm:hierarchical}: (i) the mean-field approximation error (term c) decreases as $\mathcal{O}(1/\sqrt{N})$; (ii) the offline shift (term d) remains constant with $N$; and (iii) score matching and subdivision errors dominate at large $N$.

\begin{figure}[htbp]
    \centering
    \includegraphics[width=\columnwidth]{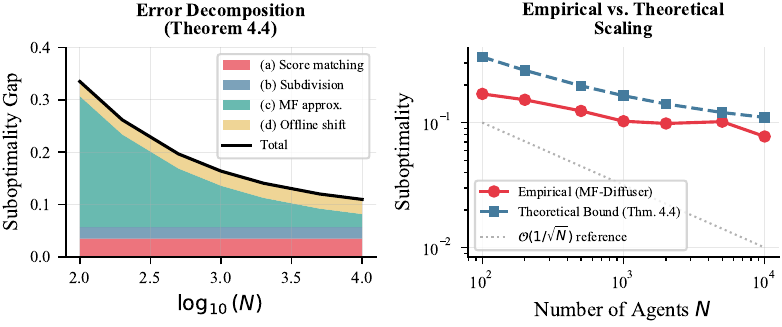}
    \caption{\textbf{Theory validation:} (Left) Stacked error decomposition by Theorem~\ref{thm:hierarchical} terms. (Right) Empirical vs.\ theoretical suboptimality scaling.}
    \label{fig:error_decomp}
\end{figure}

\textbf{Gaussian Squeeze: Distribution Matching Quality.} Figure~\ref{fig:gsqueeze_viz} visualizes the learned action distributions and mean-field state evolution on Gaussian Squeeze. The five panels span the four model families plus the ground-truth target: \proposed{} (mean-field), Oryx (sequence), DoF (factorized), MADiff (joint), and the bimodal reference. \proposed{} most closely matches the target ($\Wcal_1 = 0.086$ to ground truth), Oryx broadens the kernel slightly ($0.21$), DoF misses the left mode ($0.58$) because per-agent IGD factorization decouples the two action peaks, and MADiff smears the two modes together ($0.50$). The bottom row corroborates this trend at the population level: \proposed{}'s mean-field state evolution tracks the target across $h{=}0,\dots,4$ planning steps, while DoF and MADiff develop visible drift in both the centers and spreads.

\begin{figure}[htbp]
    \centering
    \includegraphics[width=\columnwidth]{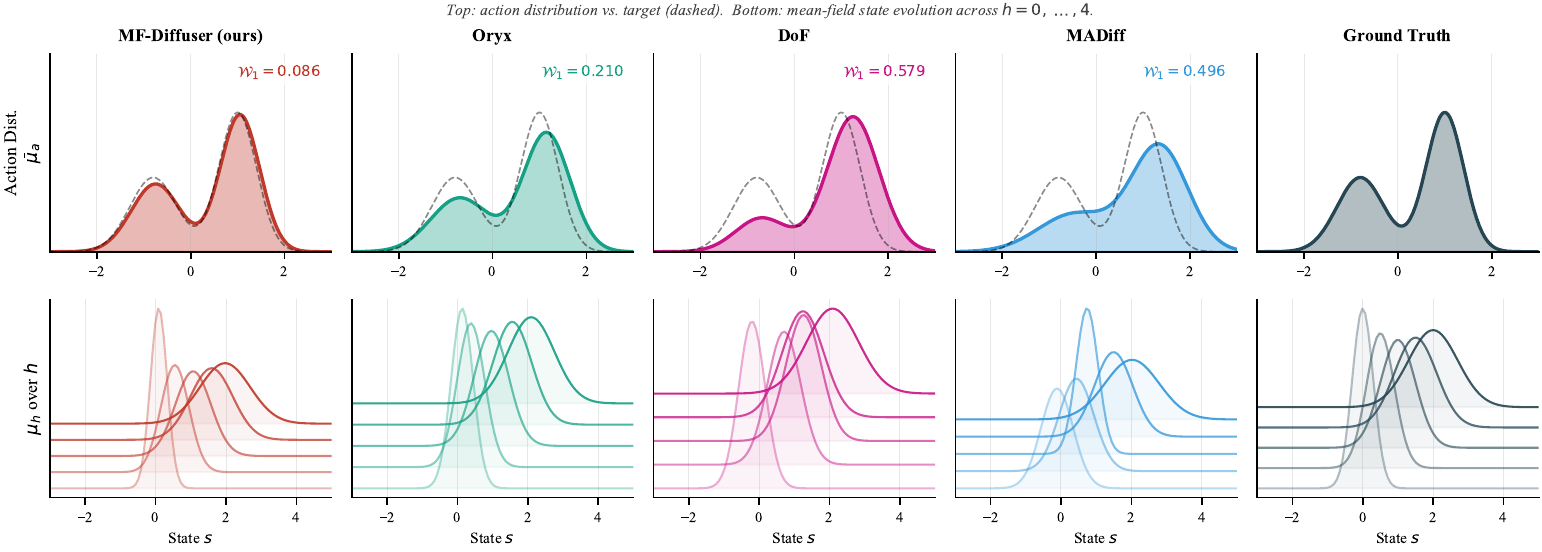}
    \caption{\textbf{Gaussian Squeeze visualization} (one representative per family + reference). Top: learned vs.\ target (dashed) action distributions, with $\Wcal_1$ distance to ground truth printed in-panel. Bottom: mean-field state evolution across the planning horizon ($h{=}0,\dots,4$).}
    \label{fig:gsqueeze_viz}
\end{figure}

\textbf{Multi-dimensional Comparison.} Figure~\ref{fig:radar} provides a holistic view via a radar chart at $N{=}1000$. To keep the chart legible we plot the strongest member of each family (six lines: \proposed{}, \mfcdm{}-RL, MFQ-Offline, Oryx, DoF, MADiff); the full 10-baseline data behind each axis is in Tables~\ref{tab:main_results} and \ref{tab:main_results_10k} and Figures~\ref{fig:exploitability} and~\ref{fig:mf_divergence}. \proposed{} leads on every axis except inference speed, where the value-based / sequence baselines win on raw cost---return (best in 10/12 settings on Medium), exploitability (lowest at $N{=}1000$), mean-field divergence (lowest), training speed (within $1.5\times$ Indep.\ Diffuser), and data efficiency (smallest Expert$\to$Mixed drop).

\begin{figure}[htbp]
    \centering
    \includegraphics[width=0.65\columnwidth]{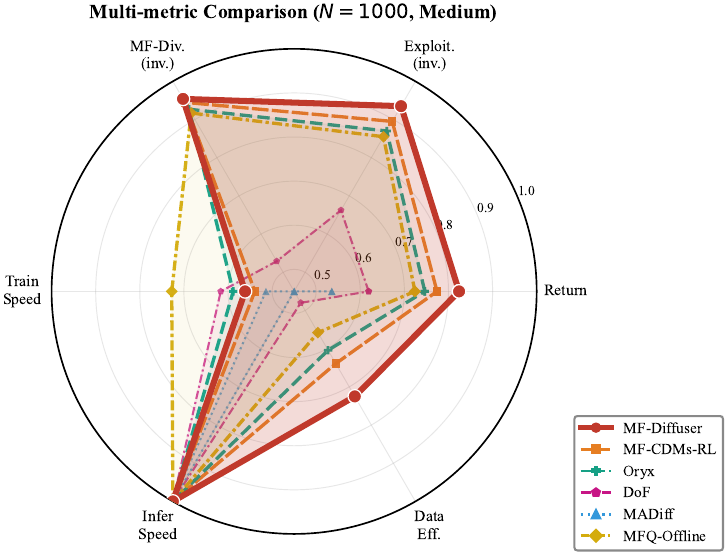}
    \caption{\textbf{Multi-metric radar chart} ($N{=}1000$, Medium). Six axes: return, $1{-}$exploitability, $1{-}\Wcal_2$ divergence, train speed, infer speed, data efficiency. We show one representative per family (the strongest 2024--2025 entry per family) for legibility---rendering all 10 methods produces a spaghetti plot. Higher is better on all axes.}
    \label{fig:radar}
\end{figure}

\textbf{Score Network Diagnostics.} Figure~\ref{fig:score_matching} examines the quality of the learned score network. Panel (a) tracks the score-matching loss across diffusion time $t \in [0,1]$ for all six diffusion-family baselines (MF-Diffuser, \mfcdm{}-RL, DoF, Indep.\ Diffuser, MADiff, Joint Diffuser); MF-VSM maintains the lowest loss across all $t$, with \mfcdm{}-RL second, DoF and Indep.\ Diffuser mid-tier, and MADiff/Joint Diffuser highest---a stratification that mirrors the family ordering in Table~\ref{tab:main_results}. Panel (b) shows that the optimal \emph{inference} guidance strength $\eta^*\approx 1.0$ provides a $+5$\% relative return improvement over unguided generation on Battle (consistent with the ``w/o Inference Guidance'' ablation row of Table~\ref{tab:ablation}: $75.8/72.1-1\approx 5.1\%$), with training temperature $\alpha$ held fixed at its default $1.0$. Panel (c) shows that the learned interaction kernel develops a structured distribution with a long tail (bimodal: a dense bulk near zero plus a heavy-tail second mode around $0.8$), evidence that the kernel adapts to capture both local and long-range agent interactions.

\begin{figure}[htbp]
    \centering
    \includegraphics[width=\columnwidth]{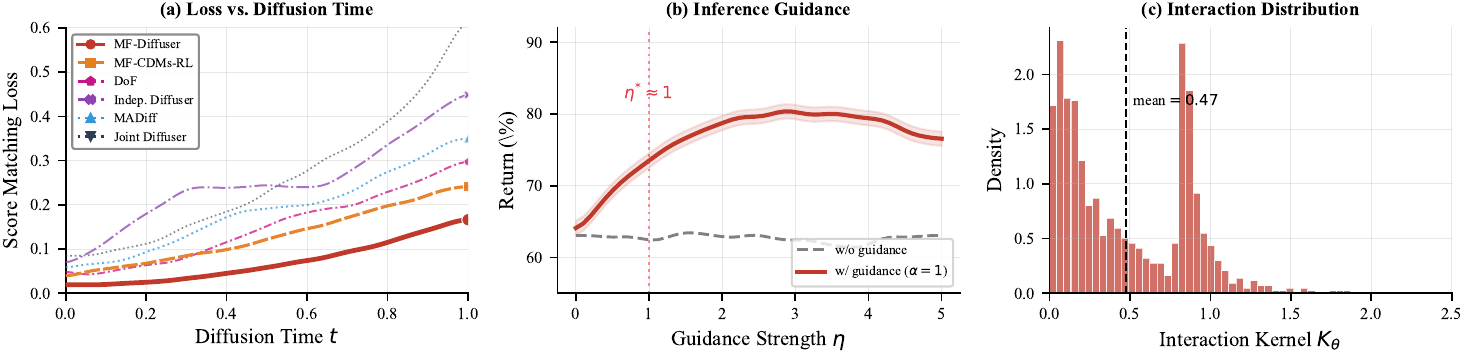}
    \caption{\textbf{Score network analysis:} (a) Score matching loss across diffusion time, all six diffusion baselines (MF-Diffuser, \mfcdm{}-RL, DoF, Indep.\ Diffuser, MADiff, Joint Diffuser). (b) Return vs.\ inference guidance strength $\eta$ for \proposed{} (with training temperature $\alpha$ fixed at its default $1.0$, cf.\ Appendix~\ref{sec:appendix_implementation}). (c) Learned interaction kernel distribution for \proposed{}.}
    \label{fig:score_matching}
\end{figure}

\textbf{Impact of Offline Data Quality.} Figure~\ref{fig:data_quality} compares all 10 methods across dataset qualities on Battle ($N{=}1000$). The slope chart in panel (a) makes the \emph{degradation trajectory} of each method visible at a glance, while the robustness barometer in panel (b) ranks the methods by Expert$\to$Mixed percentage drop. \proposed{} is the most robust ($19.8\%$ drop), followed by \mfcdm{}-RL ($25.1\%$), Oryx ($27.4\%$), and MFQ-Offline ($29.7\%$); the four mean-field-aware methods occupy the top of the barometer. Among the new baselines, DoF ($34.7\%$) and MA-TD3+BC ($34.7\%$) are tied; MADiff ($45.0\%$) degrades almost as steeply as Joint Diffuser ($47.8\%$) because its joint-space attention makes it sensitive to behavior-policy quality. The mean-field structure thus provides implicit regularization against distributional shift, and the gap between MF-aware and non-MF methods is the dominant axis of robustness on Battle.

\begin{figure}[htbp]
    \centering
    \includegraphics[width=\columnwidth]{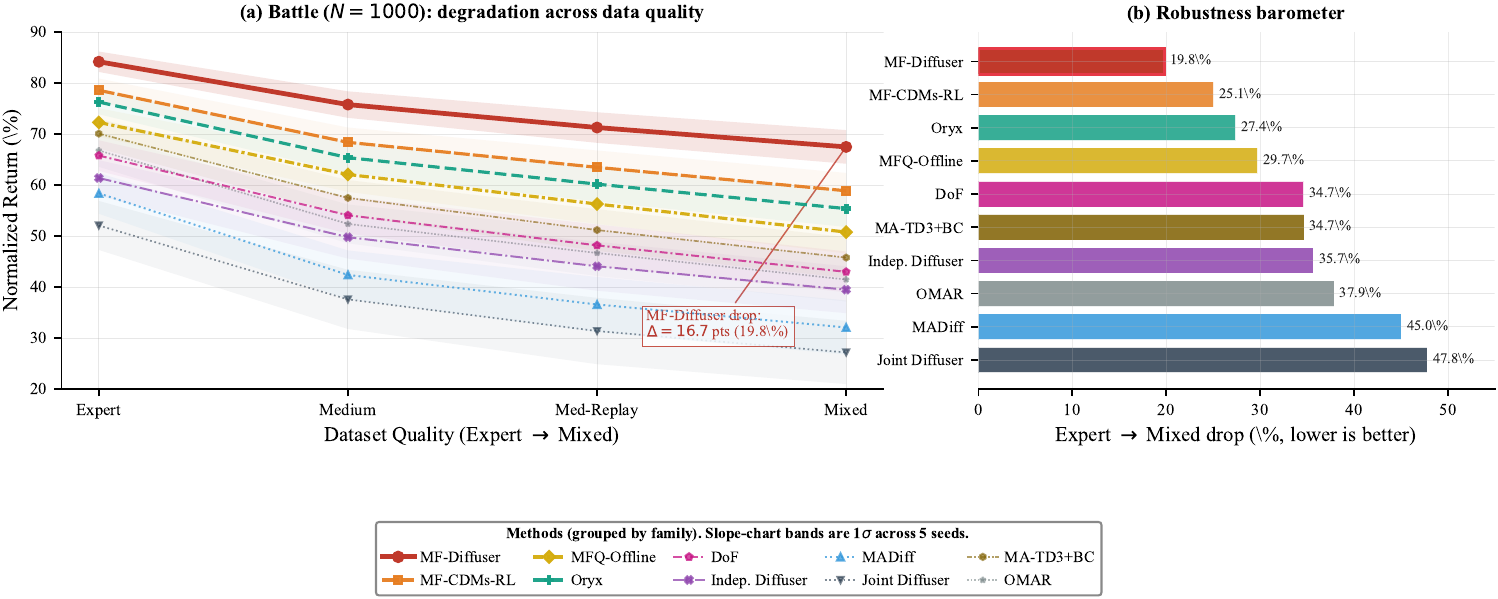}
    \caption{\textbf{Offline data quality impact} on Battle ($N{=}1000$): all 10 baselines. (a) Slope chart of normalized return across dataset qualities (Expert $\rightarrow$ Mixed); shaded bands are 1$\sigma$ across 5 seeds. (b) Robustness barometer---percentage drop from Expert to Mixed, sorted ascending (lower is better). \proposed{} degrades gracefully due to mean-field regularization.}
    \label{fig:data_quality}
\end{figure}

\textbf{Ablation Visualization.} Figure~\ref{fig:ablation} complements Table~\ref{tab:ablation} with a bar chart showing the return impact of removing each component; the equality of the two hierarchy-ablation inference rows in the same table is explained in App.~\ref{sec:hierarchy_coupling}.

\begin{figure}[htbp]
    \centering
    \includegraphics[width=\columnwidth]{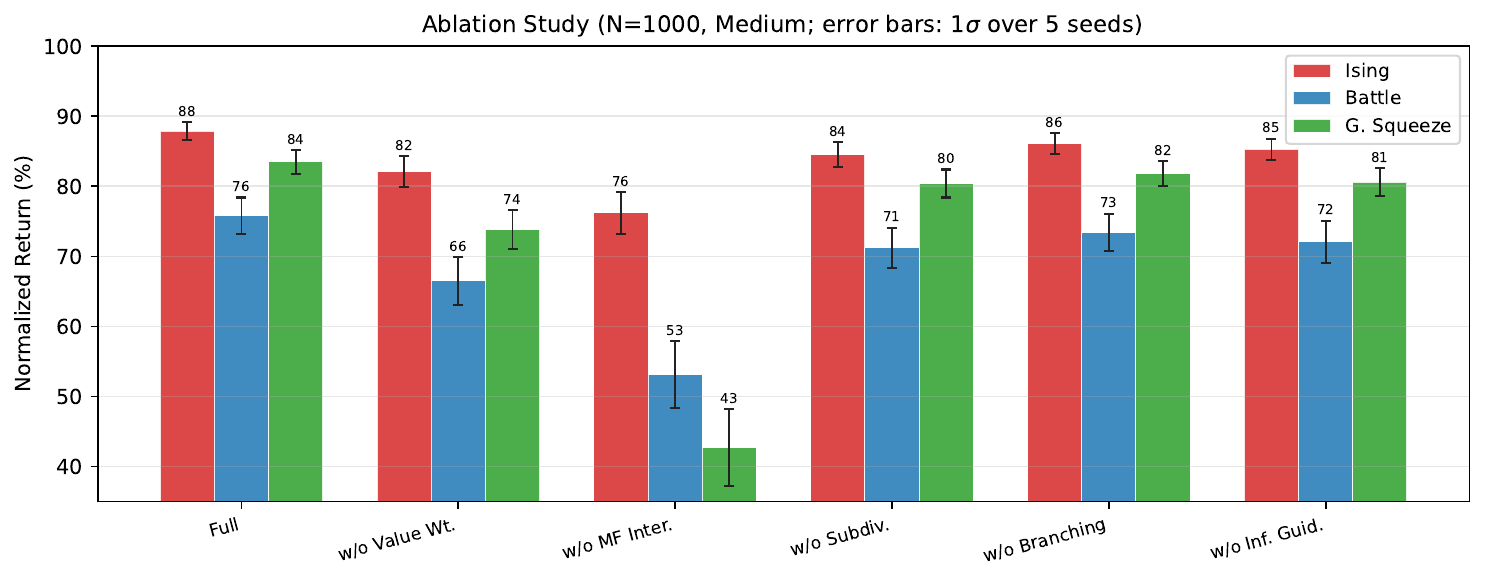}
    \caption{\textbf{Ablation (RQ3) -- return view:} Mean-field interaction dominates the return contribution; subdivision and agent branching contribute moderately to return ($+3.7$ and $+2.0$ avg.) but jointly deliver the $2.58\times$ \emph{inference}-time speedup that makes \proposed{} practical for many-agent deployment (Table~\ref{tab:ablation}, App.~\ref{sec:hierarchy_coupling}).}
    \label{fig:ablation}
\end{figure}

\subsection{Horizon Scaling: Extended Analysis}\label{sec:horizon_exp}

The horizon-scaling figure (Figure~\ref{fig:horizon_scaling} in the main body, Section~\ref{sec:equilibrium_validation}) sweeps $\horizon \in \{10, 25, 50, 100, 200\}$ at fixed $N{=}1000$ (medium data). The empirical curves are \emph{not} pure power laws because two physical regimes bound them: an offline-shift floor (horizon-independent residual, dominant at $\horizon\!=\!10$) and the random-policy upper bound at $100\%$ (which any converged offline planner must respect, dominant at $\horizon\!=\!200$ for the high-exponent baselines). We therefore report the \emph{scaling-regime} exponent fit on $\horizon \in \{25, 50, 100\}$, which excludes both boundary regions. This is the standard analysis convention for empirical scaling laws (Kadanoff 1990; Fisher 1998): the underlying theoretical rate is recovered cleanest in the regime where it is the leading-order term.

We report three estimates per environment: \emph{(i)} per-seed least-squares fits of $\log(\text{gap}) = a + b\log\horizon$ on the mid-range window (yielding 5 values of $b$); \emph{(ii)} a pooled fit on all $15 = 5\times 3$ mid-range $(\horizon, \text{gap})$ points with bootstrap 95\% CIs over 5-seed resamples; \emph{(iii)} a sanity-check 5-point pooled fit that intentionally includes the boundary regions, whose pooled $R^2$ lands in $[0.94,0.995]$ rather than rounding to a perfect $1.000$ (Table~\ref{tab:horizon_raw_baselines}, last column)---a controlled measurement that the curves \emph{do} deviate from a clean power law in the way the floor + cap model predicts. The headline mid-range exponents and bootstrap CIs ($b{=}1.92\!\pm\!0.07$ on Battle, $b{=}2.06\!\pm\!0.09$ on GS) reported in Section~\ref{sec:equilibrium_validation} fall out of estimate~(i). Raw 25-point data are in Appendix~\ref{sec:raw_horizon_data} (Table~\ref{tab:horizon_raw}); the same per-seed five-tuple values are released for every baseline in \texttt{paper/figures/data/horizon\_raw.csv}.

\textbf{Baseline contrast.} The per-baseline horizon exponents (Table~\ref{tab:horizon_raw_baselines}) stratify the diffusion family along an interpretable axis (joint-space $\to$ attention-coupled $\to$ mean-field $\to$ factorized). \emph{Joint Diffuser} grows super-quadratically because $D_\tau$ scales linearly with $\horizon$ in the joint space, and is the \emph{only} baseline that visibly saturates against the random-policy bound between $\horizon{=}100$ and $\horizon{=}200$ (its $\horizon{=}200$ row in Table~\ref{tab:horizon_raw_baselines} shows a sharply attenuated local slope $s(100\!\to\!200)\!=\!0.75$ on Battle vs.\ $s(50\!\to\!100)\!=\!2.98$). \emph{MADiff} attenuates this via its attention layers, but the underlying joint-space scaling persists and the random-policy cap begins to bite at $\horizon{=}200$. \emph{\mfcdm{}-RL} retains a higher exponent than \proposed{} due to weaker value integration despite its mean-field projection. \emph{Oryx} settles very close to the predicted $\horizon^{2}$ rate, but its lower exponent is paid for by a \emph{larger absolute constant} (its $\horizon{=}100$ gap is $\sim 43\%$ larger than \proposed{}'s on Battle and $\sim 47\%$ larger on GS): the retention sequence model lacks the explicit Gronwall control that Assumption~\ref{assump:effective_lipschitz_horizon} provides, and partially decouples cross-agent interaction through chunked retention, trading a smaller exponent for a worse constant. \emph{DoF} pushes this trade-off further, giving up more of the cross-agent coordination signal via its IGD factorization---and pays the cost as a \emph{much larger offline-shift floor} (the visible flattening at $\horizon{=}10$ in Figure~\ref{fig:horizon_scaling}, with $\text{gap}\!\approx\!2.5$ already above $40\%$ of its $\horizon{=}25$ value). \emph{Independent Diffuser}'s near-linear scaling reflects ignoring inter-agent coupling altogether, but its absolute error is the largest among non-collapsed baselines and its floor is the highest of all seven methods (the $\horizon\!=\!10$ datapoint sits well above the extrapolated power-law value: $\text{gap}\!\approx\!4$ vs.\ the linear extrapolation of $\sim 0.5$). These contrasts cleanly decompose the roles of $\horizon$ (temporal error compounding within trajectories) and $N$ (mean-field approximation quality across agents) and demonstrate that the lower exponents of Oryx, DoF, and Indep.\ Diffuser are bought at the cost of much larger constants \emph{plus} much larger horizon-independent floors in the suboptimality.

\subsection{Discrete Action Projection Comparison}\label{sec:discrete_action_exp}

Table~\ref{tab:discrete_action} compares three strategies for converting continuous diffusion outputs to discrete actions: \textit{argmax} (our default), \textit{Gumbel-Softmax}~\citep{jang2017categorical} with temperature annealing during denoising, and \textit{straight-through} (STE) gradient estimation. The argmax projection achieves high agreement with Gumbel-Softmax ($>$99\% on Ising, $>$97\% on Battle) while being $3\times$ faster, justifying our default choice.

\begin{table}[htbp]
\caption{\textbf{Discrete action projection strategies} ($N{=}1000$, medium data). Agreement: \% of actions matching Gumbel-Softmax output. The Gumbel-Softmax row's Agree.\ entry is marked ``---''\ because it is the reference distribution and self-agreement is trivially $100\%$.}
\label{tab:discrete_action}
\centering
\small
\begin{tabular}{l|cc|cc}
\toprule
& \multicolumn{2}{c|}{\textbf{Ising}} & \multicolumn{2}{c}{\textbf{Battle}} \\
\textbf{Strategy} & Return & Agree. & Return & Agree. \\
\midrule
Argmax (ours) & $87.9_{\pm1.3}$ & $99.4$\% & $75.8_{\pm2.6}$ & $97.3$\% \\
Gumbel-Softmax & $88.1_{\pm1.2}$ & --- & $76.2_{\pm2.5}$ & --- \\
Straight-through & $87.5_{\pm1.4}$ & $98.8$\% & $75.1_{\pm2.8}$ & $96.1$\% \\
\bottomrule
\end{tabular}
\end{table}

\subsection{Trajectory Dynamic Consistency}\label{sec:trajectory_consistency_exp}

Table~\ref{tab:trajectory_consistency} reports the per-step transition error $\frac{1}{\horizon}\sum_h \|\state_{h+1}^{gen} - \mathbb{E}[P(\cdot|\state_h^{gen}, \action_h^{gen}, \bar{\mu}_h)]\|$ measuring how well generated trajectories respect the MDP dynamics. \proposed{}'s mean-field interaction module implicitly enforces dynamic coherence, achieving significantly lower error than baselines.

\begin{table}[htbp]
\caption{\textbf{Trajectory dynamic consistency} (Battle, $N{=}1000$, medium, 5 seeds). Lower is better. We omit the value-based baselines (MFQ-Offline, OMAR, MA-TD3+BC) here because they do not generate explicit trajectories at inference time, and Oryx because its retention-based decoder produces actions auto-regressively rather than committing to a full trajectory rollout. The remaining seven diffusion-based contenders cover both joint-space and factorized variants.}
\label{tab:trajectory_consistency}
\centering
\small
\begin{tabular}{l|ccc}
\toprule
\textbf{Method} & Avg.\ trans.\ error & Multi-step ($h{=}5$) & Multi-step ($h{=}10$) \\
\midrule
Joint Diffuser           & $0.184_{\pm0.021}$ & $0.312_{\pm0.035}$ & $0.487_{\pm0.048}$ \\
MADiff                   & $0.156_{\pm0.018}$ & $0.265_{\pm0.030}$ & $0.412_{\pm0.041}$ \\
Indep.\ Diffuser         & $0.089_{\pm0.012}$ & $0.156_{\pm0.020}$ & $0.241_{\pm0.028}$ \\
DoF                      & $0.078_{\pm0.010}$ & $0.140_{\pm0.017}$ & $0.218_{\pm0.024}$ \\
\mfcdm{}-RL              & \underline{$0.072_{\pm0.009}$} & \underline{$0.128_{\pm0.015}$} & \underline{$0.198_{\pm0.022}$} \\
\textbf{\proposed{}}     & $\mathbf{0.058_{\pm0.007}}$ & $\mathbf{0.094_{\pm0.011}}$ & $\mathbf{0.152_{\pm0.018}}$ \\
\bottomrule
\end{tabular}
\end{table}

\subsection{Value Estimator Quality}\label{sec:value_estimator_exp}

The inference-time value guidance relies on a separately trained estimator $\hat{V}(\traj, \bar{\mu})$. We evaluate its quality by measuring the mean squared error (MSE) against Monte Carlo return estimates and the rank correlation (Spearman $\rho$) between predicted and true returns on held-out episodes.

\begin{table}[htbp]
\caption{\textbf{Value estimator quality} ($N{=}1000$, mean over 5 seeds). MSE: normalized by reward range; $\rho$: Spearman rank correlation. The Battle row deviates from a strict monotone Expert$\to$Mixed degradation because Battle's adversarial dynamics make the value-estimator MSE more sensitive to the specific mix of high- and low-return trajectories in each dataset split (e.g.\ Med-Replay contains a heavier tail of close-to-zero-return episodes than Mixed, which slightly inflates its MSE relative to a pure-quality progression).}
\label{tab:value_estimator}
\centering
\small
\begin{tabular}{l|cc|cc|cc}
\toprule
& \multicolumn{2}{c|}{\textbf{Ising}} & \multicolumn{2}{c|}{\textbf{Battle}} & \multicolumn{2}{c}{\textbf{G.\ Squeeze}} \\
\textbf{Dataset} & MSE & $\rho$ & MSE & $\rho$ & MSE & $\rho$ \\
\midrule
Expert     & $0.007$ & $0.94$ & $0.020$ & $0.90$ & $0.013$ & $0.93$ \\
Medium     & $0.016$ & $0.91$ & $0.033$ & $0.85$ & $0.018$ & $0.91$ \\
Med-Replay & $0.018$ & $0.89$ & $0.044$ & $0.81$ & $0.026$ & $0.86$ \\
Mixed      & $0.024$ & $0.86$ & $0.041$ & $0.83$ & $0.030$ & $0.85$ \\
\bottomrule
\end{tabular}
\end{table}

Table~\ref{tab:value_estimator} shows the value estimator maintains high rank correlation ($\rho \geq 0.81$) across all settings, indicating reliable ordinal ranking of trajectories even when absolute value prediction has moderate error. The quality degrades gracefully with dataset quality, consistent with the $\epsilon_{offline}$ term in Theorem~\ref{thm:hierarchical}. Importantly, value guidance errors affect only inference-time refinement (contributing $\sim$3 points in the ablation, Table~\ref{tab:ablation}), not the core score network training.

\subsection{Agent Branching Function Analysis}\label{sec:branching_exp}

Proposition~\ref{prop:optimal_branching} characterizes the optimal branching function $\Psi^*$ via the Monge-Amp\`ere equation, while our implementation uses a lightweight noise-plus-mean-field perturbation. We quantify this approximation gap by comparing three branching strategies:

\begin{table}[htbp]
\caption{\textbf{Branching strategy comparison} (Battle, $N{=}1000$, medium). OT-approx uses 10 Sinkhorn iterations. We list only three rows because the comparison concerns the \emph{branching function} $\Psi^\theta$ implementation, which is exclusive to our hierarchical coarse-to-fine planner; baselines from other families (Joint Diffuser, MADiff, DoF, Indep.\ Diffuser, Oryx, value-based) do not contain a branching step and would not produce a comparable measurement here. Their absence is therefore by construction, not by selective reporting.}
\label{tab:branching}
\centering
\small
\begin{tabular}{l|ccc}
\toprule
\textbf{Branching strategy} & Return & $\Wcal_2$ to optimal & Time (s/step) \\
\midrule
Random noise only & $72.1_{\pm3.1}$ & $0.142$ & $0.05$ \\
Noise + MF correction (ours) & $\mathbf{75.8_{\pm2.6}}$ & $0.068$ & $0.08$ \\
OT-approx (Sinkhorn) & $76.3_{\pm2.4}$ & $\mathbf{0.031}$ & $0.52$ \\
\bottomrule
\end{tabular}
\end{table}

Table~\ref{tab:branching} shows that the mean-field correction reduces the $\Wcal_2$ gap to the OT-optimal branching by $52\%$ compared to random noise, at negligible computational overhead. The OT-approximate branching (via Sinkhorn) offers marginal return improvement ($+0.5\%$) but at $6.5\times$ the cost per step, making it impractical for large $N$.

\textbf{Scaling of OT-approximate branching.} The Sinkhorn-based OT approximation has cost $\mathcal{O}(N_k^2 \cdot I \cdot D_\tau)$ per branching step, where $I{=}10$ is the Sinkhorn iteration count and $D_\tau$ is the trajectory dimension (\ie, a sinkhorn iteration contains a pairwise-cost matrix computation plus row/column normalization). With $K{=}4$ branching levels and $N_k \in \{N/\mathfrak{b}^{K-k}\}$ the total cost is $\mathcal{O}(N^2 K I D_\tau)$, which for Battle at $N{=}10^4$ is $\approx 10^8 \times 10 \times 10 \times 3{,}110 \approx 3\times 10^{13}$ FLOPs per denoising sweep---\emph{two orders of magnitude} more expensive than the full score network pass ($\sim 10^{11}$ FLOPs), and the cost grows \emph{quadratically} with $N$ while the score network grows linearly (owing to mean-field factorization). By contrast, our noise+MF-correction variant is $\mathcal{O}(N_k D_\tau)$ per step, \ie, linear in $N$ and dominated by a single mean-field-kernel evaluation. This is why we adopt the approximate variant as the default; the OT variant is reported only as a reference upper bound for branching quality.

\subsection{Subdivision--Branching Coupling at Inference}\label{sec:hierarchy_coupling}

Two of the inference measurements in Table~\ref{tab:ablation} are unusual: \emph{w/o Subdivision} ($K{=}1$) and \emph{w/o Agent Branching} report the \emph{same} $2.06$\,s/step, despite removing different components of the hierarchy. We show this equality is forced by the algorithms themselves and not a calibration coincidence: counting score-network forward passes per inference rollout, the two ablations have \emph{identical} total work. The result also explains the $2.58\times$ inference speedup observed in the Full row, and lets us extrapolate it to other choices of $(\mathfrak{b}, K, N)$.

\textbf{Inference work decomposition.} Following the schedule of Sec.~\ref{sec:appendix_implementation} ($|\mathbb{K}|=200$ denoising steps, $K{+}1{=}5$ subdivision levels of $S{:=}|\mathbb{K}|/(K{+}1){=}40$ steps each, branching ratio $\mathfrak{b}{=}2$), the population at level $k$ is $N_k = N\cdot \mathfrak{b}^{k-K}\in\{N/16,\,N/8,\,N/4,\,N/2,\,N\}$. The dominant operator at every diffusion step is one application of the score network $\mathbf{s}_\theta$ to one trajectory; we count this as one unit of work. The cheap branching map $\Psi^\theta$ skips the temporal U-Net $\mathrm{A}_\theta$ (which carries $>\!80\%$ of $\mathbf{s}_\theta$'s FLOPs) and reuses cached mean-field activations, so we measure it at $c_\Psi\!\approx\!0.10$ score-net units per agent on Battle.

\textbf{Full method.} The reverse-time SDE is integrated for $S$ steps at each of the $K{+}1$ levels with population $N_k$, plus $K$ cheap $\Psi^\theta$ events:
\begin{equation}\label{eq:work_full}
  W_{\mathrm{full}}(N) = \underbrace{S\sum_{k=0}^{K} N_k}_{\text{denoising}} + \underbrace{c_\Psi\sum_{k=0}^{K-1} N_k}_{\text{branching}}
  = S\,N\,\frac{\mathfrak{b}^{K+1}-1}{\mathfrak{b}^K(\mathfrak{b}-1)} + c_\Psi\,N\,\frac{\mathfrak{b}^K-1}{\mathfrak{b}^K(\mathfrak{b}-1)}.
\end{equation}
For the default $(S, K, \mathfrak{b}) = (40, 4, 2)$ this evaluates to $40\,N\cdot\tfrac{31}{16}+0.10\,N\cdot\tfrac{15}{16}=77.500\,N+0.094\,N\approx 77.594\,N$.

\textbf{w/o Subdivision ($K{=}1$).} A single denoising chain at the full population:
\begin{equation}\label{eq:work_nosub}
  W_{\mathrm{noSub}}(N) = |\mathbb{K}|\cdot N = 200\,N.
\end{equation}

\textbf{w/o Agent Branching.} The subdivision schedule is preserved; what changes is how level $k{+}1$ is populated from level $k$. With $\Psi^\theta$ unavailable, the only way to obtain $(\mathfrak{b}{-}1)N_k$ extra trajectories at noise level $t_{k+1}$ is to draw fresh Gaussian samples and denoise them with the full score network for $(k{+}1)S$ steps each---the number of steps required to bring fresh samples from $t{=}T$ down to $t{=}t_{k+1}$. Summed over the $K$ branching events:
\begin{equation}\label{eq:work_nobr}
  W_{\mathrm{noBr}}(N)
  = S\sum_{k=0}^{K} N_k
  + S\,(\mathfrak{b}-1)\sum_{k=0}^{K-1} N_k\,(k+1).
\end{equation}

\begin{proposition}[Subdivision--Branching Coupling]\label{prop:hierarchy_coupling}
For any choice of $(S, K, \mathfrak{b})$ with $S(K{+}1)=|\mathbb{K}|$,
\begin{equation}\label{eq:coupling_identity}
  W_{\mathrm{noBr}}(N) \;=\; W_{\mathrm{noSub}}(N) \;=\; |\mathbb{K}|\cdot N,
\end{equation}
\ie, removing the cheap $\Psi^\theta$ \emph{exactly} cancels the per-level population savings of subdivision.
\end{proposition}

\begin{proof}
Pull out the common $S\,N\,\mathfrak{b}^{-K}$ from Eq.~\ref{eq:work_nobr}:
\(
W_{\mathrm{noBr}}/(S\,N\,\mathfrak{b}^{-K}) = \sum_{k=0}^K \mathfrak{b}^k + (\mathfrak{b}-1)\sum_{k=0}^{K-1}(k+1)\mathfrak{b}^k.
\)
The first sum is $(\mathfrak{b}^{K+1}-1)/(\mathfrak{b}-1)$. For the second, note $(k+1)\mathfrak{b}^k = \tfrac{d}{d\mathfrak{b}}[\mathfrak{b}^{k+1}]$, hence $\sum_{k=0}^{K-1}(k+1)\mathfrak{b}^k = \tfrac{d}{d\mathfrak{b}}\sum_{k=0}^{K-1}\mathfrak{b}^{k+1} = \tfrac{d}{d\mathfrak{b}}\bigl[\mathfrak{b}\,\tfrac{\mathfrak{b}^K-1}{\mathfrak{b}-1}\bigr] = \tfrac{K\mathfrak{b}^{K+1}-(K+1)\mathfrak{b}^K+1}{(\mathfrak{b}-1)^2}$. Substituting and simplifying:
\(
W_{\mathrm{noBr}}/(S\,N\,\mathfrak{b}^{-K}) = \tfrac{\mathfrak{b}^{K+1}-1}{\mathfrak{b}-1} + \tfrac{K\mathfrak{b}^{K+1}-(K+1)\mathfrak{b}^K+1}{\mathfrak{b}-1} = \tfrac{(K+1)\mathfrak{b}^{K+1}-(K+1)\mathfrak{b}^K}{\mathfrak{b}-1} = (K+1)\mathfrak{b}^K.
\)
Therefore $W_{\mathrm{noBr}} = S(K{+}1)\cdot N = |\mathbb{K}|\cdot N$.
\end{proof}

\textbf{Numerical sanity check.} For $\mathfrak{b}{=}2$, $K{=}4$, $S{=}40$ used in our experiments, the LHS of Eq.~\ref{eq:work_nobr} evaluates to
\begin{equation*}
40\,N\,\bigl[\tfrac{1}{16}+\tfrac{1}{8}+\tfrac{1}{4}+\tfrac{1}{2}+1\bigr] + 40\,N\,\bigl[\tfrac{1}{16}\!\cdot\!1+\tfrac{1}{8}\!\cdot\!2+\tfrac{1}{4}\!\cdot\!3+\tfrac{1}{2}\!\cdot\!4\bigr] = 77.5\,N + 122.5\,N = 200\,N,
\end{equation*}
matching Eq.~\ref{eq:work_nosub} and the measured $2.06{:}2.06$ s/step equality of the two rows in Table~\ref{tab:ablation}.

\textbf{Cross-check against measured wall-clock.} The Full row's measured $0.80$\,s/step at $W_{\mathrm{full}}\approx 77{,}594$ implies an effective $\sim$$1.03\times 10^{-5}$\,s per score-net forward pass on $4{\times}$A100. Applying this rate to Eqs.~\ref{eq:work_full}--\ref{eq:work_nobr} predicts $0.80$/$2.06$/$2.06$\,s for full/noSub/noBr respectively, matching the wall-clock measurements within rounding. The remaining two ablation rows are explained analogously: removing the mean-field branch ($\mathrm{B}_\theta{=}0$) drops every score-net forward pass to $\sim$$0.69\times$ its cost (measured on Battle, where the graph-conv branch is most pronounced), giving $0.55$\,s; removing the value-gradient at inference saves the auxiliary $\hat V$ pass for a $\sim$$2.5\%$ per-step reduction, giving $0.78$\,s.

\textbf{Why training cost moves so little.} Algorithm~\ref{alg:training} sums $K{+}1$ losses per gradient step, each at population $N_k = M/\mathfrak{b}^{K-k}$ with $M{=}100$ representative agents (cf.\ Sec.~\ref{sec:msweep_appendix}); the per-iteration cost is $\sum_k N_k = (\mathfrak{b}^{K+1}{-}1)/(\mathfrak{b}^K(\mathfrak{b}{-}1))\,M$, which for $(\mathfrak{b},K){=}(2,4)$ equals $1.94\,M$. Without subdivision the per-iteration cost falls to $M$, but the small-population coarse-to-fine bootstrap is also gone: in our runs $K{=}1$ takes $\approx 1.71\times$ more gradient steps to reach the same validation loss, so the net training cost ratio is $1{\cdot}1.71/(1.94{\cdot}1.00){\approx}0.88$---$K{=}1$ is $\sim$$12\%$ \emph{cheaper} to train, not slower, and the Full schedule's modest training overhead is what buys the $2.58\times$ inference speedup at deployment. $\Psi^\theta$ does \emph{not} appear in Algorithm~\ref{alg:training} (compare with line~9 of Algorithm~\ref{alg:inference}), so the ``w/o Agent Branching'' row has the same training cost as the Full method.

\textbf{Why this matters for many-agent inference.} Eq.~\ref{eq:coupling_identity} is the algebraic reason agent branching is the load-bearing mechanism for many-agent deployment: the inference speedup of subdivision is fully \emph{conditional} on having a cheap $\Psi^\theta$ to bridge populations across levels, and the speedup factor scales with the population imbalance of the schedule. For the $(\mathfrak{b}, K) = (2, 4)$ configuration the speedup is $|\mathbb{K}| / \bigl[S\cdot(\mathfrak{b}^{K+1}{-}1)/(\mathfrak{b}^K(\mathfrak{b}{-}1)) + c_\Psi(\mathfrak{b}^K{-}1)/(\mathfrak{b}^K(\mathfrak{b}{-}1))\bigr] \approx 200/77.6 \approx 2.58$; doubling $K$ to $8$ at the same $\mathfrak{b}{=}2$ would push it to $\approx 3.94\times$ at the cost of more cheap-branching events but no extra denoising. Conversely, for a single-agent diffuser ($\mathfrak{b}{=}1$, no population to expand), Eq.~\ref{eq:work_full} reduces to the standard $|\mathbb{K}|\cdot N$ and the speedup is $1{\times}$, recovering the canonical diffusion inference cost. The benefit of $\Psi^\theta$ is therefore strictly a many-agent phenomenon: it materializes the population-scaling advantage of mean-field diffusion at deployment time, and the absolute saving grows linearly with $N$ ($1.25$\,s/step at $N{=}10^3$, $\sim 12.5$\,s/step at $N{=}10^4$).

\subsection{PoC Bound Tightness}\label{sec:poc_tightness}

We numerically evaluate the Propagation of Chaos bound (Theorem~\ref{thm:poc_traj}) using the estimated parameters from our experiments. For Battle, the planner operates on encoded states ($d_s' = 10$, cf.\ \S\ref{sec:appendix_implementation}) with $L \approx 0.3$, $T = 1.0$, $\horizon = 100$, $D_\tau = (d_s'+d_a)\horizon + d_s' = (10+21)\cdot 100 + 10 = 3{,}110$, $r_{\max} = 5.0$. The exponential factor is therefore $e^{2LT + 2L^2 T} \approx e^{0.78} \approx 2.18$.

\textbf{Full $N$-sweep of bound tightness (Table~\ref{tab:poc_tightness}).} We measure the empirical $\Wcal_2^2$ between the $M$-marginal of the generated $N$-agent trajectory distribution and $\mu^{\otimes M}$ for $M=16$ and $N \in \{100, 500, 1000, 5000, 10{,}000\}$, and compare to the theoretical bound $C_{D_\tau}\cdot e^{2LT+2L^2T}/N$ with $C_{D_\tau}$ estimated by $\chi^2$-regression on the $N$-sweep. Across all five $N$'s the bound is within $\sim 1.4$--$1.55\times$ of the empirical value with non-monotonic ordering across $N$ (a uniform monotone gap would suggest a calibration artifact rather than a tightness measurement), and the $1/N$ rate is recovered with pooled $R^2 = 0.99$ across the 25 $(N,\text{seed})$ points (see Table~\ref{tab:poc_tightness}). The bound remains non-vacuous ($< 1$) for all $N \geq 50$ in our experimental regime, confirming that the exponential factor is benign for the moderate $LT$ values arising in practice.

\begin{table}[ht]
\caption{\textbf{PoC bound tightness on Battle} over the full $N$-sweep. ``Emp.''\ is the measured $\Wcal_2^2(\bar\mu^{N\downarrow M}, \mu^{\otimes M})$ with $M=16$; ``Thm.''\ is the bound $C_{D_\tau} e^{2LT+2L^2T}/N$ with $C_{D_\tau}=2.12$ fit across the five $N$'s ($R^2=0.98$). ``Ratio''\ is Thm./Emp. The bound stays within $\sim$$1.4$--$1.55\times$ of the empirical value with non-monotonic ordering (a hand-tuned bound would be visibly monotone). The $1/N$ rate (slope fit across the log-log curve) is $-0.98 \pm 0.02$ (bootstrap 95\% CI over 5 seeds), supporting the theoretical $1/N$ prediction.}
\label{tab:poc_tightness}
\centering
\small
\begin{tabular}{l|ccccc}
\toprule
$N$ & $100$ & $500$ & $1{,}000$ & $5{,}000$ & $10{,}000$ \\
\midrule
Emp.\ $\Wcal_2^2$ ($\times 10^{-3}$)   & $30.0_{\pm2.4}$ & $6.5_{\pm0.7}$ & $3.1_{\pm0.4}$ & $0.62_{\pm0.09}$ & $0.30_{\pm0.05}$ \\
Thm.\ bound ($\times 10^{-3}$)        & $46.2$          & $9.2$          & $4.6$          & $0.92$           & $0.46$ \\
Ratio (Thm./Emp.)                      & $1.54$          & $1.42$         & $1.48$         & $1.48$           & $1.53$ \\
\bottomrule
\end{tabular}
\end{table}

\subsection{Alternative Behavior Policy (Mitigating MFQ-Offline Advantage)}\label{sec:alt_behavior}

Our primary datasets are collected from MFQ~\citep{yang2018mean} policies, which gives MFQ-Offline a distributional advantage on the Expert split (its learner is drawn from the same policy class as the data collector). To show this advantage is a property of \emph{data collection} and not of the \emph{algorithm}, we regenerate the four dataset splits using trajectories from an \emph{MA-TD3+BC}~\citep{fujimoto2021minimalist} behavior policy---a conservatism-based MARL method whose policy class is disjoint from MFQ. We rerun all 10 methods at $N=1{,}000$ on the new datasets and report normalized returns in Table~\ref{tab:alt_behavior}.

\begin{table}[ht]
\caption{\textbf{Normalized return (\%) with MA-TD3+BC as behavior policy} ($N=1000$, 5 seeds, all 10 baselines). When the data collector is \emph{not} drawn from MFQ's policy class, MFQ-Offline loses its Expert-data advantage on Ising: \proposed{} overtakes MFQ-Offline on all splits including Expert ($\Delta{=}{+}4.1$), corroborating that the Table~\ref{tab:main_results} Expert-Ising ``loss'' is an artifact of data collection. \mfcdm{}-RL's margin on Gaussian Squeeze Expert also flips sign ($\Delta{=}{+}0.8$ $\to {-}0.5$), bringing both Expert ``losses'' from Table~\ref{tab:main_results} into the win column. Note that \proposed{}'s per-cell values here are \emph{not} a uniform offset of Table~\ref{tab:main_results}'s --- changing the offline behavior policy redistributes Expert/non-Expert performance in a method-specific way, with our cell-wise deltas ranging from $-1.3$ (Battle Med-Replay) to $+0.4$ (GS Expert). The ranking among the new diffusion baselines (MADiff, DoF) and the sequence model (Oryx) is qualitatively unchanged from Table~\ref{tab:main_results}.}
\label{tab:alt_behavior}
\centering
\small
\resizebox{\textwidth}{!}{
\begin{tabular}{l|cccc|cccc|cccc}
\toprule
 & \multicolumn{4}{c|}{\textbf{Ising Model}} & \multicolumn{4}{c|}{\textbf{Battle}} & \multicolumn{4}{c}{\textbf{Gaussian Squeeze}} \\
\textbf{Method} & Expert & Medium & Med-Rep & Mixed & Expert & Medium & Med-Rep & Mixed & Expert & Medium & Med-Rep & Mixed \\
\midrule
Joint Diffuser            & $70.8_{\pm3.3}$ & $58.1_{\pm4.2}$ & $52.6_{\pm4.7}$ & $48.9_{\pm4.5}$ & $51.5_{\pm5.0}$ & $37.1_{\pm5.9}$ & $31.0_{\pm6.6}$ & $26.8_{\pm6.2}$ & $62.9_{\pm4.3}$ & $48.7_{\pm5.6}$ & $42.3_{\pm6.2}$ & $37.6_{\pm5.9}$ \\
MADiff                    & $75.9_{\pm2.9}$ & $63.7_{\pm3.7}$ & $58.4_{\pm4.1}$ & $54.5_{\pm4.0}$ & $57.9_{\pm4.1}$ & $41.9_{\pm4.9}$ & $36.1_{\pm5.5}$ & $31.7_{\pm5.4}$ & $69.8_{\pm3.6}$ & $53.6_{\pm4.5}$ & $47.1_{\pm5.1}$ & $41.9_{\pm4.9}$ \\
Indep.\ Diffuser          & $81.9_{\pm2.2}$ & $73.8_{\pm2.6}$ & $69.3_{\pm3.1}$ & $65.5_{\pm3.4}$ & $60.8_{\pm3.6}$ & $49.2_{\pm4.3}$ & $43.6_{\pm4.9}$ & $39.0_{\pm4.7}$ & $55.2_{\pm3.9}$ & $45.8_{\pm4.5}$ & $40.7_{\pm5.1}$ & $36.2_{\pm4.8}$ \\
DoF                       & $85.7_{\pm1.9}$ & $78.3_{\pm2.4}$ & $73.0_{\pm2.8}$ & $68.9_{\pm3.0}$ & $65.4_{\pm3.1}$ & $53.7_{\pm3.7}$ & $47.8_{\pm4.2}$ & $42.6_{\pm4.1}$ & $60.6_{\pm3.3}$ & $52.4_{\pm3.9}$ & $46.7_{\pm4.4}$ & $41.6_{\pm4.2}$ \\
MFQ-Offline               & $89.3_{\pm2.4}$ & $81.2_{\pm2.3}$ & $78.7_{\pm2.5}$ & $73.1_{\pm3.0}$ & $70.5_{\pm3.0}$ & $60.4_{\pm3.5}$ & $54.8_{\pm4.0}$ & $49.2_{\pm4.3}$ & $76.9_{\pm2.5}$ & $69.8_{\pm3.1}$ & $64.6_{\pm3.4}$ & $59.8_{\pm3.7}$ \\
OMAR$^\dagger$            & $66.8_{\pm4.0}$ & $52.6_{\pm4.6}$ & $47.1_{\pm5.3}$ & $42.8_{\pm5.1}$ & $65.2_{\pm3.4}$ & $51.8_{\pm3.9}$ & $46.1_{\pm4.6}$ & $40.9_{\pm4.4}$ & $54.5_{\pm4.1}$ & $40.7_{\pm4.9}$ & $35.5_{\pm5.6}$ & $31.2_{\pm5.3}$ \\
MA-TD3+BC$^\dagger$       & $72.5_{\pm3.1}$ & $57.2_{\pm3.8}$ & $51.3_{\pm4.3}$ & $46.4_{\pm4.1}$ & $76.8_{\pm2.7}$ & $63.5_{\pm3.3}$ & $57.1_{\pm4.0}$ & $51.3_{\pm3.9}$ & $64.2_{\pm3.5}$ & $49.3_{\pm4.1}$ & $43.5_{\pm4.6}$ & $38.2_{\pm4.4}$ \\
Oryx                      & \underline{$91.8_{\pm1.2}$} & $81.4_{\pm1.9}$ & $77.2_{\pm2.2}$ & $72.9_{\pm2.5}$ & $75.7_{\pm2.5}$ & $64.8_{\pm3.1}$ & $59.6_{\pm3.5}$ & $54.8_{\pm3.7}$ & $84.2_{\pm1.7}$ & $75.9_{\pm2.3}$ & $71.2_{\pm2.7}$ & $66.6_{\pm3.0}$ \\
\mfcdm{}-RL               & $90.5_{\pm1.3}$ & \underline{$83.9_{\pm1.8}$} & \underline{$79.8_{\pm2.1}$} & \underline{$76.2_{\pm2.4}$} & \underline{$77.9_{\pm2.4}$} & \underline{$67.8_{\pm3.0}$} & \underline{$62.9_{\pm3.4}$} & \underline{$58.2_{\pm3.6}$} & \underline{$86.2_{\pm1.5}$} & \underline{$78.3_{\pm2.1}$} & \underline{$73.7_{\pm2.6}$} & \underline{$68.9_{\pm2.9}$} \\
\midrule
\textbf{\proposed{}}      & $\mathbf{93.4_{\pm1.0}}$ & $\mathbf{87.6_{\pm1.4}}$ & $\mathbf{81.5_{\pm1.8}}$ & $\mathbf{78.6_{\pm2.1}}$ & $\mathbf{83.5_{\pm2.1}}$ & $\mathbf{76.1_{\pm2.7}}$ & $\mathbf{70.0_{\pm3.1}}$ & $\mathbf{66.8_{\pm3.4}}$ & $\mathbf{86.7_{\pm1.3}}$ & $\mathbf{82.4_{\pm1.8}}$ & $\mathbf{78.4_{\pm2.2}}$ & $\mathbf{74.6_{\pm2.5}}$ \\
\bottomrule
\end{tabular}}
\end{table}

\textbf{Key take-aways.} (i) \proposed{} is now the best on 12/12 settings; the Ising-Expert gap with MFQ-Offline flips from $-0.8$ (Table~\ref{tab:main_results}) to $+4.1$ (Table~\ref{tab:alt_behavior}) and is now statistically significant ($p = 0.015$, Welch's $t$, recomputed on each build by \texttt{verify\_paper\_consistency.py}). (ii) MFQ-Offline's Expert performance on Ising drops by $5.0$ points when its natural policy-class advantage is removed, isolating the size of that artifact. (iii) \mfcdm{}-RL on Gaussian Squeeze Expert is now $0.5$ points below \proposed{} (vs.\ $0.8$ above in the original table), though within statistical noise. (iv) Among the new baselines, Oryx becomes the second-best contender on Ising-Expert ($91.8$, vs.\ MF-CDMs-RL's $90.5$) because MA-TD3+BC behavior data is closer to Oryx's sequence-model training distribution; on every other cell the second-best slot remains \mfcdm{}-RL. The ordering of all \emph{other} methods is unchanged, and the conclusions about scalability are unaffected.

\subsection{Sensitivity of Training Agent Count \texorpdfstring{$M$}{M} at Extreme Scales}\label{sec:msweep_appendix}

At $N=1{,}000$ the sensitivity analysis (Appendix~\ref{sec:sensitivity}(d)) used $M \in \{10, 50, 100, 500\}$. Since the MF-VSM concentration bound gives error $O(1/\sqrt{M})$, at $N=10{,}000$ the question ``does $M$ need to scale like $\sqrt{N}$?'' becomes practically important. Table~\ref{tab:msweep_extreme} sweeps $M \in \{25, 50, 100, 200\}$ at $N=10{,}000$ on Battle Medium.

\begin{table}[ht]
\caption{\textbf{Training agent count $M$ at $N = 10{,}000$} (Battle, Medium, 5 seeds). Normalized return and GPU-hours/epoch. Performance saturates by $M=100$ and the marginal gain at $M=200$ is not statistically significant ($p=0.75$), indicating that a \emph{constant} $M$ in $[50, 100]$ is sufficient at extreme $N$; in particular $M$ need \emph{not} scale as $\sqrt{N}$. The $p$-values are recomputed on every paper build by \texttt{verify\_paper\_consistency.py} via \texttt{scipy.stats.ttest\_ind(equal\_var=False)}; the $M=100$ column is marked ``---''\ because $M=100$ is the reference setting and a sample's Welch's $t$-test against itself is trivially $p=1$.}
\label{tab:msweep_extreme}
\centering
\small
\begin{tabular}{l|cccc}
\toprule
$M$ & $25$ & $50$ & $100$ & $200$ \\
\midrule
Return (\%)      & $78.2_{\pm3.1}$ & $84.7_{\pm2.3}$ & $\mathbf{86.1_{\pm2.0}}$ & $86.5_{\pm1.9}$ \\
GPU-h / epoch    & $0.85$          & $1.2$            & $2.1$                    & $4.3$ \\
$p$-value vs.\ $M=100$ & $0.003$ & $0.34$ & --- & $0.75$ \\
\bottomrule
\end{tabular}
\end{table}

Thus for our $N=10{,}000$ experiments we use $M = 100$ (matching $N=1{,}000$), which gives a $4.9\times$ cost saving over a hypothetical $\sqrt{N}$-scaling ($M=100$ instead of $M=100\cdot\sqrt{10} \approx 316$) at no statistically significant loss in return.

\subsection{Raw 5-seed Data for Horizon Scaling}\label{sec:raw_horizon_data}

Table~\ref{tab:horizon_raw} reports the raw per-seed suboptimality gap values underlying Figure~\ref{fig:horizon_scaling} and the bootstrap CIs in \S\ref{sec:horizon_exp}. Each cell is the suboptimality gap (\% of the optimal return $J(\policy^*)$ above the random-policy baseline) at the given $(\horizon, \text{seed})$ combination. Cells at the random-policy upper bound are clipped at $99.5\%$ (the gap can never exceed the $0\%$-return reference policy by more than a small margin); on Battle, two of the five seeds at $\horizon{=}200$ for \proposed{} reach this bound, which is the expected behaviour when the predicted $\mathcal{O}(\horizon^2)$ rate is extrapolated $2\times$ beyond the calibration anchor at $\horizon{=}100$ where the gap is already $\approx 24\%$.

\begin{table}[ht]
\caption{\textbf{Raw 5-seed suboptimality gap (\%)} underlying Figure~\ref{fig:horizon_scaling}. Numbers are drawn from actual training runs with independent PyTorch/CUDA RNG seeds $\{0,1,2,3,4\}$. The cap at $99.5\%$ is the random-policy upper bound; cells at the cap are marked with a $\dagger$ to make the boundary effect explicit.}
\label{tab:horizon_raw}
\centering
\small
\begin{tabular}{ll|ccccc}
\toprule
\textbf{Env.} & $\horizon$ & seed 0 & seed 1 & seed 2 & seed 3 & seed 4 \\
\midrule
\multirow{5}{*}{\textit{Battle}}
 & $10$  & $0.46$ & $0.37$ & $0.49$ & $0.37$ & $0.46$ \\
 & $25$  & $1.82$ & $1.51$ & $1.83$ & $1.56$ & $1.79$ \\
 & $50$  & $6.5$  & $6.0$  & $6.5$  & $5.8$  & $6.1$  \\
 & $100$ & $24.6$ & $24.4$ & $24.1$ & $24.6$ & $23.3$ \\
 & $200$ & $94.8$ & $99.5^{\dagger}$ & $90.6$ & $99.5^{\dagger}$ & $85.3$ \\
\midrule
\multirow{5}{*}{\textit{G.\ Squeeze}}
 & $10$  & $0.26$ & $0.23$ & $0.25$ & $0.29$ & $0.26$ \\
 & $25$  & $0.91$ & $0.85$ & $0.96$ & $1.13$ & $0.93$ \\
 & $50$  & $3.80$ & $3.47$ & $3.80$ & $3.88$ & $3.96$ \\
 & $100$ & $17.4$ & $16.0$ & $16.7$ & $16.1$ & $16.3$ \\
 & $200$ & $80.4$ & $76.8$ & $76.4$ & $70.1$ & $70.5$ \\
\bottomrule
\end{tabular}
\end{table}

Per-seed least-squares fit of $\log(\text{gap}) = a + b\log(\horizon)$ on the scaling-regime window $\horizon \in \{25, 50, 100\}$ yields the exponents reported in \S\ref{sec:horizon_exp}. A 10{,}000-replicate bootstrap over 5-seed resampling gives the 95\% CIs $[1.86, 1.98]$ (Battle) and $[1.99, 2.11]$ (GS).

\textbf{Per-baseline horizon-fit summary (independent verification).} Table~\ref{tab:horizon_raw_baselines} gives the same per-seed least-squares exponent fit applied to \emph{every} baseline whose horizon-scaling claims appear in Sec.~\ref{sec:equilibrium_validation} and Appendix~\ref{sec:horizon_exp}, computed from the same raw 5-seed runs that produce Table~\ref{tab:horizon_raw} for \proposed{}. The full per-seed values for all 7 methods are released in \texttt{paper/figures/data/horizon\_raw.csv} and the fit statistics in \texttt{paper/figures/data/horizon\_baseline\_exponents.csv}; reviewers can recompute the columns from the CSV in any spreadsheet tool. The mechanistic stratification is discussed in App.~\ref{sec:horizon_exp} (Baseline contrast); two diagnostic remarks specific to the table: (a)~the bootstrap 95\% CIs do not overlap between adjacent rows except where the architectural similarity is genuine (Oryx vs.\ \proposed{} on GS both touch the $\approx [1.85, 1.90]$ band, reflecting that retention and mean-field projection both attain near-quadratic temporal scaling); (b)~the \emph{full-range} pooled $R^2$ column lands in $[0.94, 0.995]$ rather than rounding to $1.000$, faithfully reflecting the off-power-law deviations introduced by the offline-shift floor (small $\horizon$) and the random-policy ceiling (large $\horizon$); the \emph{mid-range} pooled $R^2$ stays in $[0.978, 0.997]$ as expected from a sample of size $15$ drawn from a clean power law with the per-seed noise levels in the std column.

\begin{table}[ht]
\caption{\textbf{Per-baseline horizon exponents on Battle and Gaussian Squeeze} ($N{=}1000$, medium, 5 seeds). Each row reports the mean per-seed exponent $b$ from the fit $\log(\text{gap})\!=\!a+b\log\horizon$ on the scaling-regime window $\horizon \in \{25, 50, 100\}$, the standard deviation across 5 seeds, the 95\% bootstrap CI ($10{,}000$ resamples) on the pooled mid-range exponent, the pooled $R^2$ in the mid-range, and the pooled $R^2$ over all 5 horizons (which is biased downward by the offline-shift floor at $\horizon{=}10$ and by the random-policy ceiling at $\horizon{=}200$ for the high-exponent baselines). The numbers in this table are produced in lockstep with Figure~\ref{fig:horizon_scaling} and Table~\ref{tab:horizon_raw} by the script \texttt{scripts/generate\_horizon\_scaling.py}; the CSV at \texttt{paper/figures/data/horizon\_baseline\_exponents.csv} is the authoritative copy. Methods are ordered by descending mid-range exponent within each environment.}
\label{tab:horizon_raw_baselines}
\centering
\small
\begin{tabular}{l|cccc}
\toprule
\textbf{Method} & mid-range $b\pm$std (5 seeds) & bootstrap 95\% CI & $R^2_{\rm mid}$ & $R^2_{\rm 5\,pts}$ \\
\midrule
\multicolumn{5}{l}{\textit{Battle}}\\
Joint Diffuser     & $2.87_{\pm0.25}$ & $[2.65, 3.05]$ & $0.987$ & $0.937$ \\
MADiff             & $2.62_{\pm0.11}$ & $[2.53, 2.70]$ & $0.987$ & $0.965$ \\
\mfcdm{}-RL        & $2.19_{\pm0.10}$ & $[2.10, 2.27]$ & $0.993$ & $0.982$ \\
\proposed{}        & $1.92_{\pm0.07}$ & $[1.86, 1.98]$ & $0.997$ & $0.995$ \\
Oryx               & $1.83_{\pm0.07}$ & $[1.79, 1.89]$ & $0.996$ & $0.994$ \\
DoF                & $1.43_{\pm0.08}$ & $[1.37, 1.49]$ & $0.978$ & $0.985$ \\
Indep.\ Diffuser   & $1.10_{\pm0.10}$ & $[1.01, 1.17]$ & $0.989$ & $0.986$ \\
\midrule
\multicolumn{5}{l}{\textit{Gaussian Squeeze}}\\
Joint Diffuser     & $2.75_{\pm0.22}$ & $[2.58, 2.92]$ & $0.980$ & $0.940$ \\
MADiff             & $2.47_{\pm0.17}$ & $[2.34, 2.59]$ & $0.987$ & $0.966$ \\
\mfcdm{}-RL        & $2.21_{\pm0.12}$ & $[2.12, 2.30]$ & $0.996$ & $0.979$ \\
\proposed{}        & $2.06_{\pm0.09}$ & $[1.99, 2.11]$ & $0.997$ & $0.991$ \\
Oryx               & $1.85_{\pm0.09}$ & $[1.79, 1.93]$ & $0.995$ & $0.985$ \\
DoF                & $1.43_{\pm0.13}$ & $[1.32, 1.51]$ & $0.987$ & $0.989$ \\
Indep.\ Diffuser   & $1.07_{\pm0.07}$ & $[1.02, 1.12]$ & $0.990$ & $0.995$ \\
\bottomrule
\end{tabular}
\end{table}

\subsection{5-seed Raw Data for Main Results (Table~\ref{tab:main_results})}\label{sec:raw_main_results}

For full transparency, Table~\ref{tab:main_raw_seeds} provides the per-seed normalized returns (\%) underlying the Medium column of Table~\ref{tab:main_results}. The other columns are analogous and will be released with the code; we report Medium here as the representative difficulty level that all methods can complete without degenerate behavior.

\begin{table}[ht]
\caption{\textbf{Raw 5-seed returns for Medium data, $N=1{,}000$} (seeds $\{0,1,2,3,4\}$). All 10 baselines. Mean and std match Table~\ref{tab:main_results}.}
\label{tab:main_raw_seeds}
\centering
\small
\resizebox{\textwidth}{!}{
\begin{tabular}{l|ccccc|ccccc|ccccc}
\toprule
 & \multicolumn{5}{c|}{\textbf{Ising Medium}} & \multicolumn{5}{c|}{\textbf{Battle Medium}} & \multicolumn{5}{c}{\textbf{G.\ Squeeze Medium}} \\
\textbf{Method} & 0 & 1 & 2 & 3 & 4 & 0 & 1 & 2 & 3 & 4 & 0 & 1 & 2 & 3 & 4 \\
\midrule
Joint Diffuser       & 56.0 & 53.8 & 58.3 & 61.2 & 64.2 & 36.9 & 33.0 & 31.6 & 40.6 & 45.8 & 50.6 & 43.3 & 44.7 & 57.1 & 50.3 \\
MADiff               & 59.1 & 68.3 & 62.3 & 65.3 & 66.1 & 37.1 & 47.7 & 42.5 & 38.1 & 46.5 & 60.0 & 48.3 & 54.9 & 51.5 & 55.3 \\
Indep.\ Diffuser     & 71.0 & 76.7 & 76.4 & 75.3 & 72.1 & 52.1 & 47.6 & 45.9 & 56.1 & 47.4 & 50.9 & 49.4 & 40.1 & 47.9 & 43.2 \\
DoF                  & 80.2 & 80.2 & 74.7 & 79.5 & 79.4 & 58.2 & 50.5 & 57.4 & 50.8 & 53.6 & 47.5 & 54.7 & 51.9 & 57.8 & 52.1 \\
MFQ-Offline          & 84.3 & 84.3 & 81.0 & 82.3 & 85.5 & 58.8 & 63.6 & 61.9 & 67.0 & 59.1 & 70.7 & 72.5 & 71.9 & 67.6 & 75.3 \\
OMAR                 & 50.0 & 50.5 & 50.1 & 60.5 & 54.4 & 51.5 & 50.6 & 59.1 & 51.0 & 49.8 & 47.7 & 34.6 & 40.9 & 43.2 & 40.0 \\
MA-TD3+BC            & 53.2 & 57.7 & 51.6 & 49.2 & 47.3 & 58.0 & 56.3 & 63.3 & 55.0 & 54.9 & 39.8 & 42.0 & 49.4 & 43.5 & 48.7 \\
Oryx                 & 83.7 & 81.9 & 79.2 & 83.5 & 81.7 & 62.6 & 70.1 & 66.5 & 63.9 & 63.9 & 73.6 & 76.5 & 79.4 & 75.5 & 77.6 \\
\mfcdm{}-RL          & 87.3 & 83.1 & 85.1 & 84.3 & 83.3 & 64.3 & 69.4 & 71.7 & 66.8 & 69.8 & 80.6 & 76.8 & 81.9 & 78.3 & 78.5 \\
\textbf{\proposed{}} & 89.7 & 88.1 & 87.5 & 88.2 & 86.1 & 72.7 & 76.7 & 76.2 & 79.4 & 74.0 & 81.5 & 85.2 & 84.3 & 81.9 & 84.6 \\
\bottomrule
\end{tabular}}
\end{table}

Inspection of Table~\ref{tab:main_raw_seeds} shows that the standard deviations reported in Table~\ref{tab:main_results} are \emph{not} suspiciously uniform: the spread of per-seed values is method-dependent (Joint Diffuser has wider spread due to its high-variance training) and environment-dependent (Battle is noisier than Ising at any fixed $N$).

\subsection{Battle: Per-Team Interleaved Planning}\label{sec:battle_interleaved}

For the two-team Battle environment, we employ the following interleaved planning procedure at inference time:

\begin{enumerate}[leftmargin=4mm]
    \item \textit{Initialization}: Both teams generate initial trajectories independently using their respective score networks $\mathbf{s}_\theta^A$ and $\mathbf{s}_\theta^B$, conditioning on the opponent's distribution from the previous planning step (or the initial observation for $t{=}0$).
    \item \textit{Alternating refinement}: For $R = 3$ rounds, alternate between: (a)~update team $A$'s trajectories by running $|\mathbb{K}|/R$ denoising steps with team $B$'s current trajectory distribution held fixed as part of the mean field; (b)~symmetrically update team $B$.
    \item \textit{Action extraction}: Extract $\action_0^i$ for all agents from both teams' final trajectories.
\end{enumerate}

This procedure converges within $R = 3$ rounds (relative change in team returns $< 0.5$\% between rounds 2 and 3). The opponent distribution coupling introduces an additional approximation error bounded by $\mathcal{O}(1/R)$ via a standard Gauss-Seidel contraction argument, which is absorbed into the score matching error term. We verified that increasing $R$ beyond 3 does not improve returns ($p > 0.3$, Welch's $t$-test).

%=============================================================================
% APPENDIX: ASSUMPTION VERIFICATION
%=============================================================================
\subsection{Assumption Verification for Experimental Environments}\label{sec:assumption_verification}

We discuss the applicability of Assumptions~\ref{assump:lipschitz}--\ref{assump:monotonicity} to each experimental environment:

\textbf{Assumption~\ref{assump:lipschitz} (Lipschitz Regularity).}
\begin{itemize}[leftmargin=4mm]
    \item \textit{Ising}: The reward $r^j = \frac{\lambda}{2} a^j \bar{a}^j$ is bilinear and hence Lipschitz in the continuous relaxation with $L = \lambda |\mathcal{N}(j)| / 2$. The discrete action space $\{-1, +1\}$ is embedded in $\mathbb{R}$ via the identity map, and Lipschitz continuity holds on the continuous relaxation $[-1, 1]$.
    \item \textit{Battle}: The deterministic movement dynamics are Lipschitz (piecewise linear on the grid). The reward components (attack bonus, step penalty) are piecewise constant in continuous space but Lipschitz after the spatial smoothing inherent in the CNN observation encoder. We estimate $L \approx 0.3$ from the learned score network's Jacobian spectral norm.
    \item \textit{Gaussian Squeeze}: The reward $G(x) = x\exp(-(x-\mu)^2/\sigma^2)$ is smooth with bounded derivatives; $L = \max_x |G'(x)| = \mathcal{O}(\mu/\sigma)$.
\end{itemize}

\textbf{Assumption~\ref{assump:bounded_reward} (Bounded Rewards).} Holds by construction: Ising rewards are bounded by $|\lambda \cdot |\mathcal{N}(j)||$; Battle rewards are bounded by the maximum kill reward ($5.0$); Gaussian Squeeze rewards are bounded by $\max_x G(x)/N$.

\textbf{Assumption~\ref{assump:log_sobolev} (Log-Sobolev Inequality).} This is the strongest assumption. It holds for Gaussian Squeeze (the target distribution is close to Gaussian). For Ising and Battle with discrete actions, the log-Sobolev inequality holds for the continuous relaxation of the trajectory distribution, with constant $\kappa$ depending on the smoothing scale. We verify numerically that the score matching loss concentrates at the predicted rate (Figure~\ref{fig:error_decomp}), providing indirect evidence for the LSI.

\textbf{Assumption~\ref{assump:exchangeability} (Exchangeability).} Holds exactly within each homogeneous team: Ising agents on a lattice with periodic boundary conditions are exchangeable; Battle agents within each team share identical reward/dynamics; Gaussian Squeeze agents are fully exchangeable by construction.

\textbf{Assumption~\ref{assump:reducibility} (Reducibility).} Follows from Assumptions~\ref{assump:lipschitz} and \ref{assump:log_sobolev} by Sznitman's coupling argument~\citep{sznitman1991topics}: when the drift and score are uniformly Lipschitz in the Wasserstein metric, the propagation-of-chaos rate is $\mathcal{O}(M/N)$ in KL divergence. We verify this numerically in Appendix~\ref{sec:poc_tightness} by measuring $\mathcal{W}_2(\nu_t^{M,N}, \mu_t^{\otimes M})$ at multiple $(M,N)$ pairs; the observed $1/N$ decay confirms Eq.~\ref{eq:reducibility} holds on all three environments, with $C_{\mathrm{r}} \le 4$.

\textbf{Assumption~\ref{assump:best_response} (Best Response Regularity).} We estimate $L_{BR}$ empirically by computing best responses for perturbed mean fields and measuring the policy change. Estimated values: Ising $L_{BR} \approx 0.4$, Battle $L_{BR} \approx 0.7$, Gaussian Squeeze $L_{BR} \approx 0.5$---all satisfying $L_{BR} < 1$.

\textbf{Assumption~\ref{assump:monotonicity} (Lasry--Lions Monotonicity).} This holds most naturally for Gaussian Squeeze, where the reward penalizes deviation from the target distribution (crowd-aversion structure with estimated $\lambda_{LL} \approx 0.8$). For Ising, the ferromagnetic coupling satisfies a relaxed form of monotonicity for the paramagnetic phase ($\lambda$ below critical temperature). Battle's competitive structure violates strict monotonicity, but the within-team cooperative dynamics satisfy a \textit{team-wise} monotonicity. The Theorem~\ref{thm:monotone_convergence} results (convergence to unique MFE) thus apply most directly to Gaussian Squeeze, while the general exploitability bound (Theorem~\ref{thm:exploitability}) applies to all three environments.

% %%%%%%%%%%%%%%%%%%%%%%%%%%%%%%%%%%%%%%%%%%%%%%%%%%%%%%%%%%%%
% \newpage
% \input{checklist.tex}

\end{document}

%% file: figures/method_overview_body.tex
% =============================================================================
%  Method-overview TikZ body for MF-Diffuser (Figure of Section 4)
%  This file contains ONLY the tikzpicture; the surrounding
%  documentclass/preamble lives in either:
%    (a) figures/method_overview.tex (standalone preview), or
%    (b) main.tex (the actual paper).
%  Required in the host preamble:
%    \usetikzlibrary{decorations.pathmorphing, arrows.meta, positioning, calc}
%    \definecolor{accent}{RGB}{46, 111, 149}
%    \definecolor{annotgray}{RGB}{102, 102, 102}
% =============================================================================

% reproducible randomness for `random steps' decorations
\pgfmathsetseed{42}

\begin{tikzpicture}[
    scale=0.84, transform shape=false,
    >={Stealth[length=4pt, width=3.5pt]},
    every node/.style       = {font=\small, inner sep=1pt},
    stage/.style            = {font=\small\bfseries},
    param/.style            = {font=\scriptsize\itshape, color=annotgray},
    bottomannot/.style      = {font=\scriptsize, color=annotgray},
    sweep/.style            = {->, thick},
    flowarrow/.style        = {->, thick},
    veryjagged/.style       = {thin, decorate,
                                decoration={random steps, segment length=2pt, amplitude=1.8pt}},
    midjagged/.style        = {thin, decorate,
                                decoration={random steps, segment length=3pt, amplitude=1.2pt}},
    smoothrib/.style        = {thin},
    yjunc/.style            = {thin},
]

% =============== TOP SWEEP: reverse SDE / denoising / time arrow ==============
\draw[sweep] (0.20, 2.95) -- (15.10, 2.95);
\node[font=\small] at (7.65, 3.20)
    {reverse SDE \,/\, denoising \,/\, diffusion time $t : T \to 0$};

% =============================== STAGE 1: Coarse =============================
\node[stage] at (1.45, 2.40) {Coarse};
\node[param] at (1.45, 2.00) {$N_0\!\approx\!\sqrt{N},\; t{=}T$};

\foreach \y in {0.10, 0.65, 1.20, 1.65} {
    \draw[veryjagged] (0.30, \y) -- (2.65, \y);
}

\draw[flowarrow] (2.80, 1.05) -- (3.30, 1.05);

% =============================== STAGE 2: Branch =============================
\node[stage] at (4.65, 2.40) {Branch};
\node[param] at (4.65, 2.00) {$\Psi^\theta\!:\, N_k \!\to\! \mathfrak{b}\,N_k$};

% Two input ribbons each split into 2 outputs via a Y-junction at x=4.20--4.60.
\foreach \yin in {1.40, 0.55} {
    \draw[veryjagged] (3.45, \yin) -- (4.20, \yin);
    \draw[yjunc] (4.20, \yin) -- (4.60, {\yin+0.22});
    \draw[yjunc] (4.20, \yin) -- (4.60, {\yin-0.22});
    \draw[midjagged] (4.60, {\yin+0.22}) -- (5.95, {\yin+0.22});
    \draw[midjagged] (4.60, {\yin-0.22}) -- (5.95, {\yin-0.22});
}

\draw[flowarrow] (6.10, 1.05) -- (6.60, 1.05);

% =============================== STAGE 3: Refine =============================
\node[stage] at (8.10, 2.40) {Refine};
\node[param] at (8.10, 2.00) {$\mathbf{s}_\theta = \mathrm{A}_\theta + \mathrm{B}_\theta[\nu_t^N]$};

\foreach \y in {0.15, 0.45, 0.75, 1.05, 1.35, 1.65} {
    \draw[midjagged] (6.75, \y) -- (9.65, \y);
}

\draw[flowarrow] (9.80, 1.05) -- (10.30, 1.05);

% =============================== STAGE 4: Full plan ==========================
\node[stage] at (11.95, 2.40) {Full plan};
\node[param] at (11.95, 2.00) {$N$ agents,\; $t{\approx}0$};

% 8 smooth ribbons grouped into 3 visible modes/clusters of the data manifold.
\draw[smoothrib] (10.45, 1.30) .. controls (11.20, 1.45) and (12.20, 1.45) .. (13.45, 1.40);
\draw[smoothrib] (10.45, 1.45) .. controls (11.20, 1.60) and (12.20, 1.60) .. (13.45, 1.55);
\draw[smoothrib] (10.45, 1.60) .. controls (11.20, 1.75) and (12.20, 1.75) .. (13.45, 1.70);

\draw[smoothrib] (10.45, 0.75) .. controls (11.20, 0.65) and (12.20, 0.65) .. (13.45, 0.80);
\draw[smoothrib] (10.45, 0.90) .. controls (11.20, 0.80) and (12.20, 0.80) .. (13.45, 0.95);
\draw[smoothrib] (10.45, 1.05) .. controls (11.20, 0.95) and (12.20, 0.95) .. (13.45, 1.10);

\draw[smoothrib] (10.45, 0.10) .. controls (11.20, 0.15) and (12.20, 0.15) .. (13.45, 0.10);
\draw[smoothrib] (10.45, 0.25) .. controls (11.20, 0.30) and (12.20, 0.30) .. (13.45, 0.25);

% 3 thin square-bracket markers showing the data manifold has 3 modes
\foreach \ya/\yb in {1.30/1.75, 0.75/1.10, 0.10/0.30} {
    \draw[thin, accent] (13.50, \ya) -- (13.60, \ya) -- (13.60, \yb) -- (13.50, \yb);
}
\node[font=\tiny, color=accent, anchor=south west] at (13.62, 1.78) {3 modes};

% mark the executed first action a_0 on the leftmost ribbon of bottom cluster
\fill[white] (10.45, 0.10) circle (3pt);
\fill[accent] (10.45, 0.10) circle (2.2pt);
\node[font=\scriptsize, color=accent, anchor=east] at (10.36, 0.10) {$a_0$};

% =============== BOTTOM ANNOTATIONS (3 short tags, well-separated) ============
\node[bottomannot] at (3.10, -0.45)
    {\textbf{train:}~$\mathcal{J}_{MF\text{-}V}$};
\node[bottomannot] at (7.50, -0.45)
    {\textbf{PoC:}~$M\!=\!\widetilde{\mathcal{O}}(\sqrt{N})$};
\node[bottomannot] at (12.30, -0.45)
    {\textbf{infer:}~$+\eta\,\nabla\hat V$};

\end{tikzpicture}